%% file: no_journal.tex
\pdfoutput=1  
\documentclass[10.5pt,onecolumn]{article}
\usepackage{style_no_journal}
\input{math_commands_no_journal.tex}

\providecommand{\vU}{\bm{\mathrm{U}}}
\providecommand{\vY}{\bm{\mathrm{Y}}}
\providecommand{\vA}{\bm{\mathrm{A}}}
\providecommand{\vW}{\bm{\mathrm{W}}}
\providecommand{\vF}{\bm{\mathrm{F}}}
\providecommand{\vB}{\bm{\mathrm{B}}}
\providecommand{\vV}{\bm{\mathrm{V}}}
\providecommand{\vJ}{\bm{\mathrm{J}}}
\providecommand{\vX}{\bm{\mathrm{X}}}

\providecommand{\veta}{\bm{\eta}}
\providecommand{\vI}{\bm{I}}
\providecommand{\vxi}{\bm{\xi}}

\providecommand{\vOmega}{\bm{\Omega}}
\providecommand{\veps}{\bm{\epsilon}}
\providecommand{\vSigma}{\bm{\Sigma}}
\providecommand{\vG}{\bm{\mathrm{G}}}
\providecommand{\vQ}{\bm{\mathrm{Q}}}
\providecommand{\vR}{\bm{\mathrm{R}}}

\providecommand{\one}{\mathbf{1}}
\providecommand{\psin}{\psi_n}
\providecommand{\psip}{\psi_p}

\providecommand{\sH}{\sigma_H^2}
\providecommand{\sHsquare}{\sigma_H^4}
\providecommand{\vH}{\bm{\mathrm{H}}}
\providecommand{\vL}{\bm{\mathrm{L}}}
\providecommand{\vM}{\bm{\mathrm{M}}}
\providecommand{\vK}{\bm{\mathrm{K}}}
\providecommand{\vC}{\bm{\mathrm{C}}}
\providecommand{\vD}{\bm{\mathrm{D}}}
\providecommand{\vTheta}{\bm{\Theta}}
\providecommand{\vPsi}{\bm{\Psi}}

\providecommand{\Linfty}{\mathcal{L}^{\mathrm{emp}}_{\mathrm{test}}}
\providecommand{\Ltrain}{\mathcal{L}^{n,m}_{\mathrm{train}}}
\providecommand{\Ltest}{\mathcal{L}_{\mathrm{test}}}

\definecolor{schematicBlue}{HTML}{0000FF}
\definecolor{schematicGreen}{HTML}{008000}
\definecolor{schematicRed}{HTML}{FF0000}

\providecommand{\Tr}{\operatorname{Tr}}
\providecommand{\dd}{\mathrm{d}}

\providecommand{\MM}[1]{{\color{blue}{MM: #1}}}
\makeatletter
\renewcommand{\TODO}[1]{}
\renewcommand{\RU}[1]{}
\renewcommand{\TB}[1]{}
\renewcommand{\MM}[1]{}

\makeatother

\title{Double Descent and Malign Overfitting in Diffusion Models}

\author[1]{Rapha\"el Urfin$^{\dagger}$%
\thanks{Corresponding author: \href{mailto:raphael.urfin@phys.ens.fr}{\texttt{raphael.urfin@phys.ens.fr}}}%
}
\author[2]{Tony Bonnaire$^{\dagger}$}
\author[1]{Giulio Biroli}
\author[3]{Marc M\'ezard}
\affil[1]{Laboratoire de Physique de l'\'Ecole normale sup\'erieure, ENS, Universit\'e PSL, CNRS, Sorbonne Universit\'e, Universit\'e Paris Cit\'e, F-75005 Paris, France}
\affil[2]{Universit\'e Paris-Saclay, CNRS, Institut d'Astrophysique Spatiale, 91405 Orsay, France}
\affil[3]{Department of Computing Sciences, Bocconi University, Milano, Italy}
\date{\vspace{-7ex}}

\begin{document}

\pagenumbering{arabic}
\maketitle
\def\thefootnote{$\dagger$}\footnotetext{Equal contribution.}\def\thefootnote{\arabic{footnote}}
\selectlanguage{american}

\begin{abstract}
Conventional wisdom in deep learning holds that overparameterization---having more parameters $p$ than training samples $n$---is benign: larger models generalize better and, even without regularization, interpolating models generalize well, the test error following a double-descent curve.
One might expect the same benign overfitting for diffusion models, whose training reduces to regression, i.e. to minimizing a quadratic denoising score-matching loss. Yet the opposite is observed: overfitting here is catastrophic, driving the model into a memorization regime. We resolve this paradox by combining experiments on U-Nets trained on CelebA with a random-features model for which we derive closed-form learning curves.
We show that with a fixed number $m$ of noise realizations per training sample, an interpolation peak does occur, but at $p\sim nm$ rather than at $p\sim n$ as in standard regression. The rise of the test loss, however, sets in much earlier, at $p\sim n$, independently of $m$. This overfitting is \emph{malign} because, although the implicit regularization of training is fully at work, it drives the model toward the empirical score, which memorizes the training set, rather than toward the true score. A bias--variance decomposition pinpoints the mechanism: the bias of the score estimator starts to grow at $p\sim n$; past the peak the variance decays, as in regression, whereas the bias keeps growing and both saturate at a large value. Since diffusion models are trained with $m\gg1$, the peak is pushed to very large model sizes, and therefore sit on the rising branch that precedes it, where malign overfitting is already in play. Nevertheless, overparameterization remains beneficial when paired with regularization: in the random-features theory and in U-Net experiments, optimally regularized large models---via a ridge penalty or early stopping, respectively---outperform unregularized models of any size.
\end{abstract}

\vspace{1cm}

\vspace{-5ex}

\input{Intro}

\input{double_descent_in_diffusion_models}

\input{Discussion}

\input{Tau}

\input{practice}

\input{Conclusion}

\section*{Acknowledgments}
GB acknowledges support from the French government under the management of ANR: PEPR-IA (project MAGICALL ANR-25-PEIA-0004) and PR[AI]RIE-PSAI (ANR-23-IACL-
0008). This work was performed using HPC resources from GENCI-IDRIS (Grants 2026-AD011016319R1 \& 2026-A0201016159).

\bibliography{Bibliography_arxiv}

\newpage

\begin{center}
    {\LARGE Double Descent and Malign Overfitting in Diffusion Models \\ \vspace{1ex} {\Large \bf Appendix}
    \\ \vspace{3ex} }{\large Rapha\"el Urfin, Tony Bonnaire, Giulio Biroli, Marc M\'ezard}
\end{center}

\appendix

\input{Appendix}

\end{document}

%% file: math_commands_no_journal.tex
\usepackage{amsmath,amsfonts,bm}

\def\eqref#1{equation~\ref{#1}}

\def\1{\bm{1}}

\def\vmu{{\bm{\mu}}}
\def\vtheta{{\bm{\theta}}}

\def\vb{{\bm{b}}}

\def\vg{{\bm{g}}}
\def\vh{{\bm{h}}}

\def\vs{{\bm{s}}}

\def\vv{{\bm{v}}}

\def\vx{{\bm{\mathrm{x}}}}  
\def\vy{{\bm{y}}}

\DeclareMathAlphabet{\mathsfit}{\encodingdefault}{\sfdefault}{m}{sl}
\SetMathAlphabet{\mathsfit}{bold}{\encodingdefault}{\sfdefault}{bx}{n}

\newcommand{\E}{\mathbb{E}}

\DeclareMathOperator{\Tr}{Tr}

%% file: Intro.tex
\section{Introduction}

One of the most surprising discoveries in modern machine learning is that overparameterization is not necessarily harmful, and can in fact be beneficial in classification and regression problems. Contrary to the traditional wisdom that one should ``never fit noisy training data exactly'', heavily overparameterized models in modern machine learning often interpolate the training data while achieving excellent test performance \citep{lawrence1997lessons,neyshabur2015search,zhang2017understanding}. 
The most striking manifestation of this phenomenon is the {\it double descent} curve of the test error as a function of model capacity---specifically, the number of learnable parameters in the neural network \citep{Belkin_2019,geiger2020scaling}. Rather than exhibiting the classical U-shaped curve associated to bias--variance tradeoff \citep{wasserman2004all}, the test error first follows this U-shape, then peaks at the \emph{interpolation threshold}---the point where the model can exactly fit the training data and the variance of the model is large---and then descends again, often reaching a global minimum in the heavily overparameterized regime. This phenomenon, and in particular this last descent, called {\it benign overfitting}, has proved remarkably ubiquitous across different model classes, underscoring the benefits of overparameterization. It was subsequently demonstrated at scale in deep networks by \citet{Nakkiran_2021} and analyzed theoretically in multiple studies \citep{Barlett_benign,hastie2022surprises,mei2020,Ascoli_2020}. \looseness=-1

Our work revisits this central phenomenon for diffusion and score-based models, \citep{sohl-dickstein_15,ho2020,song2019,song2021b} which have become the state of the art for generating images \citep{ esser2024scalingrectifiedflowtransformers}, sounds \citep{kong2021diffwave}, and videos \citep{yang2025cogvideox}. Diffusion models are trained by progressively noising the data and teaching the model to reverse this process.
In the language of nonequilibrium physics, they learn to time-reverse an Ornstein--Uhlenbeck (OU) process. The key ingredient is learning the \emph{score function}, which is the force field that enables time-reversal and denoising \citep{hyvarinen_05,Vincent_2011}. This reduces to {\it a sequence of regression problems}: at each noise level, one minimizes a quadratic loss between the neural network's predicted score and the ground truth.
As a result, in light of the modern machine learning findings discussed above, one might expect benign overfitting to apply here as well. Yet numerical experiments reveal a strikingly different story: increasing model capacity leads to memorization and poor generalization \citep{yoon2023diffusion, gu2023memorization, kadkhodaie_2024}. Overfitting proves highly detrimental (``malign''), as it compromises generative quality by causing the model to reproduce training data rather than capture the underlying distribution. \looseness=-1

This leads to a puzzle: diffusion models give rise to regression problems, and therefore one might expect benign overfitting. Yet experiments show the opposite.
The aim of our work is to resolve this paradox by revisiting the double descent phenomenon for diffusion models. This is relevant both for the theory of generative models and for applications. In particular, it allows us to clarify a fundamental question about generative models: {Does overparameterization help or hurt diffusion models?} \looseness=-1

\begin{figure}
    \centering
    \includegraphics[width=0.75\linewidth]{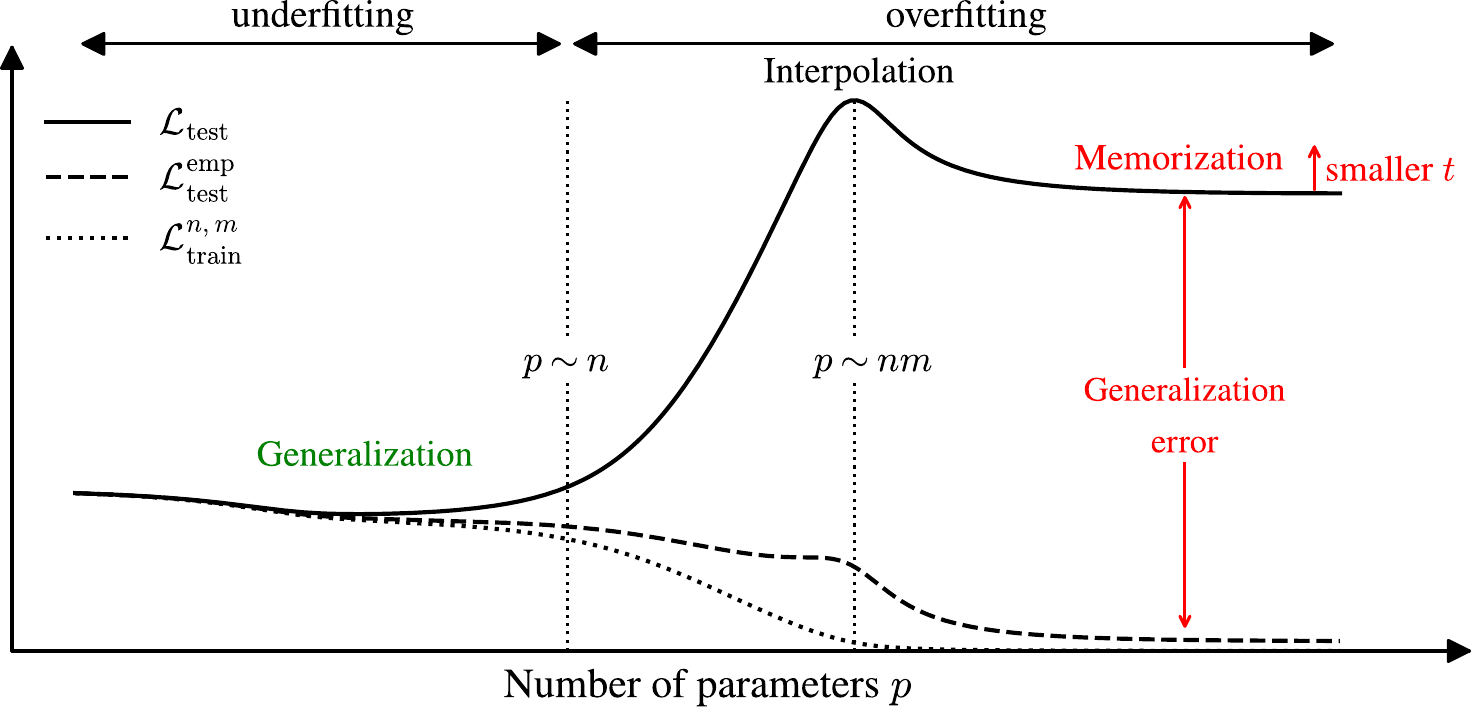}
    \caption{
    \textbf{Qualitative summary of our results.}
  Sketch of the standard test loss $\Ltest$ (solid), the empirical test loss $\Linfty$ (dashed), and the training loss $\Ltrain$ (dotted) versus the number of parameters $p$, for a diffusion model trained on $n$ samples with $m>1$ noise realizations each. $\Ltest$ reaches its minimum before $p\sim n$, where the model can \textcolor{schematicGreen}{generalize}, then rises to a peak at the \textit{interpolation threshold} $p\sim nm$, where $\Ltrain$ reaches $0$ (shifted from $p\sim n$, its value in the regression setting $m=1$). For $p\gg nm$, the test loss decreases a second time (\emph{double descent}) but settles on a plateau that lies above its earlier minimum, whereas $\Linfty$ decreases toward zero, more so the larger $m$: since $\Linfty$ measures the distance to the \emph{empirical score}, which memorizes the training set, the model is approaching the \textcolor{schematicRed}{memorization} regime. The generalization error is larger, the smaller is $t$ (the amount of noise). 
  In contrast to the regression setting $m=1$, where overfitting is \textcolor{schematicGreen}{benign}, here it is \textcolor{schematicRed}{malign}. \looseness=-1
}

    \label{fig:Sketch_results}
\end{figure}

\paragraph*{Generative Diffusion and Its Different Losses.}
Standard diffusion models transport a target distribution $P_0$ on $\mathbb{R}^d$ to Gaussian white noise $\mathcal{N}(0,\vI_d)$ via an OU \emph{forward process} $\dd\vx = -\vx\,\dd t + \sqrt{2}\,\dd\vW(t)$, where $\vW(t)$ is a standard Wiener process.
Exact time-reversal of this OU process can be implemented \citep{Anderson_1982, haussmann_1986} by using a guiding force field which is
the exact score function $\nabla_\vx\log P_t(\vx)$, where $P_t$ is the marginal density at time $t$.
 Following \citet{hyvarinen_05} and \citet{Vincent_2011}, generation is performed by using a parametrized score function $\vs(\vx,t)$. Given a database $\mathcal D=\{\vx^\nu\}_{\nu=1,\dots,n}$ consisting of $n$ i.i.d.\ samples of $P_0$, one learns this score function
 by minimizing the Denoising Score Matching (DSM) loss. At fixed $t$\footnote{The loss at fixed 
$t$, rather than averaged over $t$, reveals how the noise level controls malign overfitting.}, this DSM loss reads \looseness=-1
\begin{align}
      \Linfty(\vs) &= \frac{1}{nd}\sum_{\nu=1}^n\mathbb{E}_{\vxi}\!\left[\left\lVert\sqrt{\Delta_t}\,\vs(e^{-t}\vx^\nu+\sqrt{\Delta_t}\,\vxi)+\vxi\right\rVert^2\right],\label{eq:m_infty_loss}
\end{align}
where $\Delta_t=1-e^{-2t}$ and the expectation is over $\vxi\sim\mathcal{N}(0,\vI_d)$. In practice, the average over $\vxi$ is not performed exactly: the training is done on batches of the $n$ samples, drawing a fresh noise $\vxi$ each time a sample is visited. 
Each sample is thus seen with $m$ different noise realizations, $m$ being the number of epochs (typically\footnote{With fresh noise at each step, $m=N_{\rm steps}B/n$ with $N_{\rm steps}$ gradient steps and batch size $B$; this gives $m\approx2\cdot10^3$ for DDPM on CIFAR-10 and EDM on ImageNet-64 \citep{ho2020,karras2022elucidatingdesignspacediffusionbased}. \looseness=-1}$10^3$). To analyze the role of this finite noise sampling within the classical double-descent framework, which is formulated for the minimizer of an empirical loss \citep{Belkin_2019,geiger2020scaling}, we fix the noises and consider the loss\looseness=-1
\begin{align}\label{eq:train_loss}
  \Ltrain(\vs) = \frac{1}{nmd}\sum_{\nu=1}^n\sum_{\mu=1}^m
    \left\lVert\sqrt{\Delta_t}\,\vs(e^{-t}\vx^\nu+\sqrt{\Delta_t}\,\vxi^{\nu\mu})+\vxi^{\nu\mu}\right\rVert^2.
\end{align}
Fixing the noises is an idealization of the actual procedure, but it retains
its essential feature: only $nm$ pairs $(\vx^\nu,\vxi^{\nu\mu})$ are
presented to the network ($\Ltrain\to\Linfty$ for $m\to \infty$).
It is also the standard setting of theoretical analyses \citep{cui_2024, george_2025}.
While this resembles a regression problem, it is a non-standard one: the score $\vs$ is evaluated on points $\vy^{\nu,\mu}=e^{-t}\vx^\nu+\sqrt{\Delta_t}\,\vxi^{\nu\mu}$ that are correlated, forming clusters of $m$ points around each sample $\vx^\nu$.
Finally, generalization is measured by the \emph{standard} test loss \looseness=-1
\begin{align}
  \Ltest(\vs) &= \frac{1}{d}\,\mathbb{E}_{\vx,\vxi}\!\left[\left\lVert\sqrt{\Delta_t}\,\vs(e^{-t}\vx+\sqrt{\Delta_t}\,\vxi)+\vxi\right\rVert^2\right].\label{eq:test_loss}
\end{align}
where the expectation is taken over both fresh noises and fresh samples. 
We refer to $\Linfty(\vs)$ as the \emph{empirical} test loss: it is a test loss since the noise is averaged exactly, but also an empirical one because it 
depends on the $n$ training samples. In practice it is readily evaluated by averaging over fresh noises.
To identify where the test loss comes from, we analyze the bias and variance
contributions, $\mathcal{B}^2$ and $\mathcal{V}$, to
$\Ltest = C_t + \Delta_t\left(\mathcal{B}^2+\mathcal{V}\right)$, with \looseness=-1
\begin{align}
\label{eq:bias_variance}
    \mathcal{B}^2=\frac{1}{d}\,\mathbb{E}_{\vy}\!\left[\lVert \nabla_\vy\log P_t(\vy)-\langle \vs_{\mathcal{D},\vTheta}(\vy)\rangle\rVert^2\right], \quad
    \mathcal{V}=\frac{1}{d}\,\mathbb{E}_{\vy}\!\left[\left\langle \lVert \vs_{\mathcal{D},\vTheta}(\vy)-\langle \vs_{\mathcal{D},\vTheta}(\vy)\rangle\rVert^2\right\rangle\right].
\end{align}
Here $\vs_{\mathcal{D},\vTheta}$ denotes the minimizer of the train loss $\Ltrain$ for a given dataset $\mathcal{D}$ and initialization of the parameters $\vTheta$; $C_t$ a term 
which depends only on the data distribution and is
independent of the learned score (see Appendix \ref{app:bias_variance_general_case});
$\mathbb{E}_{\vy}$ is the average over test samples $\vy\sim P_t$, whereas $\langle\cdot\rangle$ denotes the average over the realizations of the training set, and of other sources of randomness such as $\vTheta$. \looseness=-1
As we will show, distinguishing between $\Linfty$ and $\Ltest$ is key to resolving the paradox: benign behavior in the former coexists with malign overfitting in the latter. \looseness=-1

\paragraph*{Contributions and theoretical picture.}
Figure~\ref{fig:Sketch_results} summarizes the behavior of the different losses, as obtained from both the theoretical analysis and our numerical experiments, each evaluated from the score learned by training on $n$ samples independently noised $m$ times. Importantly, we consider solutions obtained after very large training times $\tau$, representative of the $\tau \to \infty$ limit, and without regularization.
The effects of finite $\tau$ and of regularization are discussed later.
The first important observation is that the test loss $\Ltest$ does exhibit a double descent as the number of parameters increases: it first decreases toward a minimum, rises to a peak, then descends again. The peak occurs at the interpolation threshold, where the training loss reaches zero and the test loss is dominated by the variance of the estimator. Compared to the usual double descent, however, the peak is shifted to much larger model sizes: its position now grows with both $n$ and $m$ (for simple models, proportionally to $nm$). More importantly, while the test loss descends after the peak, it plateaus at a value larger than its first minimum; the smaller $t$, the
higher this plateau is.
Whereas in the supervised setting both variance and bias decrease after the interpolation peak \citep{neal2019a,Ascoli_2020}, here the variance decays and the bias grows. Both saturate at large values, producing the high plateau observed at small $t$.
Here, overfitting is clearly detrimental.
The key to understanding this phenomenon lies in the empirical test loss $\Linfty$, which, up to a constant, measures the mean square distance to the \emph{empirical score}, the score associated with the mixture of Gaussians centered on the points $e^{-t} \vx^\nu$ with variance $\Delta_t$
(see Appendix~\ref{app:loss_decomposition}). Its global minimizer is the empirical score itself \citep{Biroli_2024, li_2024_good_score}, which memorizes the training set unless $n$
grows exponentially with $d$ \citep{Biroli_2024}. 
Numerically, $\Linfty$ first follows the training loss closely. It then displays a peak at
small $m$, which fades as $m$ grows, and beyond the interpolation threshold it
reaches very small values.
This is the key to solving the paradox: the
benign-overfitting mechanism---whereby the training dynamics selects, among all
interpolating solutions, those with the smallest error---is fully operative, but
it acts on $\Linfty$ and not on the loss we actually care about, $\Ltest$.
The implicit regularization of training drives the model toward the minimizer of
the \emph{empirical} test loss, i.e. the empirical score which memorizes the training set, rather than toward the exact score: hence the name \emph{malign} overfitting\footnote{Although this loss vanishes only for $m,p\to\infty$, we find that a very small value is obtained already for moderate $m$, thus implying a very small distance to the empirical score and hence partial memorization.
}. This also explains the growth of the bias in the overparameterized regime: it does not reflect a limitation of
the model class, but the minimization of the ``wrong'' loss. None of this implies that overparameterization hurts diffusion models. Overparameterization is still beneficial provided it is paired with a suitable regularization: with a ridge penalty (in the theory) or early stopping (in the experiments), both the peak and the plateau of Fig.~\ref{fig:Sketch_results} disappear, and the test loss and the FID reach their lowest values in the overparameterized regime ($p\gg n$).  \looseness=-1

The theoretical picture described above rests on two complementary analyses. Analytically, we study Random Feature Neural Networks \citep{Rahimi_2007} at arbitrary finite $m$, in the proportional limit where the dimension $d$, the data $n$ and the number of parameters $p$ all go to infinity with fixed ratios, extending \citet{george_2025}. Numerically, we train DDPM-style U-Nets \citep{ho2020, Ronneberger2015} on the CelebA dataset \citep{CelebA} across a range of network widths $W$, tracking how the losses and the Fréchet-Inception Distance \citep[FID,][]{heusel2017gans} evolve. In the following sections, we present these findings in detail, discussing the role of $m$, the bias and variance of the estimator, and the impact of regularization. \looseness=-1

\paragraph*{Related works.} 
We present here the closest relevant works, and refer to Appendix~\ref{app:discussion_related_works} for a more thorough discussion. Diffusion models are known empirically to overfit and memorize the training data, with state-of-the-art image models reproducing a non-negligible fraction of their training set \citep{Carlini_2023, somepalli_2022, somepalli_2023}. This memorization has been shown to depend heavily on the data distribution, the model, and the training procedure \citep{gu2023memorization, yoon2023diffusion}, but also to be significantly reduced by weight decay and early-stopping \citep{gu2023memorization, baptista_2025, Favero2025_bigger}.
However, it is unavoidable under the empirical score hypothesis \citep{Biroli_2024, achilli2024, ventura2025} unless $n$ grows exponentially with $d$ \citep{Biroli_2024}. Several theoretical analyses account for the network architecture and the high dimensionality of the data, in particular by studying the score learned by a parametric model \citep{cui_2024, cui_2025, merger2026,buchanan2025}. Closest to our work, \citet{george_2025} compute asymptotic training and test DSM losses for a Random Feature model \citep{Rahimi_2007} and link them to memorization for $m=1$ and $m=\infty$. We extend their result to arbitrary finite $m$. \citet{bonnaire2025diffusionmodelsdontmemorize} study the generalization and
memorization timescales of the corresponding $m=\infty$ training dynamics, while, concurrently with our work,
\citet{latourellevigeant2026generalizationmemorizationoverfittingdiffusion}
analyze the lazy regime ($p\to\infty$) training dynamics, together with the
resulting sampling dynamics.
Finally, that diffusion models do not overfit benignly is by now well documented empirically \citep{yoon2023diffusion} and proven in a model-independent way by \citet{farghly2026benign}: overfitting and good generalization cannot coexist unless the sample size grows exponentially with $d$.
What remains open, and what we clarify in this work, is twofold: the mechanism behind this failure, i.e.\ what is specific to diffusion models that switches off the benign overfitting at work in supervised learning, and whether double descent carries over at all. On the latter, the evidence is conflicting: in the RFNN, \citet{george_2025} find an interpolation peak for $m=1$ that vanishes for $m=\infty$, while \citet{marion2026understandingdiffusionmodelsrequires} observe an epoch-wise double descent in distributional metrics such as the FID, but none in the test loss. Working at finite $m$ is the key ingredient to solve this puzzle: it is what connects the benign overfitting of regression ($m=1$) to the malign overfitting observed in practice ($m\gg1$), and allow us to identify the mechanism behind the latter. \looseness=-1

%% file: double_descent_in_diffusion_models.tex
\section{Double Descent in Diffusion Models}
\input{Numerics}
\input{Analytical_results}
\subsection{Main Results}

\begin{figure}
    \centering
    \includegraphics[width=0.361\linewidth]{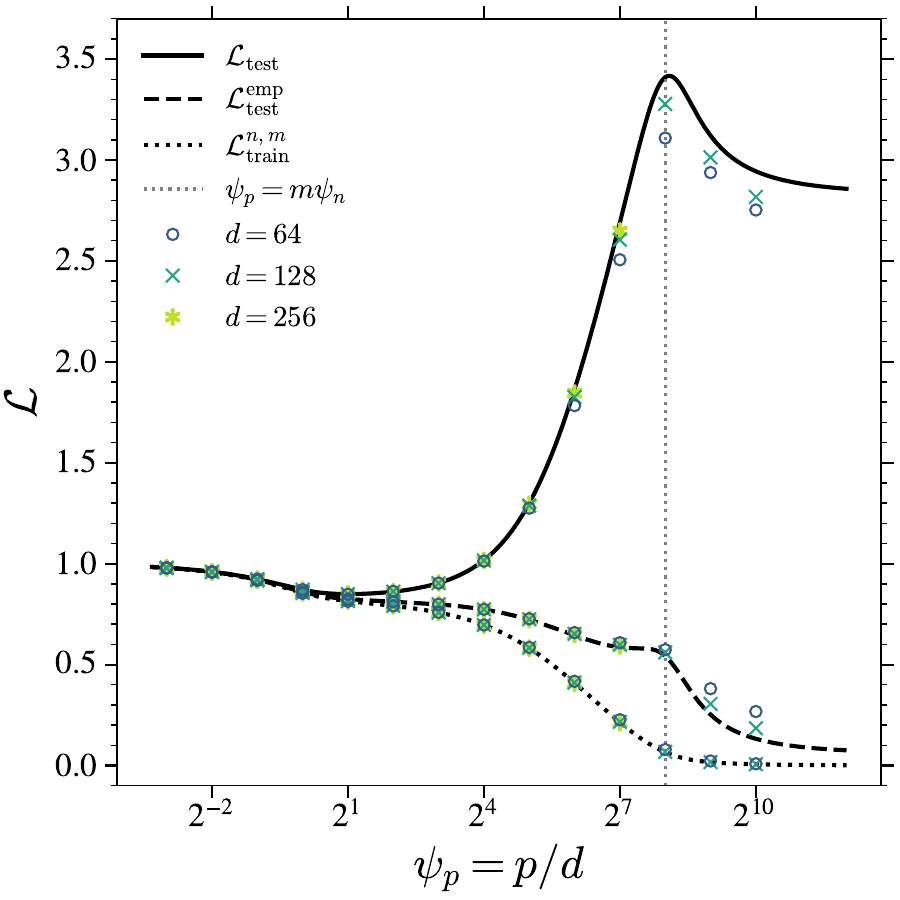}%
    \hspace{0.009\linewidth}%
    \includegraphics[width=0.361\linewidth]{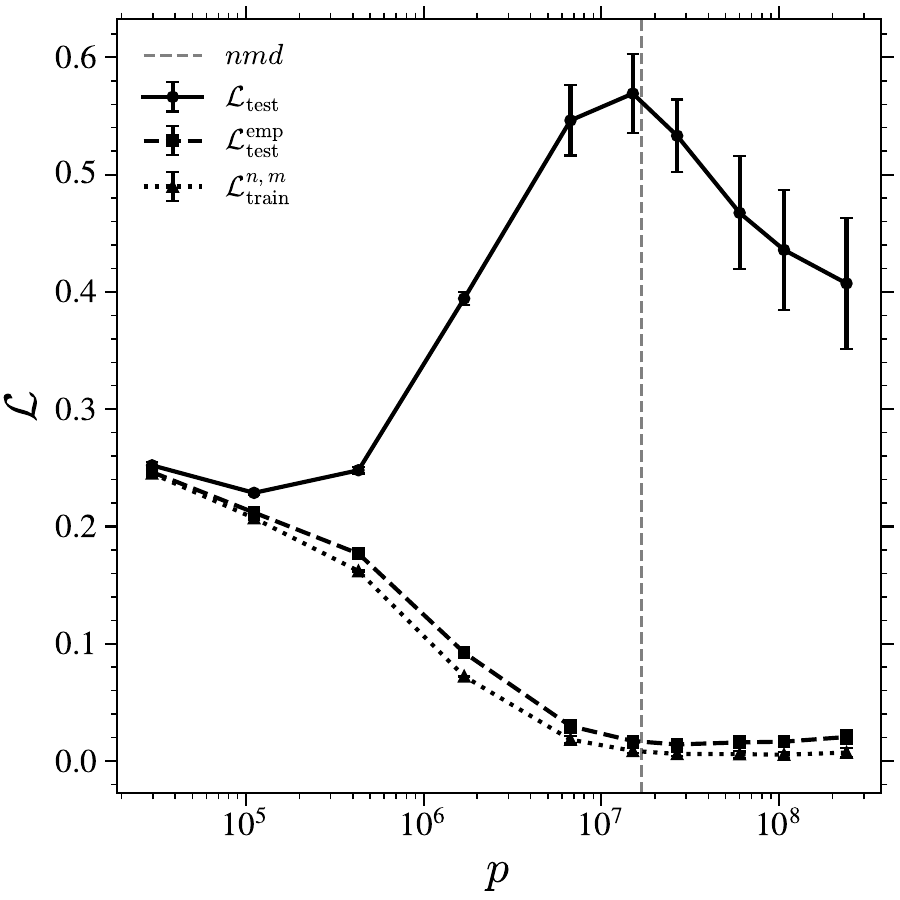}%
    \hspace{0.009\linewidth}%
    \begin{minipage}[b]{0.178\linewidth}
        \centering
        \includegraphics[width=\linewidth]{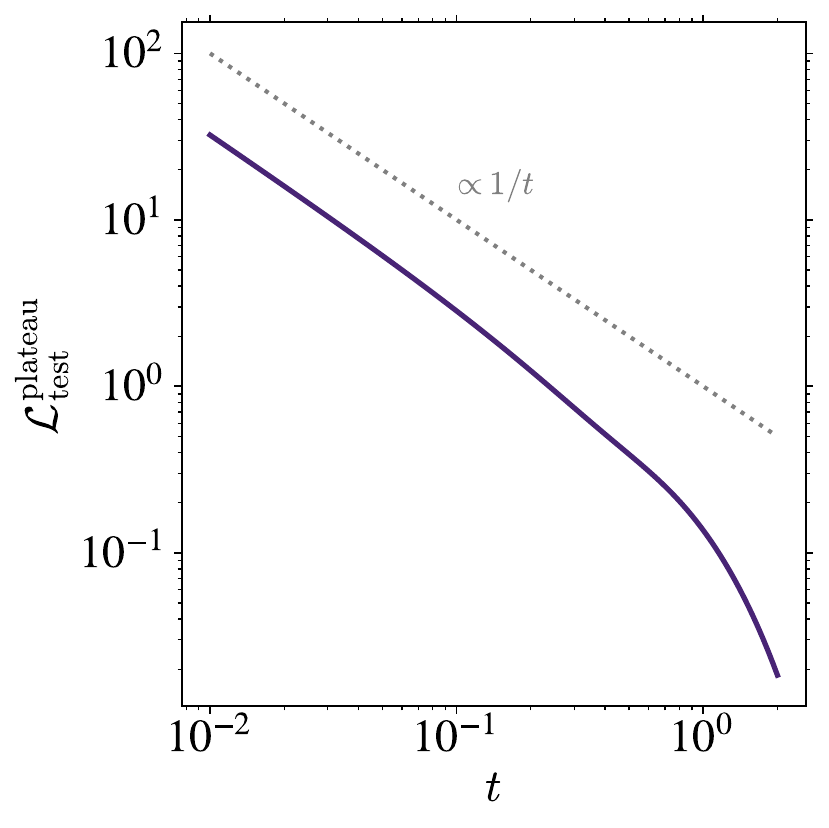}\\[0pt]
        \includegraphics[width=\linewidth]{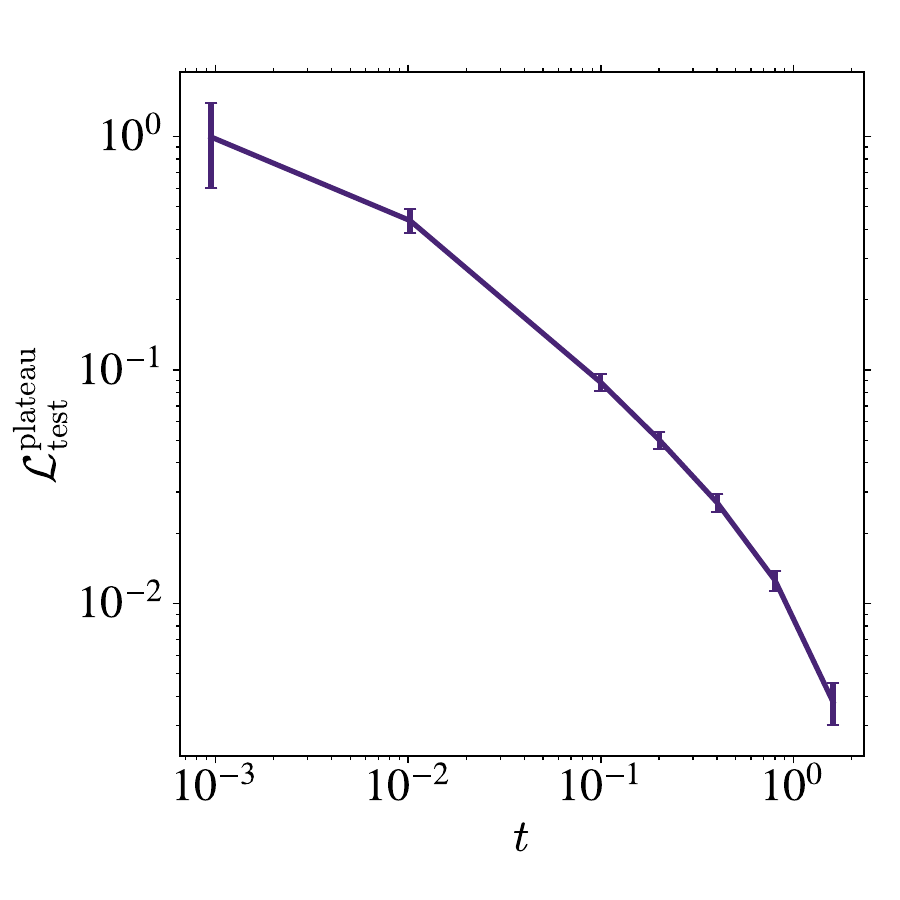}
    \end{minipage}%
    \caption{
\textbf{Double descent in Diffusion Models.} \emph{(Left)} Losses vs. $\psi_p=p/d$ for an RFNN with Gaussian data at $\psi_n=n/d=16$, $m=16$, $t=0.1$, $\lambda=10^{-3}$; black curves are the predictions of Theorem~\ref{thm:main}, markers the RFNN trained with Adam on Eq.~\ref{eq:train_loss} with ridge penalty $\lambda\Delta_t\lVert\vA\rVert_F^2/(dp)$, averaged over 10 runs. \emph{(Middle)} Losses vs. $p$ for U-Nets trained on CelebA at $n=2048$, $m=8$, evaluated at $t\approx0.01$; error bars are $\pm3$SE over 6 initial seeds. \emph{(Right)} Large $p$ plateau value of the test loss vs. $t$ for the RFNN \emph{(top)} and the U-Net \emph{(bottom)}, $W=128$, $p\approx106$M. \looseness=-1
    } 
    \label{fig:RF_UNET_three_losses}
\end{figure}

\paragraph*{Parameter-wise double descent exists in diffusion models.} 
Both the RFNN theory and the U-Net experiments exhibit \emph{double descent}, as
shown in Fig.~\ref{fig:RF_UNET_three_losses}. As a function of the model size $p$, the test loss $\Ltest$ first
decreases to a minimum and then rises to a peak located where the training loss $\Ltrain$ vanishes, i.e.\ at the interpolation threshold, corresponding to the smallest $p$ reaching very small values of the train loss. Past the peak, in the overparameterized regime, $\Ltest$ settles at a
value larger than at its first minimum; the smaller $t$, the higher is this
plateau, as shown in the right panels of
Fig.~\ref{fig:RF_UNET_three_losses}.
The empirical test loss $\Linfty$, by contrast, tracks the training loss and
reaches its minimum in the overparameterized regime. For the RFNN, $\Linfty$
exhibits a double descent at moderate $m$, which fades away as $m$ grows. These results, obtained at fixed $t$, persist for the test loss integrated over
$t$ (Appendix~\ref{app:num:integrated_loss}), which also displays a peak and a
plateau---as expected, since the integral collects all noise levels and the
small-$t$ ones exhibit these features most strongly. \looseness=-1

\paragraph*{Both $n$ and $m$ shift the interpolation peak.} As $p$ increases, the RFNN model exhibits three regimes with distinct scalings, as shown in the three leftmost panels of Fig.~\ref{fig:scaling}: \textit{(i)} for
$p < n$, the test loss decreases and depends on $p$ alone, independently of
$n$ and $m$ (left panel); \textit{(ii)} for $n < p < nm$, $\Ltest$
increases and depends only on the ratio $p/n$ (top central panel); \textit{(iii)} the peak
height, reached at $p = nm$ (lower central panel), and the plateau beyond it are independent of
$n$ (left panel). The threshold $p=nm$ is simply where the $pd$ readout parameters match the $nmd$ scalar conditions imposed by $\Ltrain$, i.e. one $d$-dimensional equation per
noisy point $\vy^{\nu\mu}$.
For the U-Net, the peak shifts to larger $p$ as both $n$ and $m$ grow, as displayed in the right panel of Fig.~\ref{fig:scaling}. The RFNN scalings carry over to the U-Net across the $(n,m,p)$ range we explore (see Appendix~\ref{app:num:scaling_nm})---the peak is also approximately located where the number of parameters equals $nmd$---but we do not expect it to hold
in general, in particular due to the strong correlations among the $nmd$ targets and to the weight sharing of the convolutional architecture, which decouples $p$ from the effective degrees of freedom. \looseness=-1

\begin{figure}
    \centering
    \includegraphics[width=0.552\linewidth]{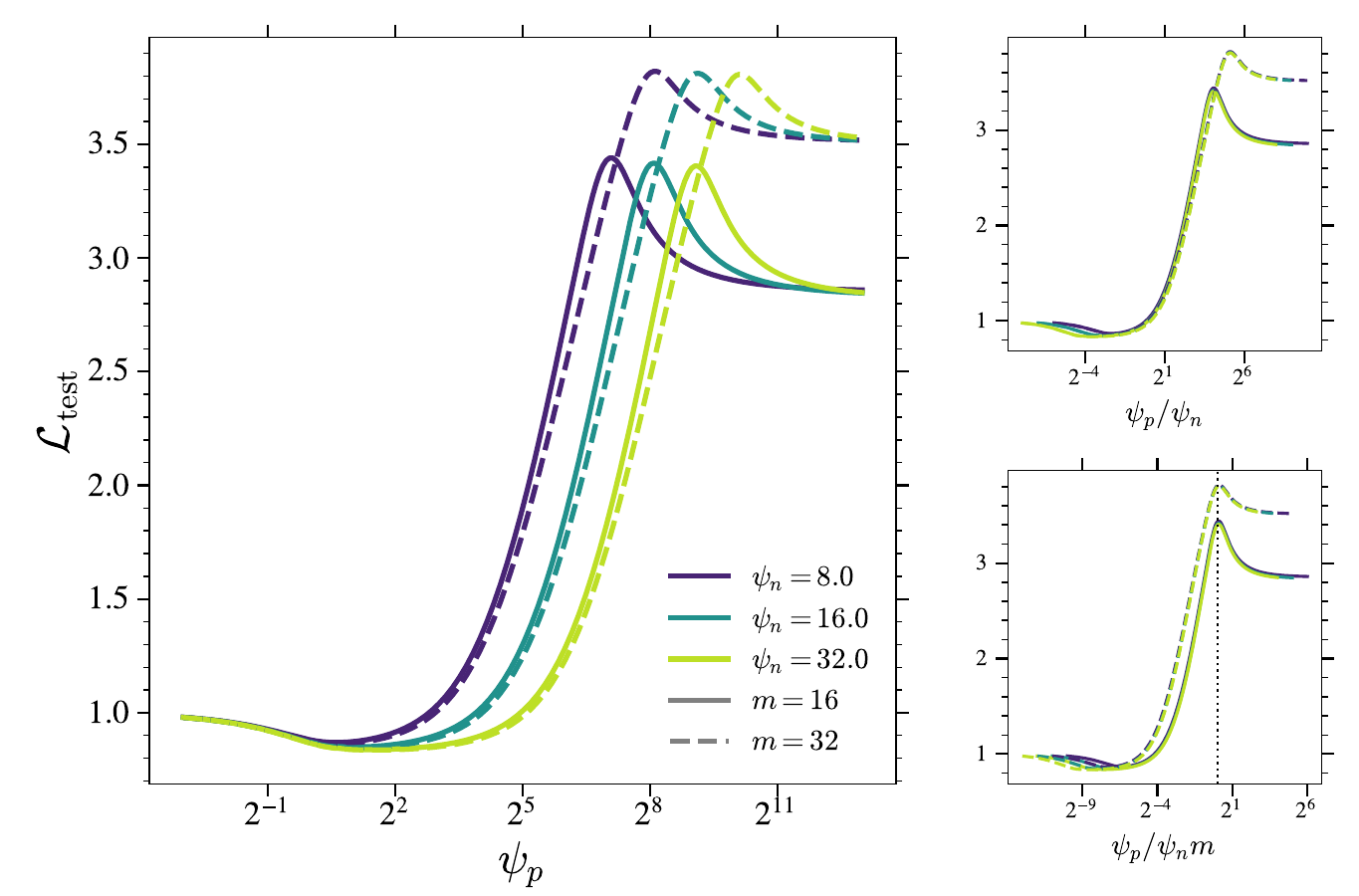}%
    \hspace{0.009\linewidth}%
    \includegraphics[width=0.354\linewidth]{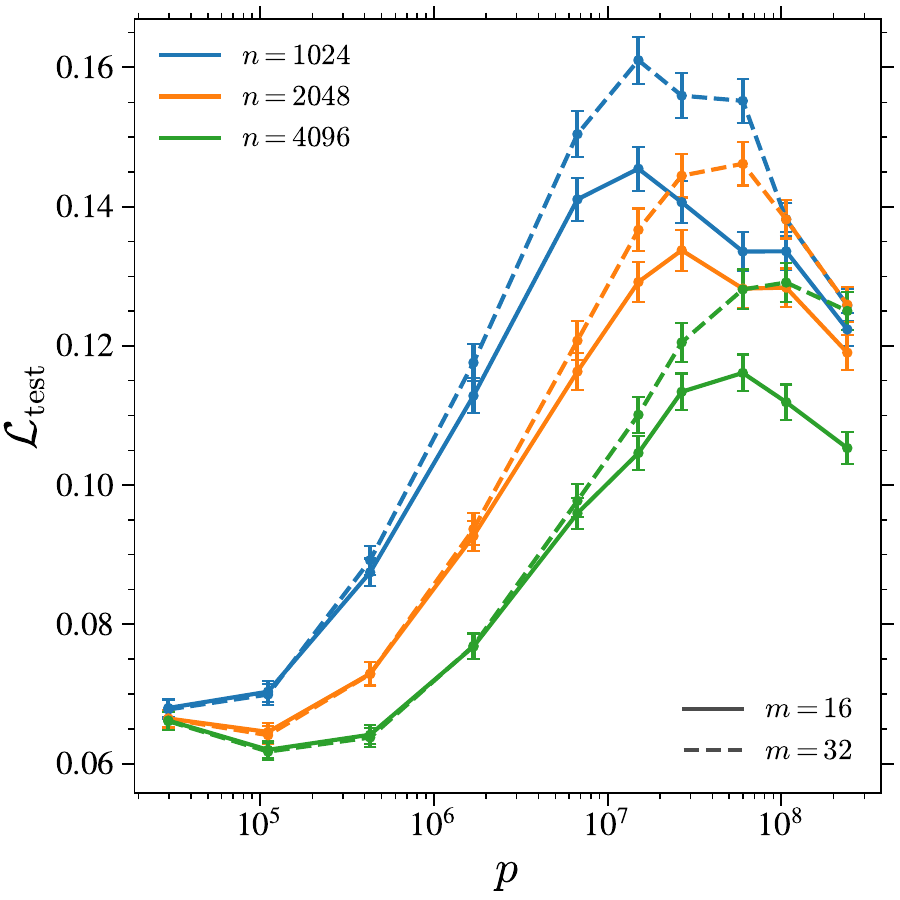}
    \caption{
    \textbf{Effects of $n$ and $m$ on the double descent.} \emph{(Left)} RFNN test loss vs. $\psi_p$ for several $\psi_n$ and $m$, at $t=0.1$, $\lambda=10^{-3}$, $\sigma=\tanh$. \emph{(Middle)} Same, with the $x$-axis rescaled by $\psi_n$ (top) and $\psi_n m$ (bottom). \emph{(Right)} U-Net test loss vs. $p$ for several $m$ at $t\approx0.1$; error bars are $\pm3$SE of one seed over test sets. \looseness=-1
    }
    \label{fig:scaling}
\end{figure}

\begin{figure}
    \centering
    
    \includegraphics[width=0.45\linewidth]{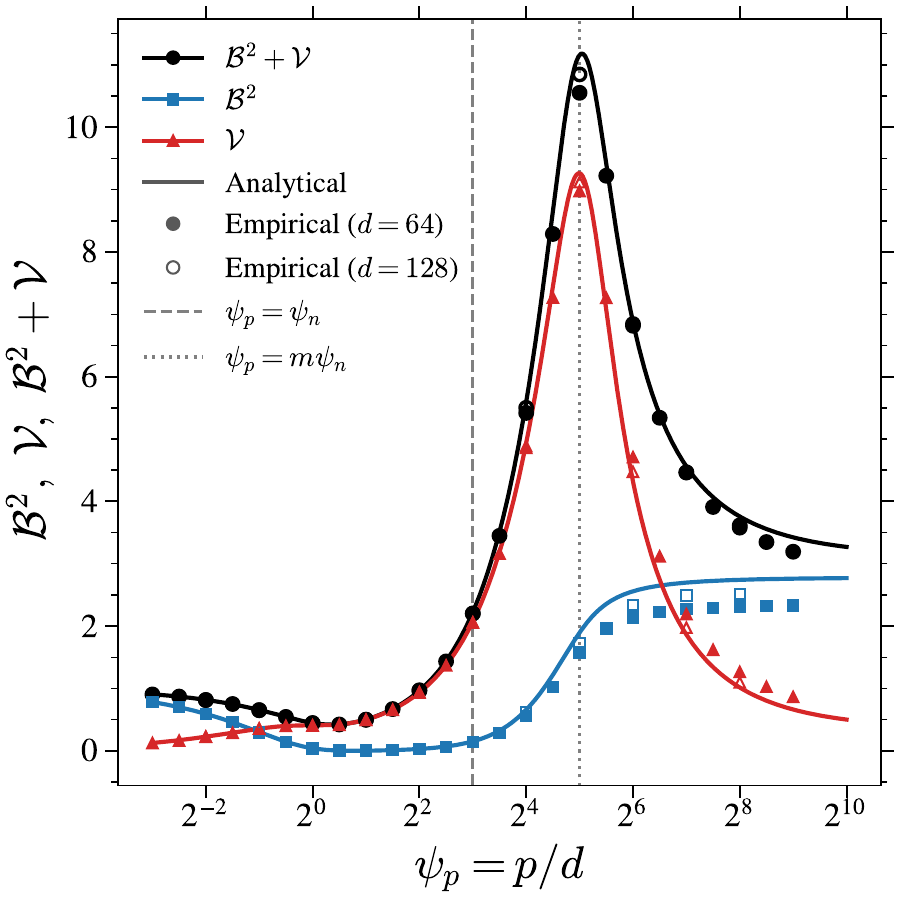
    }
    \includegraphics[width=0.45\linewidth]{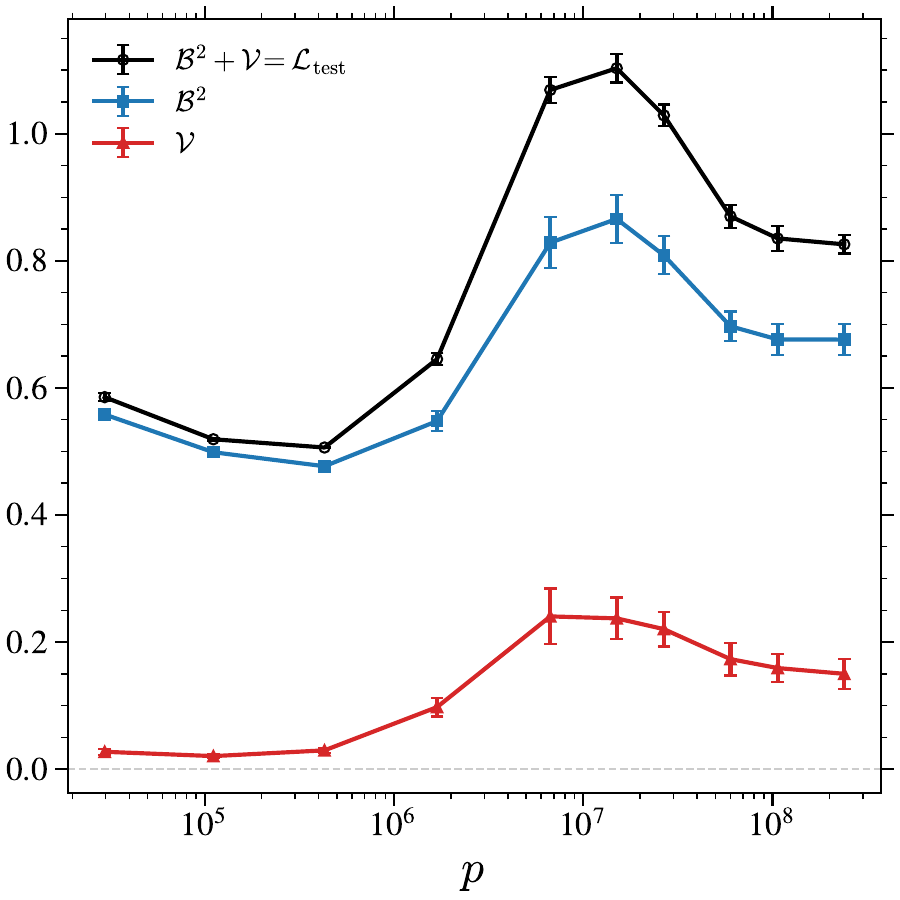}
    \caption{\textbf{Bias and variance of the score estimator.} \emph{(Left)} Bias and variance for the RFNN model vs. $\psi_p$ for $\psi_n=8, t=0.1, m=4, \lambda=10^{-3},\sigma=\tanh$ and several values of $d$. Solid curves are the analytical predictions from Theorem~\ref{thm:main}; markers are numerical estimates from 10 independent runs. \emph{(Right)}
     Bias and variance for U-Net models as a function of $p$ ($n=2048$, $m=4$, evaluated at $t\approx0.01$). Error bars correspond to $\pm3$SE over 10 models trained on disjoint datasets. \looseness=-1}
    \label{fig:bias_variance}
\end{figure}

\paragraph*{Malign overfitting and the growth of the bias.}
A finer picture emerges from the bias--variance decomposition of the test loss from Eq.~\ref{eq:bias_variance}, reported in Fig.~\ref{fig:bias_variance} for the RFNN (left) and the U-Net (right). For small models the test loss is pure bias; as $p$ grows the bias decreases while the variance builds up which sets the first minimum of the test error, as in the classical bias--variance trade-off \citep{wasserman2004all}. Approaching the interpolation threshold $p\sim nm$, both terms contribute: the variance diverges, producing the peak, while the bias---rather than continuing to decrease---starts to grow \emph{before} the peak. Past the peak the two terms decouple: the variance decays, as in the classical supervised setting \citep{neal2019a,Ascoli_2020}, whereas the bias keeps rising and saturates. The same decomposition for the U-Nets (right panel of Fig.~\ref{fig:bias_variance}) reproduces the key feature: the variance decays after the interpolation peak while the bias increases before peaking. At large $p$, both saturate at a strictly positive plateau.\footnote{For the U-Net, the bias dominates the test error at all $p$. This is expected: the exact score being inaccessible, we estimate it from the target noises, which shifts the bias by a $p$-independent constant. See Appendix~\ref{app:num:bias_var_GMM}.} Further supporting experiments on Gaussian mixture data and simpler feedforward network models of the score reproduce this behavior and can be found in Appendix~\ref{app:num:bias_var_GMM}.
In all cases, the residual test error at large $p$ is a combination of comparable contributions from the variance and the bias.
Both are due to memorization and ``benign overfitting'' of the empirical score: variance is high because of dependence on the precise drawing of the training set, and the bias is high because the empirical score is different from the population score. The two bias regimes are of different nature. At small $p$, it is an \emph{approximation error}: the model is too simple to represent the score, so the bias decreases as the capacity grows. At large $p$, the bias is instead a \emph{property of the objective}: the training drives the estimator towards the minimizers of the empirical test loss, which differ from those of the population one, a gap that no amount of capacity can close. \looseness=-1

%% file: Numerics.tex
\subsection{Numerical Method}
\label{sect:numerical}

We train DDPM-style \citep{ho2020} U-Net architectures \citep{Ronneberger2015} to predict the noise on $n$ (from $512$ to $8192$) CelebA \citep{CelebA} grayscale images downsampled to $32\times 32$. Each image $\vx^\nu$ is assigned $m$ fixed noise realizations $\{\vxi^{\nu\mu}\}_{\mu=1,\ldots, m}$, yielding a total of $n\times m$ frozen training pairs. The architecture of the U-Net follows \citet{bonnaire2025diffusionmodelsdontmemorize}, except it has four resolution levels with channel multipliers $1, 2, 4, 4$ relative to a base width $W$ that we vary from $W=2$ ($p\approx29$K parameters) to $192$ ($p\approx240$M parameters). Note that with such an architecture, $p$ scales as $W^2$. 
We discretize diffusion time on a nonuniform grid of $T=1000$ points, specified in Appendix~\ref{app:num:training}. Each model is trained to predict the noise $\vxi^{\nu\mu}$ from interpolated samples $\vx_t^{\nu\mu}
= e^{-t}\vx^\nu + \sqrt{1-e^{-2t}}\,\vxi^{\nu\mu}$ by minimizing the standard DDPM MSE loss $(nmd)^{-1} \sum_{\mu\nu} \lVert \vxi_\vtheta(\vx_t,t)-\vxi^{\nu\mu} \rVert^2$ using the Adam optimizer for $\tau_\mathrm{max}=2$M steps. At each optimization step, a diffusion time is sampled independently for each pair, uniformly over the eligible grid points. \looseness=-1 

%% file: Analytical_results.tex
\subsection{Analytical Method}
\label{sect:Analytical}

\paragraph*{Setting.} We model the score function with a Random Features Neural Network \citep[RFNN,][]{Rahimi_2007}: \looseness=-1
\begin{align}
    \vs_{\vA}(\vx)=\frac{\vA}{\sqrt{p}}\sigma\!\left(\frac{\vW\vx}{\sqrt{d}}\right).\label{eq:score}
\end{align}
An RFNN is a two-layer neural network whose first layer weights ($\vW\in \mathbb{R}^{p\times d}$) are drawn from a Gaussian distribution and remain frozen while the second layer weights ($\vA\in\mathbb{R}^{d\times p}$) are learned during training. The activation function $\sigma$, which is applied element-wise, is assumed to admit a Hermite expansion and to satisfy $\mathbb{E}[\sigma(z)]=0$. Here, we study the empirical risk minimizer of the DSM loss Eq.~\ref{eq:train_loss} with ridge regularization $\frac{\Delta_t\lambda}{pd}\|\vA\|_F^2$, trained on a dataset composed of $n$ training data points $\vx^\nu\sim\mathcal{N}(0,\vI_d)$ each corrupted $m$ times by $\vxi^{\nu\mu}\sim\mathcal{N}(0,\vI_d)$. Most of our results can be extended to any zero-mean sub-Gaussian distribution with extensive trace covariance (see Appendix~\ref{app:subgaussian}).  We work in the high-dimensional setting $p,n,d\gg 1$ with $n/d\to \psi_n$, $p/d\to \psi_p$ and $\psi_n,\psi_p,m=O(1)$. We derive a tight characterization of the asymptotic learning curves of the losses Eq.~\ref{eq:train_loss}, Eq.~\ref{eq:m_infty_loss}, and Eq.~\ref{eq:test_loss}.
This model has already been studied in the context of diffusion \citep{george_2025, bonnaire2025diffusionmodelsdontmemorize} for $m\in\{1,\infty\}$ and recently for linear activation $\sigma(x)=x$ \citep{farghly2026benign}. Note that the number of adjustable parameters of the RFNN model (\ref{eq:score}) is $pd$. \looseness=-1

\paragraph*{Gaussian Equivalence Principle.} To derive asymptotic closed-form expressions for the losses, we build on the \emph{Gaussian Equivalence Principle} \citep[GEP,][]{Pennington_2017, Peche2019,  mei2020,goldt2020modeling, Gerace_2020, goldt_2021, hu2023}, which states that the nonlinear features $\vF=\sigma(\vW\vY/\sqrt{d})$ can be replaced, in the high-dimensional regime, by a linear surrogate matching their first two moments: $\vF^{\mathrm{GEP}} = \mu_1\,\vW\vY/\sqrt{d}+\mu_*\,\vOmega$, with $\mu_1=\mathbb{E}[\sigma(z)z]$, $\mu_*^2=\mathbb{E}[\sigma^2(z)]-\mu_1^2$ and $\vOmega$ a Gaussian tensor independent of $\vW$ and $\vY$ with covariance $\mathbb{E}[\Omega^{\nu\mu}_\alpha\Omega^{\nu'\mu'}_\beta]=\delta_{\alpha\beta}\,\delta^{\nu\nu'}(\delta^{\mu\mu'}+\kappa(1-\delta^{\mu\mu'}))$. The constant $\kappa$ captures the correlation between features of different noise copies of the same sample and reads \looseness=-1
\begin{equation}\label{eq:kappa}
  \kappa=\frac{1}{\mu_*^2}\,\mathbb{E}_{u,v,w}\big[\big(\sigma(e^{-t}u+\sqrt{\Delta_t}\,v)-\mu_1 e^{-t}u\big)\big(\sigma(e^{-t}u+\sqrt{\Delta_t}\,w)-\mu_1 e^{-t}u\big)\big],
\end{equation}
with $u,v,w\sim\mathcal{N}(0,1)$ independent (see Appendix~\ref{app:gep} for more details).\looseness=-1

\paragraph*{Decomposition of the losses, bias and variance.} The GEP allows us to express the losses as rational polynomials of random matrices. Define the resolvent and the two traces \looseness=-1
\begin{align}\label{eq:resolvent}
  \mathcal{G}_{z,\zeta,\epsilon} = \left(\frac{\vF\vF^T}{nm}-z\vI_p-\zeta\frac{\vW\vW^T}{d}+\epsilon\frac{\vH\vH^T}{n}\right)^{-1}, \quad \vH=\mathbb{E}_{\vxi}[\vF],
\end{align}
\begin{align}\label{eq:traces}
  T_1(z,\zeta,\epsilon) = \frac{1}{d(nm)^2}\Tr\!\left(\vxi\vF^T\mathcal{G}_{z,\zeta,\epsilon}\vF\vxi^T\right),\qquad
  T_2 = \frac{1}{dnm}\Tr\!\left(\vxi\vF^T\mathcal{G}_{-\lambda,0,0}\frac{\vW}{\sqrt{d}}\right).
\end{align}
Using the GEP, the three losses and the bias--variance decomposition Eq.~\ref{eq:bias_variance} read
\begin{align}
  \Ltrain &= 1-T_1(-\lambda,0,0),\label{eq:loss_trainm}\\
  \Ltest &= 1-2\mu_1\sqrt{\Delta_t}\,T_2
    +\mu_1^2\,\partial_\zeta T_1\big|_{(-\lambda,0,0)}
    +\mu_*^2\,\partial_z T_1\big|_{(-\lambda,0,0)},\label{eq:loss_test}\\
  \Linfty &= 1-2\mu_1\sqrt{\Delta_t}\,T_2
    -\partial_\epsilon T_1\big|_{(-\lambda,0,0)}
    +\Delta_t\mu_1^2\,\partial_\zeta T_1\big|_{(-\lambda,0,0)}
    +\mu_*^2(1-\kappa)\,\partial_z T_1\big|_{(-\lambda,0,0)},\label{eq:loss_traininf}\\
  \mathcal{B}^2 &= \left(1-\frac{\mu_1 T_2}{\sqrt{\Delta_t}}\right)^{\!2},
  \qquad
  \mathcal{V} = \frac{\mu_1^2\,\partial_\zeta T_1\big|_{(-\lambda,0,0)}
    +\mu_*^2\,\partial_z T_1\big|_{(-\lambda,0,0)}}{\Delta_t}
    -\frac{\mu_1^2 T_2^2}{\Delta_t}.\label{eq:bias_variance_traces}
\end{align}
The traces $T_1$ and $T_2$ can be computed asymptotically, via a pair of auxiliary order parameters $(q,r)$.

\begin{thm}[Asymptotic characterization of the traces]\label{thm:main}
For $(z,\zeta,\epsilon)\in\mathbb{C}^3$, define the Stieltjes transforms \looseness=-1
\begin{equation}\label{eq:qr-def}
  q(z,\zeta,\epsilon)=\frac{1}{p}\Tr\!\left(\mathcal{G}_{z,\zeta,\epsilon}\right),\qquad
  r(z,\zeta,\epsilon)=\frac{1}{p}\Tr\!\left(\frac{\vW^T}{\sqrt{d}}\,\mathcal{G}_{z,\zeta,\epsilon}\,\frac{\vW}{\sqrt{d}}\right).
\end{equation}
Then, as $d\to\infty$: \emph{(i)} $(q,r)$ concentrate and solve a system of algebraic equations (see Appendix~\ref{app:self-consistent}); \looseness=-1 \emph{(ii)} $T_1(z,\zeta,\epsilon)=f_1\big(q(z,\zeta,\epsilon),\,r(z,\zeta,\epsilon),\,\epsilon\big)$ and $T_2=f_2\big(q(-\lambda,0,0),\,r(-\lambda,0,0)\big)$, for explicit functions $f_1,f_2$ (see Appendix~\ref{app:comp-traces}). \looseness=-1 

\end{thm}

The limiting equations in the overparameterized regime $\psi_p\to\infty$ used in Fig.~\ref{fig:RF_UNET_three_losses} (top right panel), and in the infinite-noise regime $m\to\infty$, used in Fig.~\ref{fig:effect_m}, are given in Appendix~\ref{app:psip_infty_limit} and Appendix~\ref{app:m_infty_limit} respectively. In the following figures, when presenting RFNN asymptotic results, we will also show simulations of the RFNN at finite $d$ to illustrate the convergence toward the asymptotic limit.\looseness=-1

%% file: Discussion.tex
\section{Discussion}
\label{Sect:Discussion}

\begin{figure}
    \centering
    \includegraphics[width=0.62\linewidth]{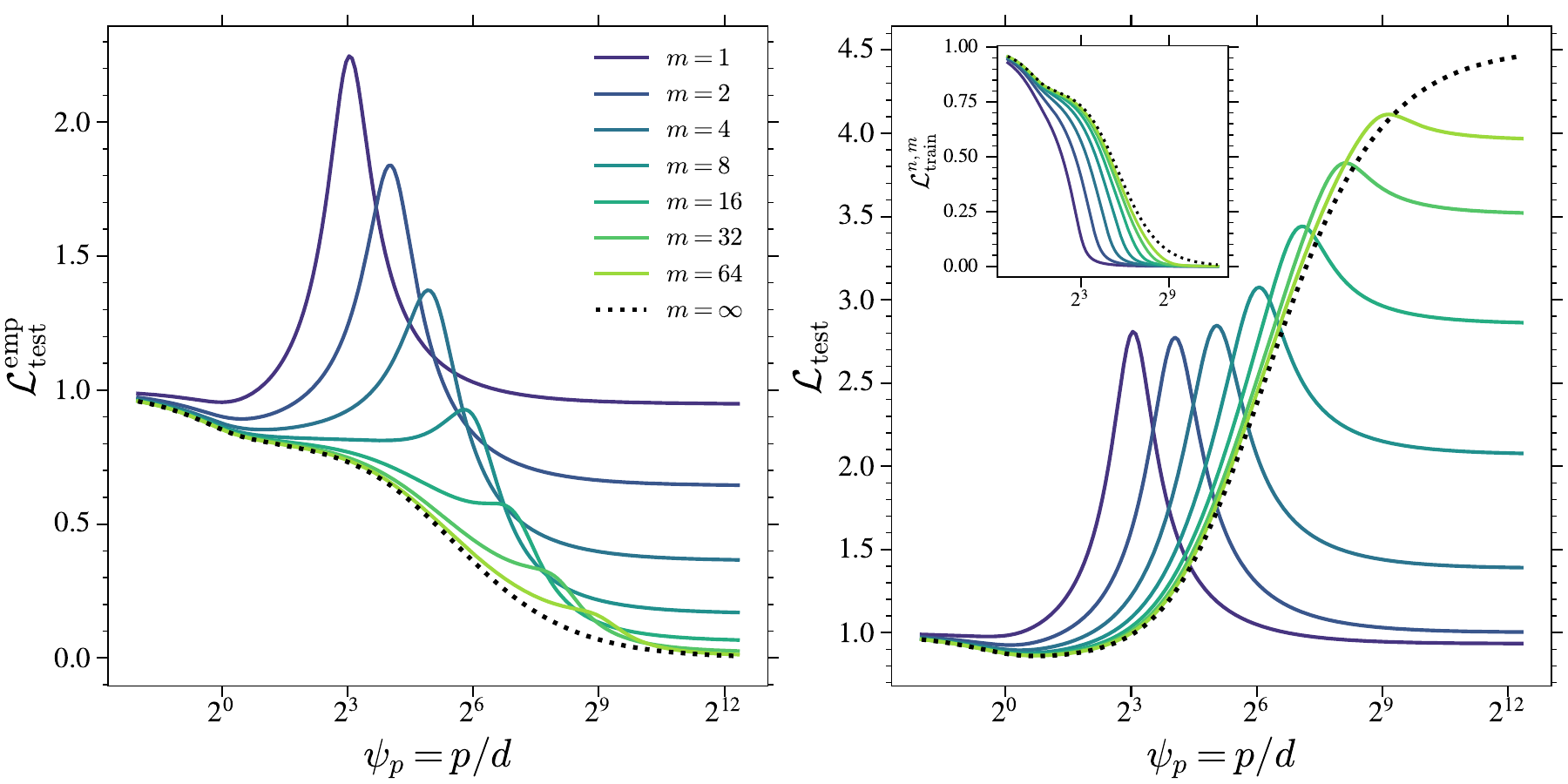}
     \includegraphics[width=.307\linewidth]{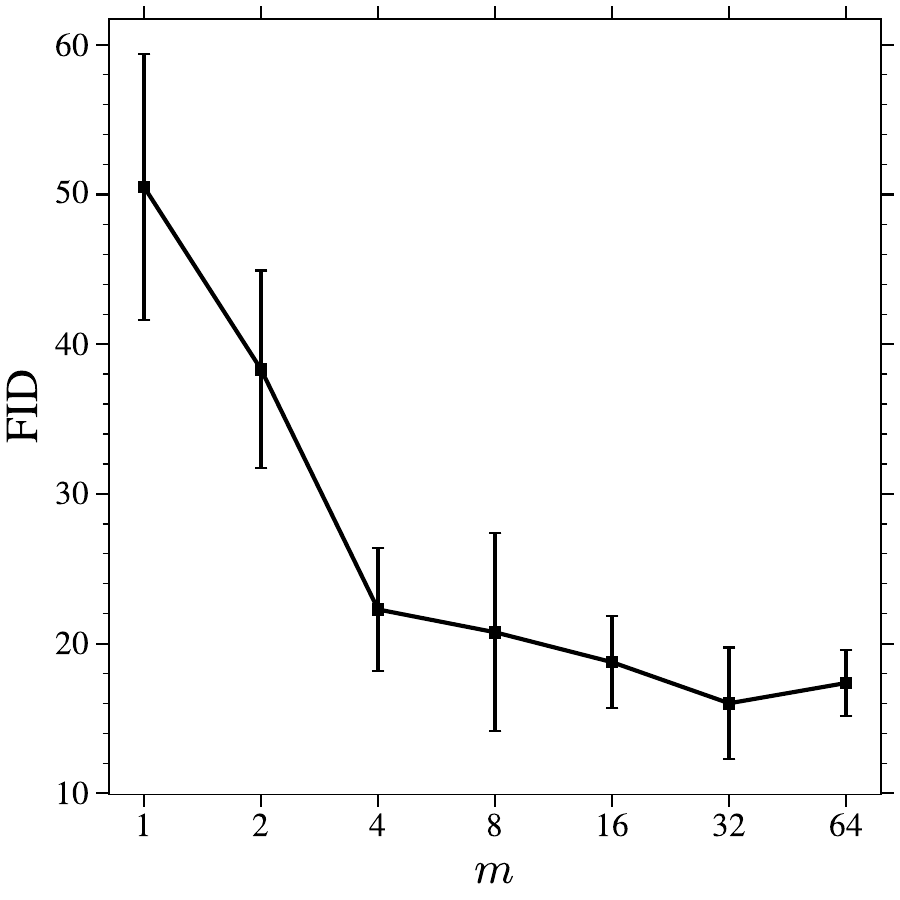}

    \caption{
    \textbf{Effect of $m$ on the double descent.} \emph{(Left)} Empirical test loss $\Linfty$ and \emph{(Middle)} test loss $\Ltest$ vs $\psi_p$ for an RFNN at several $m$, with $\psi_n=n/d=8$, $t=0.1$, $\lambda=10^{-3}$, $\sigma=\tanh$; inset: training loss. \emph{(Right)} FID vs $m$ for early-stopped U-Nets trained on CelebA at $n=16384$, $W=32$; error bars are $\pm3$SE over 5 seeds. \looseness=-1
} 
\label{fig:effect_m}
\end{figure}

\paragraph*{From benign to malign: the role of $m$.}

We now characterize how $m$ controls the interpolation threshold and the nature
of overfitting, focusing on the RFNN setting. Recall that $m=1$ corresponds to
standard regression, where each sample carries a single target, while practice
operates at $m\gg1$ (large number of epochs). For $m=1$, the middle panel of Fig.~\ref{fig:effect_m}
shows an interpolation peak in $\Ltest$ at $\psi_p=\psi_n$;
after the peak, $\Ltest$ decreases below its underparameterized minimum: overfitting is here \emph{benign}. Upon increasing $m$, the interpolation threshold shifts to $\psi_p=m\psi_n$. As $m$ increases, $\Ltest$ decreases for $\psi_p<\psi_n m$ but increases for
$\psi_p>\psi_n m$: overfitting turns from benign to malign. This effect is already visible at $m=2$. In the limit $m\to\infty$, the double descent
disappears altogether, and $\Ltest$ has a single minimum, located in the
underparameterized regime \citep{george_2025}. The left panel shows the evolution of $\Linfty$, which also shows a standard double descent for $m=1$. Increasing $m$ lowers both the peak height and the loss throughout the $\psi_p>m\psi_n$ regime; for large $m$ the peak fades away and $\Linfty$ becomes strictly decreasing, converging to $0$ as $\psi_p\to\infty$. Note that already for moderate values of $m$, $\Linfty$, which measures the distance to the empirical score, reaches very small values for large $p$.
Given the harmful effect of increasing $m$ on the test loss, one might be
tempted to conclude that it is preferable to work at $m=1$. 
In practice, however, one works at $m\gg1$, i.e.\ many epochs, and for good
reason: once regularized, the model performs better at larger $m$, as shown in the right panel of
Fig.~\ref{fig:effect_m}: for early-stopped U-Nets trained on $n=16384$ images with $W=32$, the FID decreases swiftly as $m$ increases, before saturating. The effect of regularization is discussed further below. \looseness=-1

%% file: Tau.tex
\begin{figure}
    \centering
    
    \includegraphics[width=0.45\linewidth]{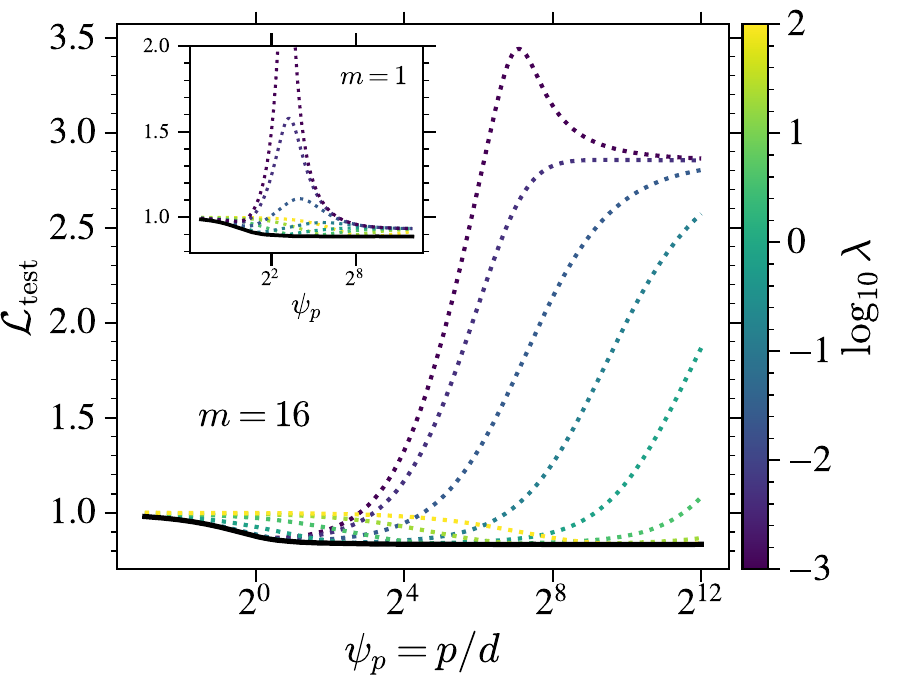}
    \includegraphics[width=.45\linewidth]{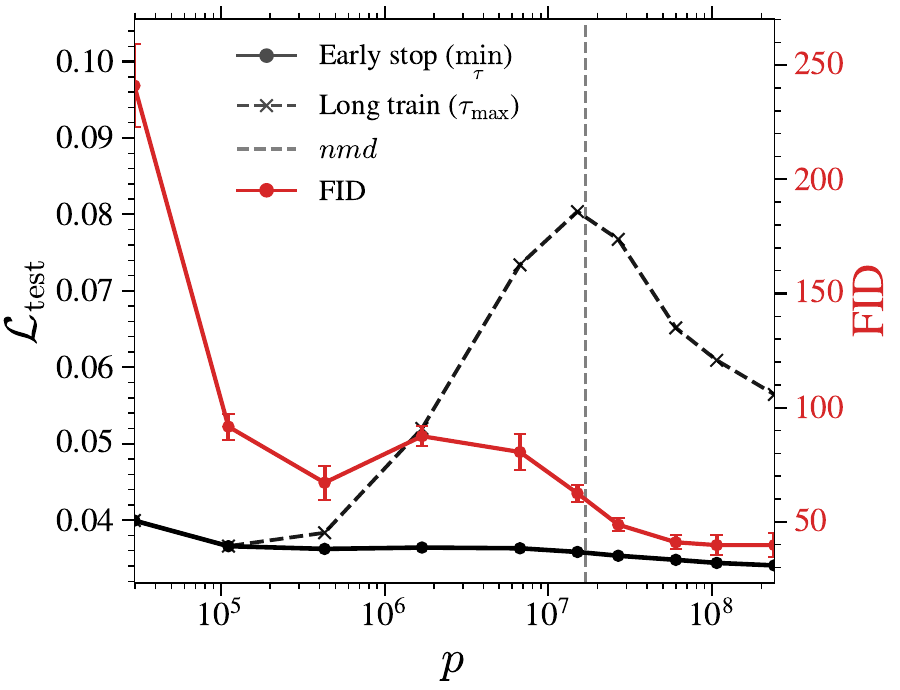}

    \caption{
     \textbf{Benefits of regularization on generalization.} \emph{(Left)} RFNN test loss $\Ltest$ vs. $\psi_p$ for several $\lambda$, at $\psi_n=8$, $t=0.1$, $m=16$, $\sigma=\tanh$ (inset: $m=1$); the black curve is the lower envelope. \emph{(Right)} U-Net integrated test loss vs. $p$ at $n=2048$, $m=8$, early-stopped (solid) and fully trained (dashed), with the FID of the early-stopped models (red, right axis); the vertical dashed line marks the naive interpolation threshold $p=nmd$. \looseness=-1
}
     
    \label{fig:effect_regularization}
\end{figure}

%% file: practice.tex
\paragraph*{Overparameterize, but regularize.} In supervised learning, regularization can drastically alter the double descent picture: optimally tuned ridge regularization is known to flatten the interpolation peak and induce a monotone test error \citep{nakkiran2021optimal, mei2020}. We now show that the same holds in the diffusion setting, and in particular that overparameterizing is beneficial when paired with regularization. The left panel of Fig.~\ref{fig:effect_regularization} shows that, at small regularization (dark blue dotted curve), the RFNN still displays malign overfitting. If instead $\lambda$ is optimized for each $p$, which yields the lower envelope of the family of curves (solid black line), then $\Ltest$ decreases monotonically with $\psi_p$: an optimally regularized large model outperforms any smaller unregularized one. We discuss the effect of $t$ and the optimal regularization in Appendix~\ref{App:analytical_discussion_t}.
For U-Net models trained on CelebA with no explicit regularization, early stopping plays a role analogous to the ridge penalty \citep{ali2019_earlystopping}. \citet{bonnaire2025diffusionmodelsdontmemorize} and \citet{Favero2025_bigger} identify two characteristic timescales: a generalization timescale $\tau_{\mathrm{gen}}=O(1)$ and a memorization timescale $\tau_{\mathrm{mem}}=O(n)$, beyond which the model starts to collapse onto individual training samples; stopping within the window $\tau_{\mathrm{gen}}\ll\tau^\star\ll\tau_{\mathrm{mem}}$ yields high-quality samples while avoiding memorization. Our numerical results complement this picture by showing that the test loss of optimally early-stopped models ($\tau^\star=\arg\min_\tau \Ltest(\tau)$, right panel of Fig.~\ref{fig:effect_regularization}) no longer exhibits the interpolation peak nor the associated double descent. It decreases instead monotonically with the number of parameters. The FID of these early-stopped models does still exhibit a double descent; we conjecture this reflects the test loss being an imperfect proxy for sample quality, and that early-stopping on the FID itself would remove it (see Appendix~\ref{app:num:fid_tau} for a lengthier discussion). In particular, an overparameterized early-stopped model systematically beats any fully-trained underparameterized one: overparameterization remains beneficial once paired with early stopping. \looseness=-1

%% file: Conclusion.tex
\section{Conclusion, Limitations and Future Work}

In this work, we have shown that diffusion models do exhibit \emph{double descent} but with crucial differences from the classical supervised regression setting. 
The interpolation threshold shifts from $p\sim n$ in regression to $p\sim nm$ in diffusion, where $m$ is the number of noise realizations per training sample. The peak vanishes as $m\to\infty$, but the rise of the test loss beyond its minimum, and the bias mechanism behind it, do not. 
Since $m\gg1$ in practice, the peak is pushed to very large model sizes and the overfitting observed in practice, corresponding to the rising branch of a U-shaped curve \citep{marion2026understandingdiffusionmodelsrequires, farghly2026benign}, is the approach to it, where
malign overfitting is already at work.
A bias--variance decomposition traces this malign overfitting back to a persistent bias of the score estimator, which grows past the interpolation peak and saturates at a plateau. Finally, we establish that an optimally regularized overparameterized model outperforms any unregularized one: overparameterization paired with regularization is beneficial despite malign overfitting. \looseness=-1

\paragraph*{Limitations \& future work.} The RFNN and the U-Net differ mainly in the $n$-dependence of the interpolation peak height and in the bias peak seen for the U-Net, both absent from the RFNN. Clarifying the role of feature learning in these effects is an exciting question we leave for future work. Our empirical validation is moreover restricted to a single dataset, moderate number of samples, and models substantially smaller than the state of the art. Establishing our predictions at scale is certainly an important future step. \looseness=-1

%% file: Appendix.tex
\makeatletter
\@ifundefined{proof}{%
  \newenvironment{proof}{\par\noindent\textit{Proof.}\ }{\hfill$\square$\par\medskip}%
}{}
\makeatother

This appendix provides detailed derivations and additional results supporting the main text (MT).  Appendix~\ref{app:discussion_related_works} expands the discussion of closely related works \citep{george_2025, latourellevigeant2026generalizationmemorizationoverfittingdiffusion, marion2026understandingdiffusionmodelsrequires, farghly2026benign}.  Appendix~\ref{app:numerical} provides further details about the numerical experiments carried out in Sect.~\ref{sect:numerical}, as well as additional experiments and discussions.
Appendix~\ref{app:analtical_proofs} gives formal proofs of the main theorems from Sect.~\ref{sect:Analytical}, together with additional results and discussions. We also discuss LLM usage in the conception of this work in Appendix \ref{app:llm_usage}.

\input{Appendix_discussion_related_works}

\input{Appendix_numerical}

\input{Appendix_analytical}

%% file: Appendix_discussion_related_works.tex
\section{Discussion of related works}
\label{app:discussion_related_works}

We discuss here in detail four works close to ours: the precise learning curves of \citet{george_2025}, the lazy regime dynamics of \citet{latourellevigeant2026generalizationmemorizationoverfittingdiffusion}, the empirical study of \citet{marion2026understandingdiffusionmodelsrequires}, and the impossibility result of \citet{farghly2026benign}. The four tackle a similar question along complementary axes. The common thread of our comparison is that all four works operate either at $m=1$ or at $m=\infty$, whereas the phenomenology we report---the interpolation peak at $p\sim nm$, the benign-to-malign transition, and the bias mechanism behind it---require to understand the finite-$m$ interpolation between these two limits.

\paragraph*{Precise learning curves for the RFNN score.} \citet{george_2025} derive asymptotically exact train and test errors using the Gaussian Equivalence Principle and the theory of linear pencils for the same framework as ours: nonlinear RFNN and Gaussian data, studied in the proportional regime of finite $\psi_n = n/d$ and $\psi_p=p/d$. Their central message is a crossover at $p\sim n$ between a regime where the model generalizes and a regime where it approaches the empirical score and memorizes, showing that increasing $m$ both improves generalization for $p\leq n$ and intensifies memorization for $p\geq n$. First, our arbitrary-$m$ solutions reveal the intermediate regime with an interpolation threshold sitting at $p\sim nm$. Second, our dissection of the problem into three losses ($\Ltrain$, $\Linfty$, $\Ltest$) identifies what is actually being overfit and the origin of the malign overfitting.

\paragraph*{Training and generative dynamics in the lazy regime.} \citet{latourellevigeant2026generalizationmemorizationoverfittingdiffusion} study the gradient-flow dynamics of neural networks trained on the score-matching objective in the proportional regime $n\asymp d$, in the lazy regime where the network is linearized around its initialization. The infinite-width limit then replaces the network by a kernel operator $\mathcal{K}$ acting on a vector-valued RKHS,
\begin{align}
    (\mathcal{K}f)(\vx)=\int \dd\vy\; p_t^{\mathrm{emp}}(\vy)\, K(\vx,\vy)\, f(\vy),
\end{align}
with $p_t^{\mathrm{emp}}$ the noised empirical distribution. They derive closed-form equations for the training and test losses along the whole trajectory. Their central message is a separation of training timescales: the model generalizes on $\tau=O(d)$, where $\Ltest$ decreases; it then overfits the empirical score on $\tau=d^{1+\Theta(1)}$, where $\Ltest$ rises again; and it memorizes only on the much later $\tau=d^{\omega(1)}$. Because they also characterize the generated distribution, they can separate these last two stages and conclude that overfitting the empirical score during training does not necessarily entail memorization in the backward dynamics. The differences with our work are twofold. First, they follow the training and generative dynamics, whereas we characterize only the empirical risk minimizer solution. Second, their lazy limit sits at $p\to\infty$ with an infinite number of noise samples $m\to\infty$, whereas our analysis holds at finite $p$ and $m$.

\paragraph*{Empirical study of memorization dynamics.} \citet{marion2026understandingdiffusionmodelsrequires} perform an extensive numerical study of U-Nets trained on CIFAR-10 subsets at various $n$ and $p$ and track their train and test denoising losses. They never observe a parameter-wise peak or double descent in the test loss. Crucially, a fresh noise realization is drawn at every optimization step, so their setting corresponds to our $m\to\infty$ case. Consistently with our findings, there is therefore no peak to see as $p$ increases in this case. What they do observe however is a training-time-wise double descent of the distributional distances with a peak located at the onset of memorization, which is accelerated by the model size. This is consistent with our $\tau$-$W$ discussion reported in Appendix~\ref{app:num:epochs}. We emphasize that the two double descents are distinct phenomena occurring on different axes and with different mechanisms: theirs is dynamical, observed at $m\to\infty$, and only appears in \emph{distribution space}; ours is static, parameter-wise, exists only at finite $m$, and appears in the test loss. Their position that the field should rethink generalization in diffusion models is precisely what our two distinct test losses formalize: the benign overfitting mechanism \emph{does exist}, but it targets the minimizers of $\Linfty$, the empirical score, rather than those of $\Ltest$.

\paragraph*{Impossibility of benign overfitting.} \citet{farghly2026benign} establish, in a model-independent way, that overfitting and
generalization cannot coexist in diffusion models unless the sample size grows exponentially with the data dimension. Their first argument is information-theoretic: as $t\to0$, the relative Fisher information between the empirical and population noised distributions diverges, which is consistent with our picture as it forces the test loss to be large. They also study a linear RFNN and find a U-shaped test loss as $p$ increases. Our analysis refines their conclusions in two ways. First, their setting is again $m\to\infty$: at any finite $m$, double descent \emph{does} occur, with an interpolation peak at $p\sim nm$ that moves to infinite model size as $m\to\infty$, recovering their U-shape. Along this crossover, we show that overfitting is benign only in the regression case $m=1$ and increasingly malign as $m$ grows. The impossibility that they derive is thus the $m\to\infty$ endpoint of a benign-to-malign crossover. 
Second, their finding that time-integrated objective and early-stopping act as implicit regularizers parallels our regularization analysis, to which we add that, once suitably regularized, overparameterization becomes strictly beneficial. Interestingly, they also trace the difference with standard regression to an alignment obstruction: benign overfitting of the minimum-norm interpolator relies on a favorable alignment between the target and the empirical covariances, which the score-matching target cannot reach. This is complementary to our bias mechanism: both analyses locate the malign origin of the overfitting in a \emph{deterministic} distortion of the learned score.

%% file: Appendix_numerical.tex
\section{Additional numerical results} \label{app:numerical}

\subsection{Architecture \& Training} \label{app:num:training}

The model is made of an initial $3\times 3$ convolution mapping to $W$ channels. The encoder has 4 resolution levels with channel multipliers $(1, 2, 4, 4)$, meaning the successive width is $W, 2W, 4W, 4W$ at resolutions $L, L/2, L/4, L/8$. Each level is composed of 2 residual blocks with $3\times 3$ stride-2 convolutions for downsampling. The encoder is followed by a bottleneck of 2 residual blocks and then by a symmetric decoder using nearest-neighbor upsampling and traditional skip connections. A final group-norm convolution projects back to a single channel. We apply multi-head self-attention with 4 heads at the two coarsest levels and in the bottleneck. Each residual block uses GroupNorm and SiLU activations. Finally, the diffusion time is encoded with a standard sinusoidal embedding passed through a 2-layer MLP and added to the features in every residual block. At very small $W=2$, the attention head count and GroupNorm groups are reduced so the layers remain valid, hence the configuration slightly changes at this $W$ compared to the others.

As stated in the main text, each model is trained to predict the noise $\vxi^{\nu\mu}$ by minimizing the standard noise-prediction objective. We use Adam with default moment parameters and without weight decay or exponential moving average.
The forward pass is discretized using $T=1000$ steps, indexed by $k=\{0,\ldots,T-1\}$, with a linear noise schedule $\beta_k$ going from $\beta_0=10^{-4}$ to $\beta_{T-1}=2\times10^{-2}$ and $\bar\alpha_k=\prod_{j=0}^{k}(1-\beta_j)$ \citep{ho2020}. To connect this schedule to the continuous-time notation of the
main text, we define $t_k=-\frac12\log\bar\alpha_k$. Consequently, the noisy samples can equivalently be written as $\vx_{t_k}^{\nu\mu} =\sqrt{\bar\alpha_k}\,\vx^\nu
 +\sqrt{1-\bar\alpha_k}\,\vxi^{\nu\mu}$.
Throughout training, the learning rate is fixed to $10^{-4}$ with a batch size $B=128$. Unlike standard diffusion models training, the noise is here fixed: each of the $n$ images $\vx^\nu$ has $m$ frozen noise vectors $\vxi^{\nu\mu}$, giving a total of $n\times m$ fixed pairs. The training is time-dependent: at every step we pick a diffusion index $k\sim\mathcal{U}\left[1,\dots,T-1\right]$ and we compute the interpolated noisy sample at $t_k$. In the numerical experiments, no explicit regularization is used: the only regularizer we study is the finite training time $\tau$ through early-stopping.

\begin{figure}
    \centering
    \includegraphics[width=.49\linewidth]{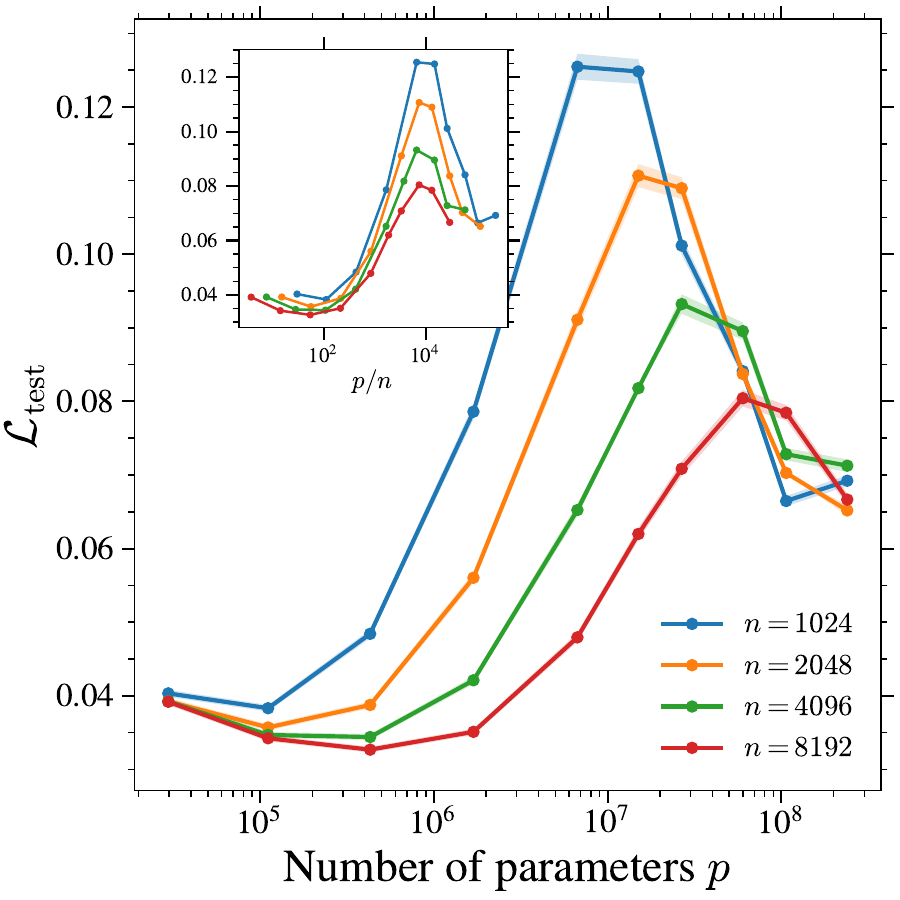}
    \includegraphics[width=.49\linewidth]{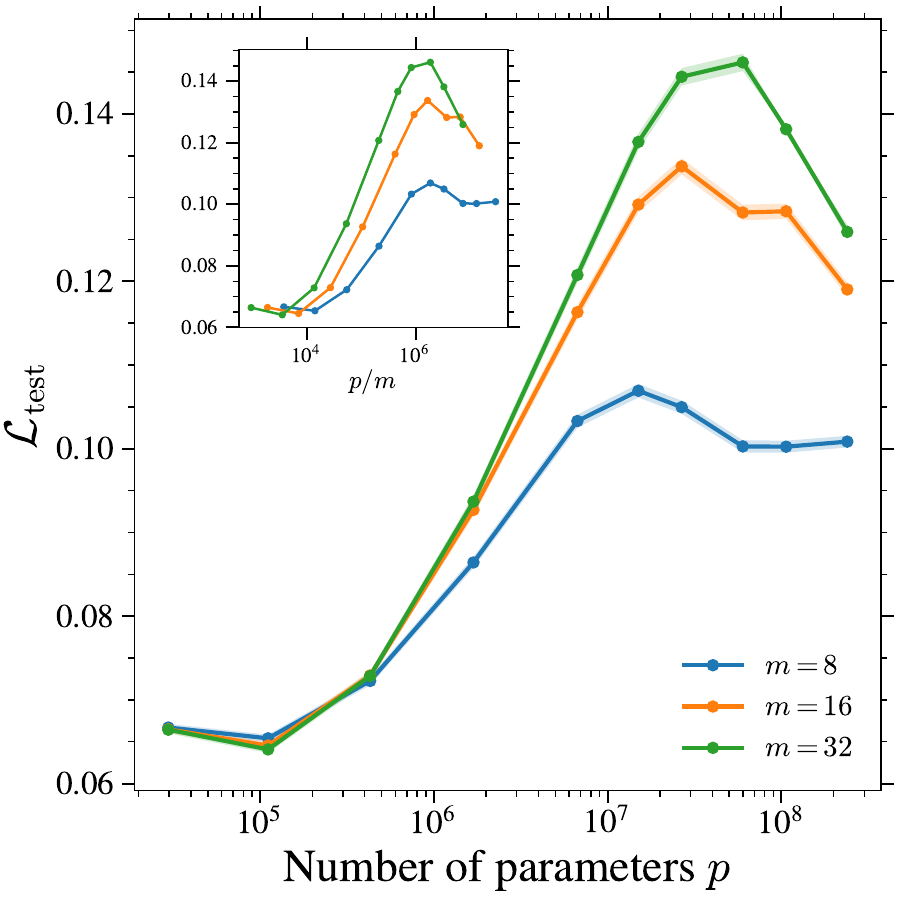}
    \caption{\textbf{Additional scaling experiments.} Evolution of \emph{(Left)} the time-integrated test loss with $p$ for several $n$ at $m=32$, and \emph{(Right)} the $t\approx0.1$ test loss with $p$ for several $m$ at $n=2048$. Both plots are obtained at $\tau=\tau_\mathrm{max}=2$M steps.} 
    \label{fig:app:nm_scalings}
\end{figure}

\subsection{More scaling experiments of the double descent with $n$ and $m$} \label{app:num:scaling_nm}

In the right panel of Fig.~\ref{fig:scaling}, we show the evolutions of the test loss for several $n$ and $m$ at once. In Fig.~\ref{fig:app:nm_scalings} we provide complementary experiments with $p$ for more $n$ and $m$ values, shown individually. In this case, we clearly see the collapse of the (time-integrated) test loss at both small and large $p$, independently of $n$ in the left panel, with the clear scaling of the peak with $n$ in the inset. The right panel shows three values of $m$ for a fixed $n=2048$ training. At small $p$, all curves collapse independently of $m$ while the height of the large $p$ plateau is increasing with it. The inset also shows the $p/m$ evolution, showing the scaling of the peak with $m$ as well. Together, these plots emphasize the scaling of the double descent peak with $nm$, as in the random features case.

\subsection{Training-time-wise double descent in U-Net diffusion models} \label{app:num:epochs}

\paragraph*{$\tau$--$W$ coupling of the double descent.}
The left panel of Fig.~\ref{fig:app:epoch_wise} traces the interaction between model size (through the width $W$) and training time in the double descent phenomenon: each curve shows $\Ltest$ as a function of $W$ at fixed number of updates $\tau$. A double descent emerges only for large enough $\tau$, and its peak is not static, slipping towards smaller $W$ as $\tau$ grows. This couples $W$ and $\tau$ as two substitutable axes of the double descent phenomenon, echoing the view of the supervised learning setting from \citet{Nakkiran_2021} in which training time and model size feed a single effective complexity measure governing the peak location, although our analysis does not provide any insight on the nature of the origin of these two peaks.

\paragraph*{Training-time double descent.} In the middle panel of Fig.~\ref{fig:app:epoch_wise}, we show the evolution of the test loss with the number of Adam updates $\tau$ for different $W$. It reveals three regimes. For small-size models, the test loss decreases monotonically to a minimum value reached at the end of training. For intermediate-size models ($16\leq W \leq 64$) the test loss reaches a minimum and then rises again at a moderate number of updates, the signature of overfitting and memorization at large training times. Larger models ($W\geq 96$), however, exhibit a \emph{training-time-wise double descent}: the test loss quickly decreases to a minimum ($\tau \approx 10^4$), rises, and then decreases again at even larger $\tau$. Strikingly, the first two regimes with the descent and the subsequent rise are largely robust to $W$ with all the curves collapsing, whereas the second descent clearly depends on the width, with larger models exhibiting a sharper late-time decrease. The right panel actually informs us about the mechanism behind this rise at large $W$: the gap between $\Linfty$ and $\Ltrain$ grows with both $\tau$ (at fixed $W$) and $W$ (at fixed $\tau$), showing that the model progressively fits the specific training noise realizations rather than the underlying score, meaning it is heading towards memorization and the overfitting of $\Linfty$ is detrimental. These three regimes were also observed in the supervised setting \citep{Nakkiran_2021} but with a crucial difference: there, the large-model second descent is persistent so that longer trainings decrease the test loss, whereas in diffusion models, every sufficiently expressive model eventually turns back towards memorization \citep{bonnaire2025diffusionmodelsdontmemorize}. This marks a clear departure from the epoch-wise double descent in supervised learning.

\begin{figure}
    \centering
    
    \includegraphics[height=.3\linewidth]{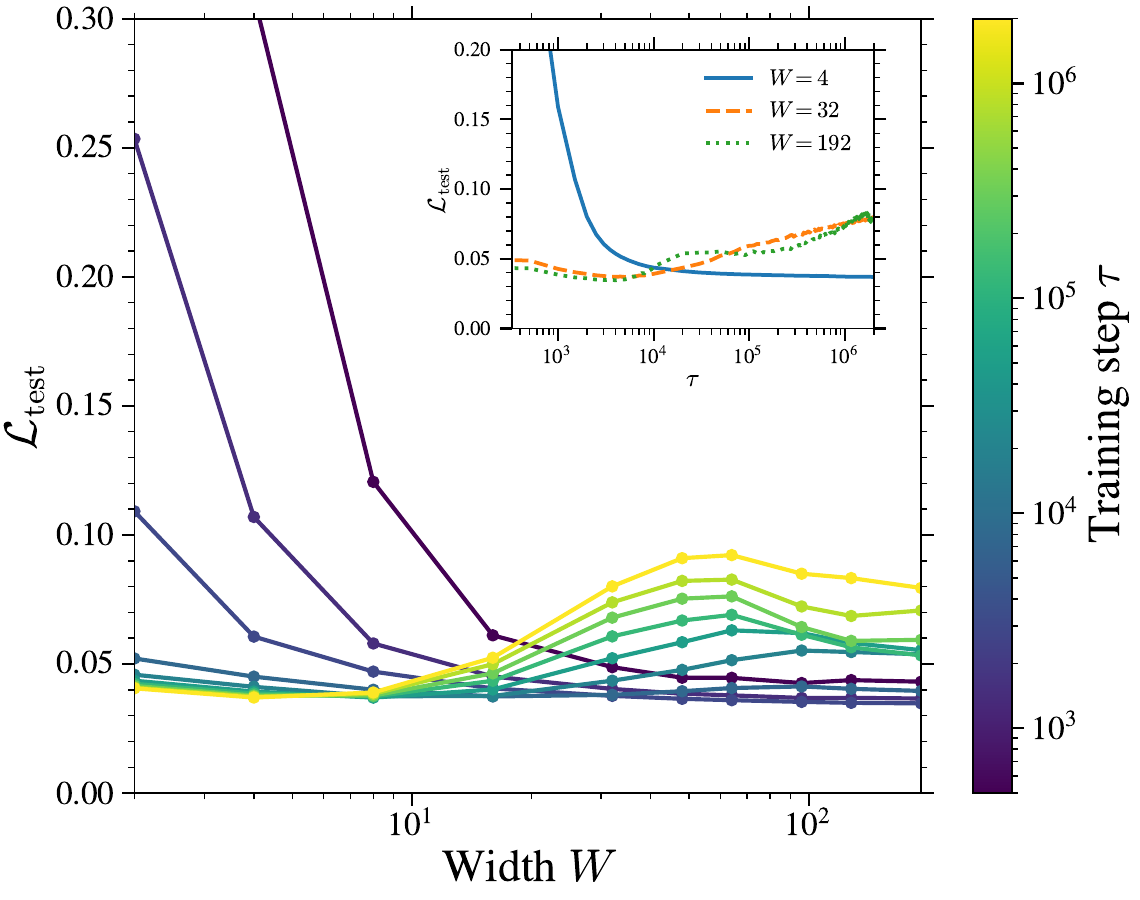} 
    \includegraphics[width=.6\linewidth]{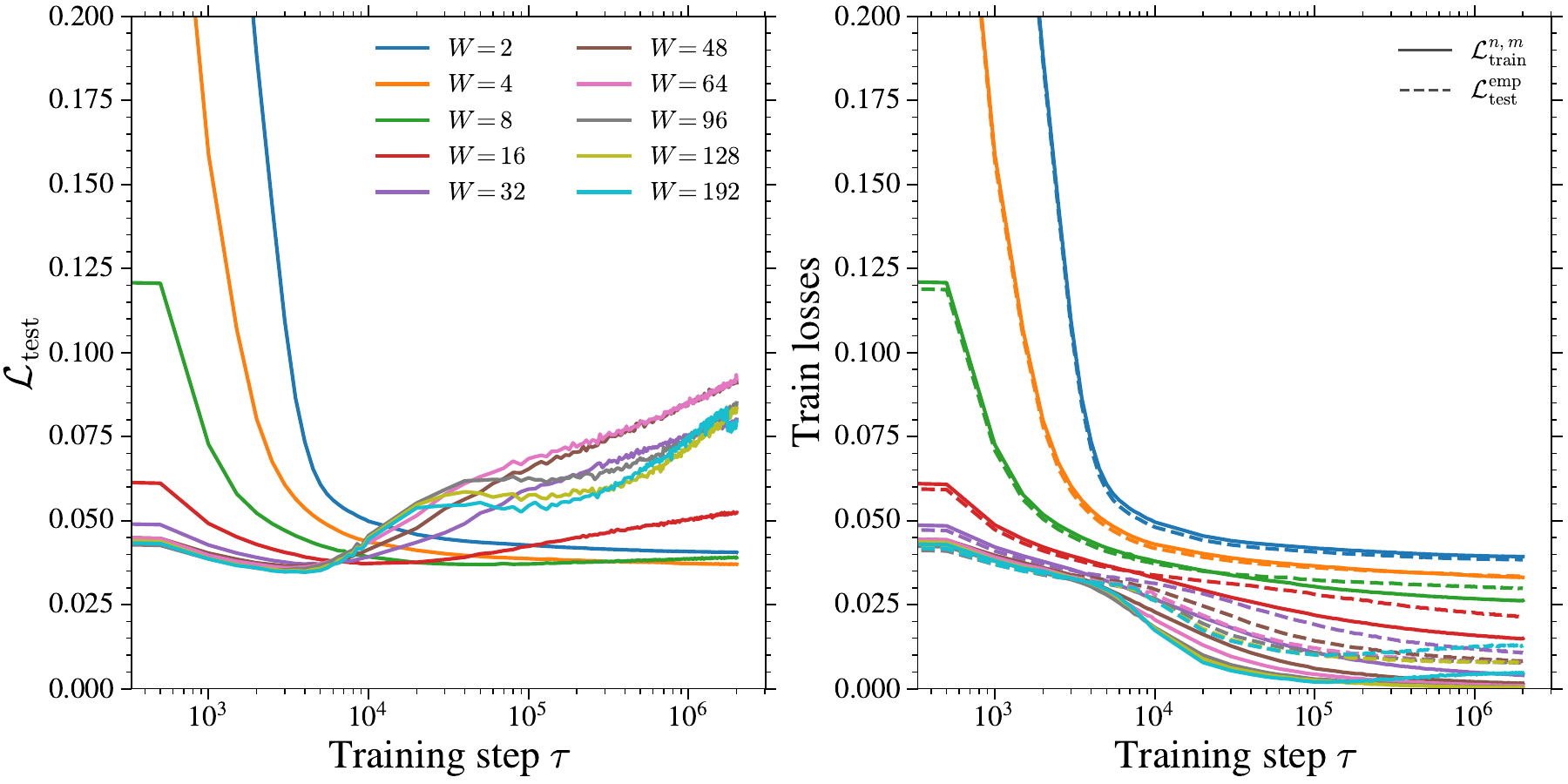}
    \caption{\textbf{Training-time double descent in diffusion models.}
    \emph{(Left)} Evolution of the integrated test loss with $W$ for several values of training times $\tau$. The middle and right panels show the evolution of the time-integrated \emph{(Middle)} test and \emph{(Right)} train losses against $\tau$ for several $W$. All the models were trained with $n=2048$ and $m=8$.}
    \label{fig:app:epoch_wise}
\end{figure}

\subsection{Visual inspection of the generated samples} \label{app:num:visual}

\paragraph*{Sample generation.} For each configuration $(n, m, W)$, we generate from the early-stopped model at $\tau^\star =\arg\min_\tau \Ltest(\tau)$ that minimizes the time-integrated test loss. From this checkpoint, we draw $10^4$ samples with the deterministic DDIM sampler from \citet{song2022DDIM} ($\eta=0$, $T'=100$ steps). All models are initialized with the same Gaussian noise seed so that a given position in the grid of Fig.~\ref{fig:app:visual} corresponds to the same starting point for all models. Samples are then de-normalized to the original range and the quality is measured using the Fréchet-Inception Distance \citep[FID,][]{heusel2017gans} against $10^4$ CelebA held-out test images that no models have seen during training, and using the \texttt{PyTorch-FID} package\footnote{Available at \url{https://github.com/mseitzer/pytorch-fid}.}.

\paragraph*{Results.} In Fig.~\ref{fig:app:visual}, we show for $n=2048$ a $3\times3$ set of generated samples from the early-stopped models at different $m$ and $W$, with increasing $m$ over the rows and $W$ across the columns. The FID of the full generated set is reported in each panel. Since all panels share the same initial noise, it is easy to visually appreciate the quality variations with $m$ and $W$. Two trends are visible. First, increasing the number of noise realizations $m$ improves the sample quality up to a certain point: the $m=1$ row is clearly the blurriest (FID$\approx 100$), whereas $m=8$ and $m=32$ produce sharper faces at fixed $W$, although the trend is not consistent for all widths. Second, at fixed $m$ (say 8), the FID decreases with the width, reflecting the benefit of additional capacity coupled with regularization, as emphasized in the main text. The only notable exception is $m=1$ where the largest model is \emph{not} the best one. Note that all models shown here are trained with only $n=2048$ images with no data augmentation, hence the modest sample quality. This is intentional to keep $n$ moderate so that the generalization--memorization transition during training \citep{bonnaire2025diffusionmodelsdontmemorize} and the interpolation peak in the test loss remain visible in a reasonable amount of training time and number of parameters.

\begin{figure}
    \centering
    \includegraphics[width=.7\linewidth]{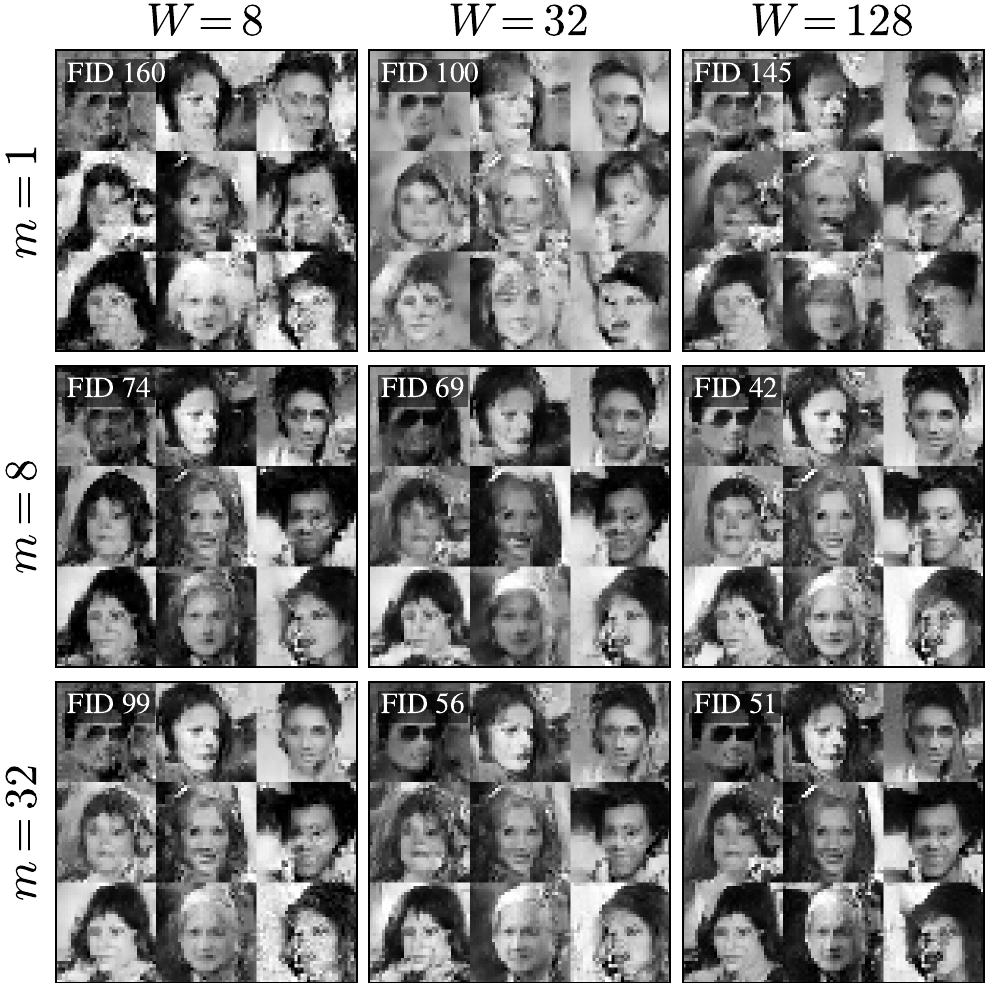}
    \caption{\textbf{Non-curated generated samples.} Example of CelebA generated samples for various combinations of $m$ (across rows) and $W$ (across columns) along with the test FID on the full set of generated images. All the models were trained with the same $n=2048$ data.} 
    \label{fig:app:visual}
\end{figure}

\subsection{More details on the effect of $m$ on U-Nets}

In Fig.~\ref{fig:app:effect_m} we display the evolution of the time-integrated test loss with the training step $\tau$ for several $m$ at fixed $n=16384$ and $W=64$. We deliberately take $n$ large so that the models reach good generalization and operate in a regime closer to practical settings. Every curve follows the dynamics discussed previously: the test loss first decreases to a minimum and then rises again, making early stopping necessary, even at this $n$. The effect of $m$ is clear and monotonic: as $m$ increases, the minimum test loss decreases and is reached at progressively larger $\tau$. This improvement saturates for $m\gtrsim 16$, where the curves essentially collapse onto one another.
Together with the FID measurements in the right panel of Fig.~\ref{fig:effect_m}, this confirms that a larger $m$ is beneficial for generalization, both at the level of the test loss and of the quality of the generated samples.

\begin{figure}[t]
    \centering
    \begin{minipage}[t]{0.48\linewidth}
        \centering
        \includegraphics[width=\linewidth]{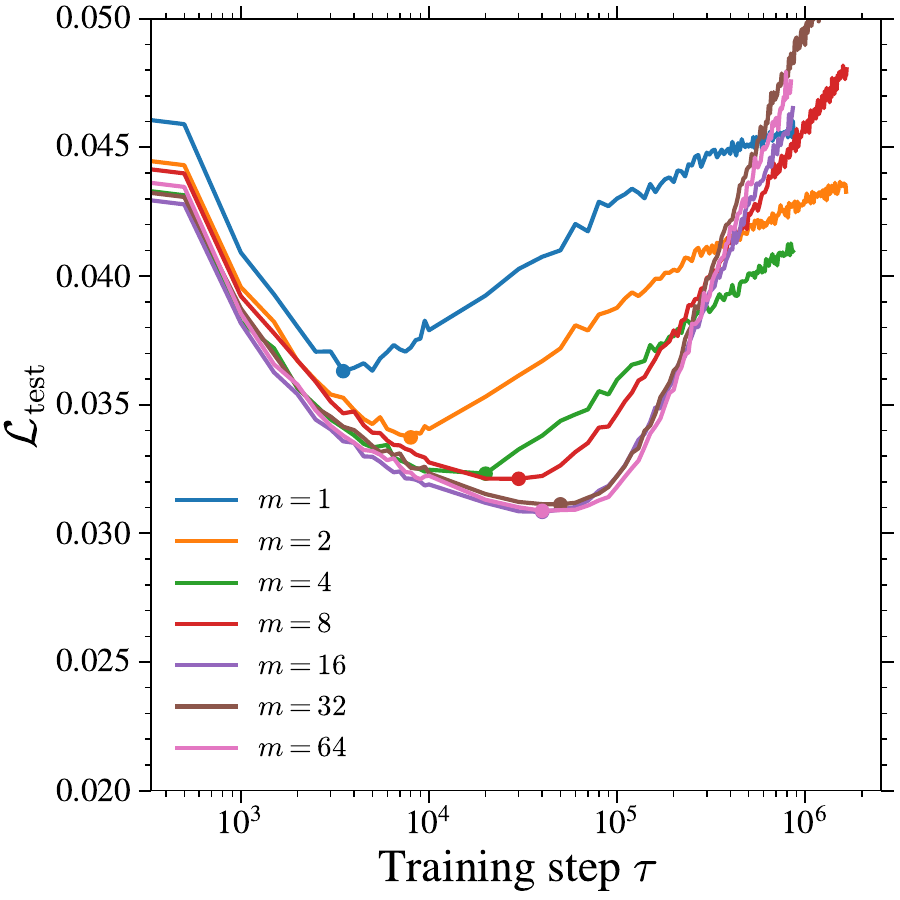}
        \caption{\textbf{Effect of $m$ on the test loss dynamics.} Time-integrated test loss as a function of $\tau$ at fixed $n=16384$ and $W=64$ for several values of $m$. The colored dots mark the per-curve minima, i.e. the early-stopped checkpoints $\tau^\star$ used in the main text.}
        \label{fig:app:effect_m}
    \end{minipage}
    \hfill
    \begin{minipage}[t]{0.48\linewidth}
        \centering
        \includegraphics[width=\linewidth]{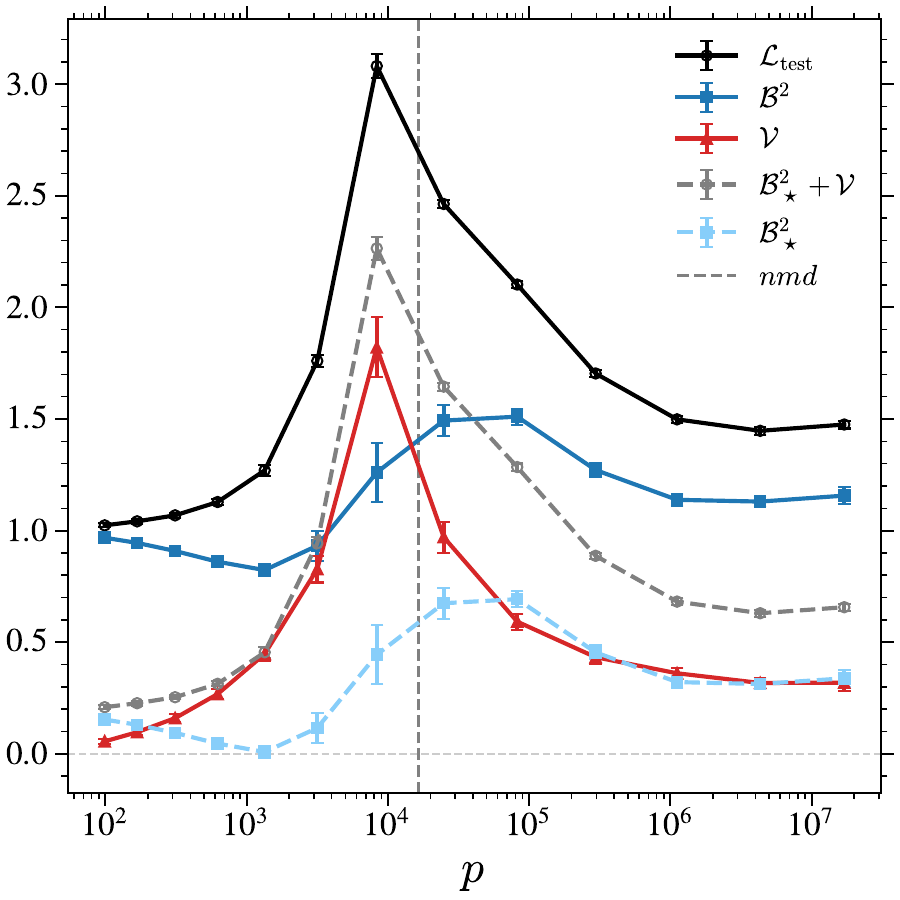}
        \caption{\textbf{Bias--variance decomposition for Gaussian Mixture Model.} Evolution of the bias and variance for a 2-hidden-layer neural network trained on data from a 2-component Gaussian Mixture Model. $\mathcal{B}_\star$ is the population bias computed from the analytically-known score while $\mathcal{B}$ is the one accessible in practice. Here, $n=64$, $m=8$, $d=32$, and $t\approx0.1$.}
        \label{fig:app:bias_var}
    \end{minipage}
\end{figure}

\subsection{Bias--Variance decomposition in the Gaussian Mixture Model} \label{app:num:bias_var_GMM}

\paragraph*{Data and exact score.} We use here a symmetric two-component mixture model such that
\begin{equation}
  \vx_0 \sim \tfrac12\,\mathcal{N}(+\vmu,\vI_d)+\tfrac12\,\mathcal{N}(-\vmu,\vI_d),
  \qquad \vmu=\mu\,\mathbf{1}_d,\quad \mu=1,\ \ \vSigma=\vI_d.
\end{equation}
In this case, the analytical (population) score is known at all $t$ and corresponds to
\begin{equation}
  \vs^\star(\vx,t)=\nabla_{\vx}\log p_t(\vx)=\sqrt{\bar\alpha_t}\,\vmu\, \tanh\!\big(\sqrt{\bar\alpha_t}\,\vmu^\top\vx\big)-\vx.
  \label{eq:gmm_score}
\end{equation}

\paragraph*{Model and training.} We learn the score at a single, fixed diffusion time $t$ with a fully-connected network of $L$ hidden layers of width $W$ and $\tanh$ activations. With
$\vh_0=\vx$ and
\begin{equation}
\vh_\ell=\tanh\!\big(\mathbf{W}_\ell\,\vh_{\ell-1}+\vb_\ell\big),\quad \ell=1,\dots,L,
  \qquad
  \mathbf{W}_1\in\mathbb{R}^{W\times d},\ \ \mathbf{W}_\ell\in\mathbb{R}^{W\times W}\ (\ell\ge2),
\end{equation}
the noise prediction is the linear final read-out
\begin{equation}
\veps_\vtheta(\vx)=\mathbf{W}_{L+1}\,\vh_L+\vb_{L+1}, \qquad \mathbf{W}_{L+1}\in\mathbb{R}^{d\times W},
\end{equation}
so that the model has a total of $p = 2dW + d + LW + (L-1)W^2$ trainable parameters. The network predicts the noise, equivalently the score $\vs_\vtheta=-\veps_\vtheta/\sqrt{1-\bar\alpha_t}$, and is time-agnostic since $t$ is fixed to the discrete step $t=137$, which corresponds to $t\approx0.1$ in the continuous-time notation of the main text. Following the benign-overfitting protocol of the main text, we draw $n$ clean points
$\{\vx_0^i\}_{i=1}^n$ from $P_0$ and, for each, $m$ \emph{fixed} noise realizations
$\{\veps^{ij}\}_{j=1}^m$, yielding $nm$ anchors
$\vx_t^{ij}=\sqrt{\bar\alpha_t}\,\vx_0^i+\sqrt{1-\bar\alpha_t}\,\veps^{ij}$, and
minimize the empirical denoising loss
\begin{equation}
  \hat{\mathcal{L}}(\vtheta)=\frac{1}{nm}\sum_{i=1}^n\sum_{j=1}^m
    \big\|\veps_\vtheta(\vx_t^{ij})-\veps^{ij}\big\|^2
\end{equation}
with Adam for $\tau_\mathrm{max}=100$K steps.

\paragraph*{Estimating $\mathcal{B}^2$ and $\mathcal{V}$.} We train an ensemble of $K=20$ independent models that differ only in their random
seed (data, noises, and initialization), and evaluate all of them on a single frozen test set of $nm$ fresh points $\{\vx_t^\nu\}$ from the same mixture at the same fixed $t$. We fix $L=2$ and vary only the width $W$ to increase the number of parameters in the model. We then estimate empirically the bias, variance and test loss. The bias is computed against two targets: the empirical labels $\veps$ (i.e. the target noises) giving the quantities $\mathcal{B}^2$ and $\Ltest=\mathcal{B}^2 + \mathcal{V}$ as measured in main text; and against the population target $\veps^\star$ from Eq.~\ref{eq:gmm_score}, granting access to population estimates of $\mathcal{B}^2_\star$ and $\mathcal{B}^2_\star+\mathcal{V}$.

\paragraph*{Empirical vs. population targets.} Fig.~\ref{fig:app:bias_var} reports both decompositions. The variance is target-independent, and the two biases differ only by the \emph{irreducible} denoising error
\begin{equation}
  \mathcal{B}^2-\mathcal{B}_\star^2=\mathbb{E}\left[\big\|\veps^\star-\veps\big\|^2\right],
\end{equation}
i.e., a $p$-independent constant. Consequently, the population curves $\mathcal{B}_\star^2$ and $\mathcal{B}_\star^2+\mathcal{V}$ (dashed lines in Fig.~\ref{fig:app:bias_var}) exactly correspond to the empirical $\mathcal{B}^2$ and $\Ltest$ (solid lines) shifted downward by $\mathcal{B}^2-\mathcal{B}_\star^2$. Beyond this effect, we observe that the plot reproduces exactly the non-monotonic behaviors of the bias and the variance that were depicted in the main text for the U-Net architecture trained on CelebA, both saturating to a non-vanishing $O(1)$ value at large $p$. The interpolation threshold, corresponding to where the train loss becomes small and the test loss peaks, is also consistent with $p\approx nmd$.

\subsection{Integrated test loss} \label{app:num:integrated_loss}

The main text reports losses at a fixed diffusion time $t$, which isolates the phenomenon at the noise levels where it is most pronounced. However, the double descent and the peak-and-plateau structure appear across the whole small-$t$ range. Figure~\ref{fig:app:integrated_loss} displays the integrated test loss over all diffusion times\footnote{Note that many figures from the Appendix (e.g., Fig.~\ref{fig:app:nm_scalings} (Left), Fig.~\ref{fig:app:epoch_wise}, Fig.~\ref{fig:app:effect_m} and Fig.~\ref{fig:app:FID_tau}) in fact display the integrated loss.} as a function of the model size, for the same model and settings as the middle panel of Fig.~\ref{fig:RF_UNET_three_losses}. The peak is found at around the same value of $p$ and then plateaus above its minimum. The qualitative picture we draw in the main text is therefore robust to aggregating over diffusion times, although the peak amplitude and the plateau may depend on the precise loss parameterization and the weights associated to diffusion times.

\subsection{Test loss and FID during training} \label{app:num:fid_tau}

\begin{figure}[t]
    \centering
    \begin{minipage}[t]{0.48\linewidth}
        \centering
        \includegraphics[width=\linewidth]{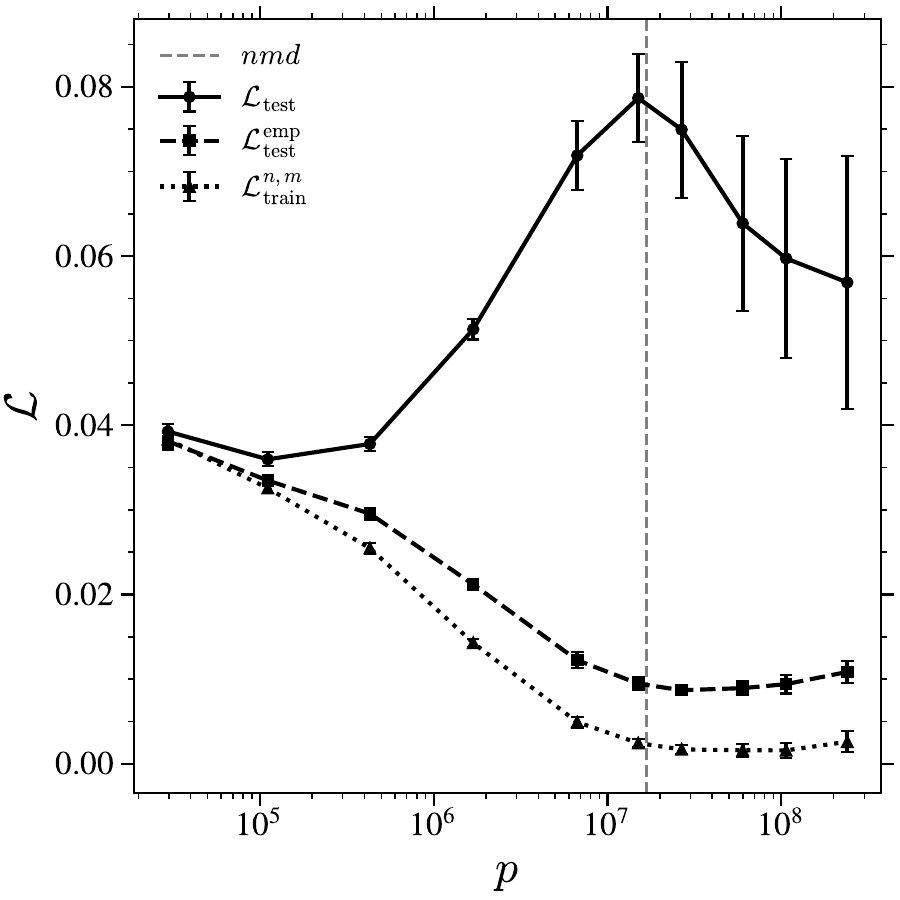}
        \caption{\textbf{Integrated losses behavior.} Losses versus $p$ for U-Net models trained on CelebA at fixed $n=2048$, $m=8$. Error bars correspond to $\pm3$SE obtained from 6 random initial seeds. The vertical dashed line corresponds to the naive interpolation threshold $nmd$.}
        \label{fig:app:integrated_loss}
    \end{minipage}
    \hfill
    \begin{minipage}[t]{0.48\linewidth}
        \centering
        \includegraphics[width=\linewidth]{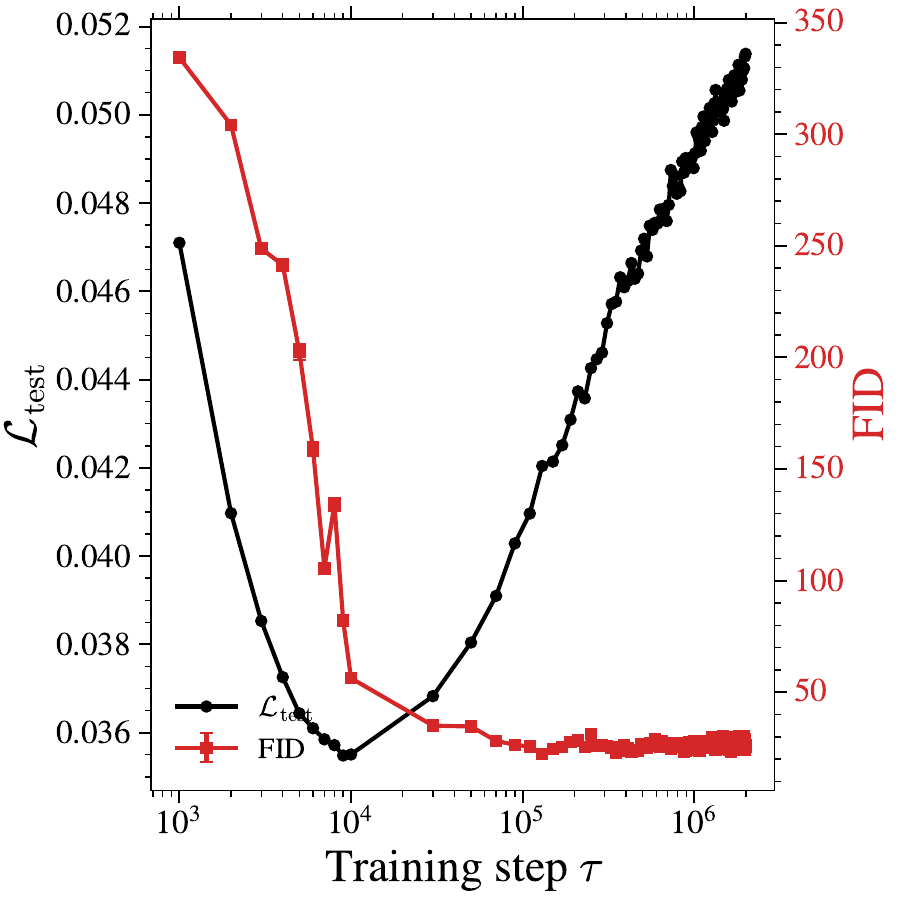}
        \caption{\textbf{Test loss and FID with $\tau$.} Time-integrated test loss (black, left axis) and FID (red, right axis) as a function of $\tau$ at fixed $n=2048$, $m=8$ and $W=16$. Throughout the training, the evaluation of $f_\mathrm{mem}$ stays below $10^{-3}$.}
        \label{fig:app:FID_tau}
    \end{minipage}
\end{figure}

Figure~\ref{fig:app:FID_tau} shows the evolution of the time-integrated test loss and of the FID as a function of the training time $\tau$ for one of the models used in the right panel of Fig.~\ref{fig:effect_regularization} at $n=2048$, $m=8$ and $W=16$ ($p\approx1.8$M). We select this model in particular because it corresponds to the first point at which the FID of the early-stopped models increases in Fig.~\ref{fig:effect_regularization}.
The two early-stopping criteria clearly disagree: while the test loss selects a finite
$\tau^\star=\mathrm{argmin}_\tau\,\Ltest(\tau)$ (used in the main text), the FID decreases monotonically during training and its own optimum lies at
long training times. This supports the conjecture of Sect.~\ref{Sect:Discussion} that early-stopping on the FID itself would remove the FID double descent as the peak value ($W=16$) would now sit below the FID measure of $W=8$. The two criteria (test loss and FID) do not optimize the same quantities as the FID only measures similarities in the features of the training and generated samples and is therefore insensitive to memorization: a model copying exactly its training set would get a good FID. The test loss, however, penalizes the deviation of the learned score from the population one, and therefore increases as the estimator starts to correlate with the training set. Stopping at $\tau^\star$ therefore minimizes this precise criterion and remains far cheaper to monitor.\footnote{Evaluating the FID requires generating many samples at different training times, whereas the test loss only computes a few forward passes on a test set.} In the precise configuration of Fig.~\ref{fig:app:FID_tau} (small $m$ and $W$), the memorization fraction (as measured by the ratio of first to second nearest neighbor in the $L_2$ sense being below $1/3$, see \citealp{yoon2023diffusion, bonnaire2025diffusionmodelsdontmemorize}) remains below $10^{-3}$: starting from random noise does not lead to memorized copies of training samples.

%% file: Appendix_analytical.tex
\section{Proof of the analytical results and additional results}
\label{app:analtical_proofs}
In this section of the Appendix we first present some preliminary notions on random matrix theory, then we present the derivation of the results of the main text, namely the closed-form losses and bias--variance decomposition of Theorem~\ref{thm:main}, together with the overparameterized ($\psi_p\to\infty$) and infinite noise ($m\to\infty$) limits, which we state and prove here (Proposition~\ref{prop:app-psip-inf} and Proposition~\ref{prop:app-minf}).

\subsection{Preliminaries}
In this subsection, we present some useful results and formulae of Random Matrix Theory and Gaussian calculus.
\label{sec:preliminaries}

\subsubsection{Resolvent and Stieltjes transform}
For a symmetric matrix $\vM\in\mathbb{R}^{N\times N}$ and $z$ outside its spectrum,
we define its \emph{resolvent} as the matrix $\mathbf{\mathcal{G}}_N=(\vM-z\vI)^{-1}$ and
its \emph{Stieltjes transform} as
\begin{equation}
    g_N(z)=\frac{1}{N}\Tr\bigl(\vM-z\vI\bigr)^{-1}
        =\int\frac{\dd\rho_N(\lambda)}{\lambda-z},
\end{equation}
where $\rho_N$ is the empirical spectral distribution of $\vM$. Although $g_N(z)$ is a
random variable, typically, in the
high-dimensional limit $N\gg 1$ it
converges almost surely to a deterministic limit $\bar g(z)$. 
\subsubsection{Gaussian Toolkit}
When computing the traces, we will use usual formulae of Gaussian calculus,
\begin{enumerate}
    \item For symmetric positive-definite $\vM\in\mathbb{R}^{N\times N}$ and $\vJ\in\mathbb{R}^N$,
\begin{equation}
    \int_{\mathbb{R}^N}\dd\phi\;
        e^{-\frac12\phi^T \vM\phi+\vJ^T\phi}
    =(2\pi)^{N/2}(\det \vM)^{-1/2}\,e^{\frac12 \vJ^T \vM^{-1}\vJ}.
\end{equation}
\item Let $\vX\sim\mathcal{N}(0,\vSigma)$ in $\mathbb{R}^N$ and let $\vK\in\mathbb{R}^{N\times N}$ be symmetric with
$\vI+\vSigma \vK\succ0$. Then
\begin{equation}
    \E\bigl[e^{-\frac12 \vX^T \vK \vX}\bigr]
    =\det\bigl(\vI+\vSigma \vK\bigr)^{-1/2}.
\end{equation}
\item (Wick's theorem) For jointly centered Gaussian variables $X_1,\dots,X_{2k}$,
\begin{equation}
    \E[X_1\cdots X_{2k}]
    =\sum_{\text{pairings }\pi}\;\prod_{(i,j)\in\pi}\E[X_iX_j],
\end{equation}
the sum running over all perfect matchings. In particular
\begin{align}
\E[X_1X_2X_3X_4]=\E[X_1X_2]\E[X_3X_4]+\E[X_1X_3]\E[X_2X_4]+\E[X_1X_4]\E[X_2X_3].
\end{align}
\end{enumerate}

\subsubsection{The Replica Method}
In our context, we frequently need to compute the expected logarithm of a partition function $\mathcal{Z}$ that arises from the Gaussian integral representation of the resolvent, e.g., $\mathcal{Z} \propto \int \dd\phi \, e^{-\frac{1}{2}\phi^T \mathbf{M} \phi}$, where $\mathbf{M}$ is a random matrix of interest. To compute such expectations of logarithms of random variables, we employ the replica method \citep{mezard1987spin}. The core idea is based on the algebraic identity:
\begin{equation}
    \log \mathcal{Z} = \lim_{s\to 0} \frac{\mathcal{Z}^s - 1}{s}.
\end{equation}
By evaluating the integer moments $\mathbb{E}[\mathcal{Z}^s]$ for $s \in \mathbb{N}$, one can analytically continue the result to $s \to 0$. This procedure is a cornerstone of the replica method from statistical physics, a theoretical physics technique that has had a wide range of applications, notably in statistical learning and the theory of neural networks \citep{Ascoli_2020}. The replica method is widely believed to yield exact asymptotic predictions, as has been established rigorously in a wide range of problems \citep{guerra2002thermodynamic,talagrand2006parisi,barbier2019optimal,gerbelot2023asymptotic,vilucchio2025asymptotics}. More precisely, our results rely on a so-called replica-symmetry (RS) assumption. The RS ansatz is known to hold for the trace of rational polynomials of random matrices, where it has been shown to yield the exact same results as rigorous methods such as linear pencils \citep{bodin2021model, george_2025, bonnaire2025diffusionmodelsdontmemorize}.

\subsubsection{Replica representation of an inverse matrix}
To obtain Gaussian integrals we use the ``replica'' representation of the element
$(\alpha\beta)$ of a $p\times p$ matrix $\vM$:
\begin{equation}
    \bigl(\vM^{-1}\bigr)_{\alpha\beta}
    =\lim_{s\to0}\int\!\left(\prod_{a=1}^{s}\prod_{\gamma=1}^{p}\dd\phi^{a}_{\gamma}\right)
        \phi^{1}_{\alpha}\,\phi^{1}_{\beta}\,
        \exp\!\left(-\tfrac12\,\phi^{a}_{\gamma}\vM_{\gamma\delta}\phi^{a}_{\delta}\right).
    \label{eq:replica-inverse}
\end{equation}
Indeed, using the Gaussian integral representation of the inverse of $\vM$,
\begin{align*}
    \bigl(\vM^{-1}\bigr)_{\alpha\beta}
    &=\mathcal{Z}^{-1}\int\!\left(\prod_{\gamma=1}^{p}\dd\phi_{\gamma}\right)
        \phi_{\alpha}\phi_{\beta}\,
        \exp\!\left(-\tfrac12\,\phi_{\gamma}\vM_{\gamma\delta}\phi_{\delta}\right),\\
    \mathcal{Z}&=\sqrt{\frac{(2\pi)^{p}}{\det \vM}}
     =\int\!\left(\prod_{\gamma=1}^{p}\dd\phi_{\gamma}\right)
        \exp\!\left(-\tfrac12\,\phi_{\gamma}\vM_{\gamma\delta}\phi_{\delta}\right).
\end{align*}
Using the replica identity, we rewrite the normalization as
$\mathcal{Z}^{-1}=\lim_{s\to0}\mathcal{Z}^{\,s-1}$, obtaining
\begin{equation*}
    \bigl(\vM^{-1}\bigr)_{\alpha\beta}
    =\lim_{s\to0}\,\mathcal{Z}^{\,s-1}\int\!\left(\prod_{\gamma=1}^{p}\dd\phi_{\gamma}\right)
        \phi_{\alpha}\phi_{\beta}\,
        \exp\!\left(-\tfrac12\,\phi_{\gamma}\vM_{\gamma\delta}\phi_{\delta}\right).
\end{equation*}
Renaming the integration variable of this integral as $\phi^{1}$ and writing the
factor $\mathcal{Z}^{\,s-1}$ as $s-1$ further Gaussian integrals over replicas $\phi^{a}$,
$a\in\{2,\dots,s\}$, we obtain expression Eq.~\ref{eq:replica-inverse}.

\subsubsection{Schur complement}

For a block matrix $\bigl(\begin{smallmatrix}\vA&\vB\\\vC&\vD\end{smallmatrix}\bigr)$:
if $\vA$ is invertible,
\begin{equation}
    \det\begin{pmatrix}\vA&\vB\\\vC&\vD\end{pmatrix}
    =\det(\vA)\,\det\!\bigl(\vD-\vC\vA^{-1}\vB\bigr),
\end{equation}
and if $\vD$ is invertible, the top-left block of the inverse is
\begin{equation}
    \left[\begin{pmatrix}\vA&\vB\\\vC&\vD\end{pmatrix}^{-1}\right]_{11}
    =\bigl(\vA-\vB\vD^{-1}\vC\bigr)^{-1}.
\end{equation}

\subsubsection{Saddle-point method}
In the computation of the traces we will obtain expressions of the form
\begin{equation}
    \int \dd\vTheta\;P(\vTheta)\,e^{-\frac{d}{2}\,s\,S(\vTheta)},
\end{equation}
where $\vTheta$ collects the order parameters, $S$ is the effective action,
$P$ a non-exponential prefactor, and $s$ the number of replicas. As $d\to\infty$, the Laplace method applies and the
measure $\propto e^{-\frac{d}{2}sS(\vTheta)}$ concentrates on the stationary
point $\vTheta^*$ defined by $\partial_{\vTheta} S(\vTheta^*)=0$, giving
\begin{equation}
    \int \dd\vTheta\;P(\vTheta)\,e^{-\frac{d}{2}sS(\vTheta)}
    \;\sim\; P(\vTheta^*).
\end{equation}

\subsection{Bias--variance decomposition of the test loss in the general case}
\label{app:bias_variance_general_case}

In this section we derive the bias--variance decomposition of the denoising
score matching (DSM) test loss, for a general data distribution.

We recall that $\vs_{\mathcal{D},\vTheta}$ denotes the minimizer of the train
loss $\Ltrain$ obtained from a training set $\mathcal{D}$ and an initialization
$\vTheta$ of the parameters. Throughout, $\mathbb{E}_{\vy}$ denotes the average
over a fresh test sample $\vy\sim P_t$, drawn independently of $\mathcal{D}$,
while $\langle\cdot\rangle$ denotes the average over the realizations of the
training set and over the remaining sources of randomness, such as $\vTheta$.
Note that $\Ltest(\vs_{\mathcal{D},\vTheta})$ is itself a random variable: the
decomposition below concerns its average $\langle\Ltest\rangle$, which is the
quantity we look at. During the computation, we shall also use Tweedie's identity
\begin{equation}
\label{eq:tweedie}
    \mathbb{E}\!\left[\vxi\,\middle|\,\vy\right]
    =-\sqrt{\Delta_t}\,\nabla\log P_t(\vy).
\end{equation}
Adding and subtracting $\sqrt{\Delta_t}\,\nabla\log P_t(\vy)$ inside the norm,
at fixed $\mathcal{D}$ and $\vTheta$, yields
\begin{align}
    \Ltest(\vs_{\mathcal{D},\vTheta})
    &=\frac{1}{d}\,\mathbb{E}_{\vx,\vxi}\!\left[\left\lVert
    \sqrt{\Delta_t}\,\vs_{\mathcal{D},\vTheta}(\vy)+\vxi\right\rVert^{2}\right]
    \\[2pt]
    &=\frac{1}{d}\,\mathbb{E}_{\vx,\vxi}\!\left[\left\lVert
    \sqrt{\Delta_t}\Big(\vs_{\mathcal{D},\vTheta}(\vy)-\nabla\log P_t(\vy)\Big)
    +\Big(\vxi+\sqrt{\Delta_t}\,\nabla\log P_t(\vy)\Big)\right\rVert^{2}\right]
    \\[2pt]
    &=\frac{\Delta_t}{d}\,\mathbb{E}_{\vy}\!\left[\left\lVert
    \vs_{\mathcal{D},\vTheta}(\vy)-\nabla\log P_t(\vy)\right\rVert^{2}\right]
    +C_t .
\end{align}
The cross term vanishes: since the test sample is drawn independently of
$(\mathcal{D},\vTheta)$, the factor
$\vs_{\mathcal{D},\vTheta}(\vy)-\nabla\log P_t(\vy)$ is a deterministic
function of $\vy$ and can be pulled out of the conditional expectation, while
Eq.~\ref{eq:tweedie} gives
$\mathbb{E}[\vxi+\sqrt{\Delta_t}\,\nabla\log P_t(\vy)\,|\,\vy]=\mathbf{0}$.
The remaining term,
\begin{equation}
\label{eq:C_t}
    C_t=\frac{1}{d}\,\mathbb{E}_{\vx,\vxi}\!\left[\left\lVert
    \vxi+\sqrt{\Delta_t}\,\nabla\log P_t(\vy)\right\rVert^{2}\right]
    =1-\frac{\Delta_t}{d}\,\mathbb{E}_{\vy}\!\left[\left\lVert
    \nabla\log P_t(\vy)\right\rVert^{2}\right],
\end{equation}
depends only on the data distribution and on $t$, and is independent of the
learned score. It is the irreducible part of the test loss, namely the residual
uncertainty on $\vxi$ left after observing $\vy$, and is reached by the exact
score $\nabla\log P_t$.

Averaging over $\mathcal{D}$ and $\vTheta$ and inserting
$\pm\langle\vs_{\mathcal{D},\vTheta}(\vy)\rangle$ inside the norm yields
\begin{equation}
\label{eq:decomposition}
    \big\langle\Ltest\big\rangle
    =C_t+\Delta_t\left(\mathcal{B}^{2}+\mathcal{V}\right),
\end{equation}
with
\begin{align}
\label{eq:bias_variance_app}
    \mathcal{B}^{2}=\frac{1}{d}\,\mathbb{E}_{\vy}\!\left[\left\lVert
    \nabla\log P_t(\vy)-\big\langle\vs_{\mathcal{D},\vTheta}(\vy)
    \big\rangle\right\rVert^{2}\right],
    \qquad
    \mathcal{V}=\frac{1}{d}\,\mathbb{E}_{\vy}\!\left[\Big\langle\left\lVert
    \vs_{\mathcal{D},\vTheta}(\vy)-\big\langle\vs_{\mathcal{D},\vTheta}(\vy)
    \big\rangle\right\rVert^{2}\Big\rangle\right].
\end{align}
Here the cross term also vanishes. The factor
$\langle\vs_{\mathcal{D},\vTheta}(\vy)\rangle-\nabla\log P_t(\vy)$ is
deterministic once $\vy$ is fixed, while
$\vs_{\mathcal{D},\vTheta}(\vy)-\langle\vs_{\mathcal{D},\vTheta}(\vy)\rangle$
has zero mean under $\langle\cdot\rangle$ by construction.

\subsection{Loss decomposition}
\label{app:loss_decomposition}
In this section we use the Gaussian Equivalence Principle to decompose the three losses into traces of rational functions of random matrices.

The three losses are all of the form
\begin{equation}
    \mathcal{L}(\vs)=\frac{1}{d}\,\mathbb{E}_{(\vx,\vxi)\sim p}\!\left[\Big\lVert \sqrt{\Delta_t}\vs(e^{-t}\vx+\sqrt{\Delta_t}\vxi)+\vxi\Big\rVert^2\right],
\end{equation}
where the joint law $p(\vx,\vxi)$ of the noisy data and the noise is one of
\begin{equation}
    p_{\mathrm{data}}\otimes\mathcal{N}(0,1),
    \qquad
    \tfrac{1}{n}\sum_{\nu=1}^n\delta(\vx-\vx^\nu)\otimes\mathcal{N}(0,1),
    \qquad
    \tfrac{1}{nm}\sum_{\nu=1}^n\sum_{\mu=1}^m\delta(\vx-\vx^\nu)\otimes\delta(\vxi-\vxi^{\nu\mu}),
\end{equation}
yielding respectively $\Ltest$, $\Linfty$ and $\Ltrain$. We derive the following result on $\Linfty$ which relates to the distance to the empirical score,
\begin{proposition}
\label{prop:empirical-score}
Let $\hat{P}_t(\vy)=\frac{1}{n}\sum_{\nu=1}^n\mathcal{N}(\vy;\,e^{-t}\vx^\nu,\Delta_t\vI_d)$ be the empirical noised distribution and $\vs_{\mathrm{emp}}(\vy)=\nabla_\vy\log\hat{P}_t(\vy)$ the associated empirical score. Then, for any score function $\vs$,
\begin{equation}
    \Linfty(\vs)
    =\frac{\Delta_t}{d}\,\mathbb{E}_{\vy\sim\hat{P}_t}\!\left[\bigl\lVert\vs(\vy)-\vs_{\mathrm{emp}}(\vy)\bigr\rVert^2\right]+C,
\end{equation}
where $C$ depends only on $t$ and the empirical score $\vs_{\mathrm{emp}}$, but not on $\vs$. Moreover, assume that $n$ grows at most polynomially in $d$, that $t=O_d(1)$, and that $\lVert\vx^\mu-\vx^\nu\rVert^2 = O(d)$ for all $\mu \neq \nu$. Then $C\to 0$ as $d\to\infty$.
\end{proposition}
\begin{proof}
Write $\vy^\nu=e^{-t}\vx^\nu+\sqrt{\Delta_t}\,\vxi$ with $\vxi\sim\mathcal{N}(0,\vI_d)$. Since each $\vy^\nu$ is drawn from $\hat{P}_t$, we can express $\Linfty$ as an expectation over $\hat{P}_t$:
\begin{equation}
    \Linfty(\vs)
    =\frac{1}{d}\,\mathbb{E}_{\vy\sim\hat{P}_t}\!\left[\Delta_t\lVert\vs(\vy)\rVert^2
    +2\sqrt{\Delta_t}\,\vs(\vy)\cdot\mathbb{E}[\vxi\mid\vy]\right]+C_1,
\end{equation}
where $C_1=\frac{1}{d}\mathbb{E}\lVert\vxi\rVert^2$ does not depend on $\vs$. To evaluate $\mathbb{E}[\vxi\mid\vy]$, note that for each component $\nu$ the conditional score of the Gaussian kernel satisfies $\nabla_\vy\log p(\vy\mid\vx^\nu)=-({\vy-e^{-t}\vx^\nu})/{\Delta_t}=-\vxi^\nu/\sqrt{\Delta_t}$, so $\vxi^\nu=-\sqrt{\Delta_t}\nabla_\vy\log p(\vy\mid\vx^\nu)$. Taking the posterior expectation over $\vx^\nu$ given $\vy$ under $\hat{P}_t$,
\begin{equation}
    \mathbb{E}[\vxi\mid\vy]
    =-\sqrt{\Delta_t}\,\mathbb{E}_{\vx^\nu\mid\vy}\!\left[\nabla_\vy\log p(\vy\mid\vx^\nu)\right]
    =-\sqrt{\Delta_t}\,\nabla_\vy\log\hat{P}_t(\vy)
    =-\sqrt{\Delta_t}\,\vs_{\mathrm{emp}}(\vy),
\end{equation}
where the second equality uses $\nabla_\vy\log\hat{P}_t(\vy)=\mathbb{E}_{\vx^\nu\mid\vy}[\nabla_\vy\log p(\vy\mid\vx^\nu)]$. Substituting back and completing the square gives the result, with $C=C_1-\frac{\Delta_t}{d}\mathbb{E}_{\vy}\lVert\vs_{\mathrm{emp}}(\vy)\rVert^2$. The vanishing of $C$ as $d\to\infty$ with $t=O(1)$ follows from $\vs_{\mathrm{emp}}(e^{-t}\vx^\nu+\sqrt{\Delta_t}\,\vxi)\to-\vxi/\sqrt{\Delta_t}$, which gives $\Linfty(\vs_{\mathrm{emp}})\to 0$, i.e.\ $C\to 0$.
\end{proof}
Hence, in our regime, the smaller $\Linfty$ is, the closer the learned score is to the empirical score and hence to memorization.

For the random-feature score $\vs_{\vA}(\vx)=\frac{\vA}{\sqrt{p}}\sigma(\frac{\vW\vx}{\sqrt{d}})$, writing $\vs(\vx_t)$ for the model at the noised input $\vx_t=e^{-t}\vx+\sqrt{\Delta_t}\vxi$ and expanding the square,
\begin{multline}
    \mathcal{L}(\vs)
    =\frac{1}{d}\mathbb{E}[\lVert\vxi\rVert^2]
     +\frac{2\sqrt{\Delta_t}}{d}\Tr\frac{\vA}{\sqrt{p}}\mathbb{E}\!\left[\sigma(\tfrac{\vW\vx_t}{\sqrt{d}})\,\vxi^T\right]
     \\+\frac{\Delta_t}{d}\Tr\!\left(\frac{\vA}{\sqrt{p}}\mathbb{E}\!\left[\sigma(\tfrac{\vW\vx_t}{\sqrt{d}})\sigma(\tfrac{\vW\vx_t}{\sqrt{d}})^T\right]\frac{\vA^T}{\sqrt{p}}\right)
     +\frac{\Delta_t\lambda}{pd}\lVert\vA\rVert_F^2,
\end{multline}
where we added the ridge penalty $\frac{\Delta_t\lambda}{pd}\lVert\vA\rVert_F^2$. In all three cases $\frac{1}{d}\mathbb{E}[\lVert\vxi\rVert^2]$ concentrates to 1. Introducing the feature--noise and feature--feature correlations
\begin{align}
    &\tilde{\vV}=\tfrac{1}{\sqrt{\Delta_t}}\mathbb{E}_{\vx,\vxi}[\sigma(\tfrac{\vW\vx_t}{\sqrt{d}})\vxi^T],
    \qquad
    \tilde{\vU}=\mathbb{E}_{\vx_t\sim p_t}[\sigma(\tfrac{\vW\vx_t}{\sqrt{d}})\sigma(\tfrac{\vW\vx_t}{\sqrt{d}})^T],\\
    &\vV_n^m=\tfrac{1}{\sqrt{\Delta_t}nm}\sum_{\nu,\mu}\sigma(\tfrac{\vW(e^{-t}\vx^\nu+\sqrt{\Delta_t}\vxi^{\nu\mu})}{\sqrt{d}})(\vxi^{\nu\mu})^T,
    \quad
    \vU_n^m=\tfrac{1}{nm}\sum_{\nu,\mu}\sigma(\tfrac{\vW\vx_t^{\nu\mu}}{\sqrt{d}})\sigma(\tfrac{\vW\vx_t^{\nu\mu}}{\sqrt{d}})^T,\\
    &\vV_n^\infty=\tfrac{1}{\sqrt{\Delta_t}n}\sum_{\nu}\mathbb{E}_{\vxi}[\sigma(\tfrac{\vW(e^{-t}\vx^\nu+\sqrt{\Delta_t}\vxi)}{\sqrt{d}})\vxi^T],
    \quad
    \vU_n^\infty=\tfrac{1}{n}\sum_{\nu}\mathbb{E}_{\vxi}[\sigma(\tfrac{\vW\vx_t^{\nu}}{\sqrt{d}})\sigma(\tfrac{\vW\vx_t^{\nu}}{\sqrt{d}})^T],
\end{align}
the regularized empirical risk minimizer is the ridge estimator
\begin{equation}
    \frac{\vA_{\mathrm{ERM}}}{\sqrt{p}}=-(\vV_n^m)^T\,(\vU_n^m+\lambda\vI_p)^{-1}.
\end{equation}
Substituting it back, the three losses read
\begin{align}
   \Ltrain
    &=1
    -\frac{2\Delta_t}{d}\Tr\!\big((\vV_n^m)^T (\vU_n^m+\lambda\vI_p)^{-1}\vV_n^m\big)\\
    &\qquad+\frac{\Delta_t}{d}\Tr\!\big((\vV_n^m)^T (\vU_n^m+\lambda\vI_p)^{-1}\vU_n^m(\vU_n^m+\lambda\vI_p)^{-1}\vV_n^m\big)\notag\\
    &\qquad+\frac{\Delta_t\lambda}{d}\Tr\!\big((\vV_n^m)^T (\vU_n^m+\lambda\vI_p)^{-2}\vV_n^m\big),\\
     \Linfty
    &=1-\frac{2\Delta_t}{d}\Tr\!\big((\vV_n^m)^T (\vU_n^m+\lambda\vI_p)^{-1}\vV_n^\infty\big)\\
    &\qquad+\frac{\Delta_t}{d}\Tr\!\big((\vV_n^m)^T (\vU_n^m+\lambda\vI_p)^{-1}\vU_n^\infty (\vU_n^m+\lambda\vI_p)^{-1}\vV_n^m\big),\\
    \Ltest
    &=1-\frac{2\Delta_t}{d}\Tr\!\big((\vV_n^m)^T (\vU_n^m+\lambda\vI_p)^{-1}\tilde{\vV}\big)\\
    &\qquad+\frac{\Delta_t}{d}\Tr\!\big((\vV_n^m)^T (\vU_n^m+\lambda\vI_p)^{-1}\tilde{\vU} (\vU_n^m+\lambda\vI_p)^{-1}\vV_n^m\big).
\end{align}

\subsection{Assumptions and notation}
\label{app:assumptions_notation}
In this section we recall the assumptions made in the analytical part as well as the notations used.

The assumptions are
\begin{enumerate}[label=\textbf{(A\arabic*)},ref=A\arabic*,leftmargin=3.4em]

\item\label{ass:data} \textbf{Data distribution.} The training samples $\vx^\nu$ are i.i.d. samples from $P_0$ and the frozen noises are i.i.d. from $\mathcal{N}(0,\vI_d)$. We present the result with $P_0=\mathcal{N}(0,\vI_d)$ but in the appendix we extend the results to $P_0$ with zero mean, sub Gaussian and with a covariance $\vSigma$ that admits a limiting spectral distribution $\rho_{\vSigma}$ and verifies $\Tr(\vSigma)/d\to \sigma_{\vx}^2=O_d(1)$. 

\item\label{ass:model} \textbf{RFNN Architecture.} The score is modeled by the RFNN
$\vs_{\vA}(\vy)=\frac{\vA}{\sqrt p}\sigma\!\big(\frac{\vW\vy}{\sqrt d}\big)$,
where the first-layer weights $\vW\in\mathbb{R}^{p\times d}$ have i.i.d.\
$\mathcal{N}(0,1)$ entries and are frozen and independent of the data and of the
noises, while only the read-out $\vA\in\mathbb{R}^{d\times p}$ is trained.

\item\label{ass:activation} \textbf{Activation Function.} The activation $\sigma$ admits the Hermite expansion
$\sigma(z)=\sum_{s\geq0}\frac{c_s}{s!}He_s(z)$ with
$\E_{z\sim\mathcal{N}(0,1)}[\sigma(z)^2]=\sum_{s\geq0}\frac{c_s^2}{s!}<\infty$,
and is centered,
$\mu_0=c_0=\E_{z\sim\mathcal{N}(0,1)}[\sigma(z)]=0$. We write
$\mu_1=c_1=\E[z\,\sigma(z)]$ and
$\mu_*^2=\E[\sigma(z)^2]-\mu_1^2=\sum_{s\geq2}\frac{c_s^2}{s!}$. \looseness=-1

\item\label{ass:limit} \textbf{Asymptotic limit.} We consider
the proportional regime $n,p,d\to\infty$ with
$n/d\to\psi_n\in(0,\infty)$ and $p/d\to\psi_p\in(0,\infty)$, the remaining
parameters $m\in\mathbb{N}^*$, $t>0$ being held fixed and $O(1)$. The two further limits we consider,
$\psi_p\to\infty$ (Appendix~\ref{app:psip_infty_limit}) and $m\to\infty$
(Appendix~\ref{app:m_infty_limit}), are always taken after $d\to\infty$.

\end{enumerate}

\paragraph*{Index conventions.}
Unless stated otherwise we use the Einstein summation convention, repeated
indices being implicitly summed over. Latin indices $i,j,k\in\{1,\dots,d\}$ label
input coordinates, Greek indices $\alpha,\beta\in\{1,\dots,p\}$ label features
(hidden units), $\nu,\nu'\in\{1,\dots,n\}$ label training samples,
$\mu,\mu'\in\{1,\dots,m\}$ label noise realizations, and $a,b\in\{1,\dots,s\}$
label replicas. Bold upright symbols denote vectors, matrices and
tensors, and $\vI_k$ is the $k\times k$ identity.

\begin{center}\small
\begin{tabular}{@{}lp{0.63\linewidth}@{}}
\hline
$d,\;n,\;p$ & input dimension, number of training samples, number of features \\
$m$ & number of frozen noise realizations per training sample \\
$\psi_n=n/d$, $\psi_p=p/d$ & sample and parameter ratios \\
$t$, $\Delta_t=1-e^{-2t}$ & diffusion time and variance of the forward kernel \\
$\lambda$, $\tilde\lambda=\lambda/\psi_p$ & ridge strength, in the $O(1)$ and $O(\psi_p)$ scalings of Appendix~\ref{app:psip_infty_limit} \\
$\vSigma$, $\rho_{\vSigma}$ & data covariance ($\vSigma=\vI_d$ in the main text) and its limiting spectral density \\
$\vx^\nu\in\mathbb{R}^d$ & training sample $\nu$; $\mathcal{D}=\{\vx^\nu\}_{\nu\leq n}$ is the training set \\
$\vxi\in\mathbb{R}^{d\times nm}$ & frozen training noises \\
$\vY\in\mathbb{R}^{d\times nm}$ & noised inputs, $\vY^{\nu\mu}_i=e^{-t}\vx^\nu_i+\sqrt{\Delta_t}\,\vxi^{\nu\mu}_i$ \\
$\vW\in\mathbb{R}^{p\times d}$, $\vA\in\mathbb{R}^{d\times p}$ & frozen first layer and trained read-out \\
$\vF\in\mathbb{R}^{p\times nm}$ & feature matrix $\sigma(\vW\vY/\sqrt d)$ \\
$\vv\in\mathbb{R}^m$ & all-ones vector, $\vv^\mu=1$\\
$\one_m=\vv\vv^T$ & all-ones matrix, $\one_m\in\mathbb{R}^{m\times m}$ \\

$\mu_1$ &  $\E[z\sigma(z)]$  \\
$\mu_*^2$ & $\E[\sigma(z)^2]-\mu_1^2$ \\
$\kappa$ & $\dfrac{1}{\mu_*^2}\,\E_{u,v,w}\big[\big(\sigma(e^{-t}u+\sqrt{\Delta_t}\,v)-\mu_1e^{-t}u\big)$ \newline
\hspace*{1.2em}$\times\big(\sigma(e^{-t}u+\sqrt{\Delta_t}\,w)-\mu_1e^{-t}u\big)\big]$, \newline
$u,v,w\sim\mathcal{N}(0,1)$ independent \\
$\Ltrain$, $\Linfty$, $\Ltest$ & train loss at finite $m$, empirical test loss ($m=\infty$ at fixed $\mathcal{D}$), and population test loss, Eqs.~\ref{eq:train_loss}, ~\ref{eq:m_infty_loss} and ~\ref{eq:test_loss} \\
$\mathcal{B}^2$, $\mathcal{V}$ & bias, variance\\
\hline
\end{tabular}
\end{center}

\subsection{Gaussian Equivalence Principle}
\label{app:gep}
According to the Gaussian Equivalence Principle (GEP)~\citep{Pennington_2017,Peche2019,goldt2020modeling,goldt_2021,mei2020,hu2023}, nonlinear random matrices of the form $\vF=\sigma(\frac{\vW\vY}{\sqrt{d}})$ have asymptotically the same limiting spectral distribution and resolvent traces as linear matrices with the same first two moments,
\begin{equation}
    \vF=\sigma\!\left(\tfrac{\vW\vY}{\sqrt{d}}\right)\ \longrightarrow\ \mu_0\,\one\one^T+\mu_1\tfrac{\vW\vY}{\sqrt{d}}+\mu_*\vOmega,
\end{equation}
with scalar constants $\mu_0,\mu_1,\mu_*$ and $\vOmega$ a Gaussian tensor independent of $\vW$ and $\vY$. Since the data are sampled from $\mathcal{N}(0,\vI_d)$ and we assumed that $\mathbb{E}_{z\sim\mathcal{N}(0,1)}[\sigma(z)]=0$, we have $\mu_0=0$. Matching the first two moments of the tensor $\vF_\alpha^{\nu\mu}$ requires $\mu_1=\mathbb{E}[\sigma(z)z]$, $\mu_*^2=\mathbb{E}[\sigma^2(z)]-\mu_1^2$, and
\begin{equation}
    \mathbb{E}[\vOmega^{\nu\mu}_\alpha\vOmega^{\nu'\mu'}_\beta]=\delta_{\alpha\beta}\delta^{\nu\nu'}\bigl(\delta^{\mu\mu'}+\kappa(1-\delta^{\mu\mu'})\bigr),
\end{equation}
with
\begin{equation}\label{eq:app-kappa}
    \kappa=\frac{1}{\mu_*^2}\,\mathbb{E}_{u,v,w}\!\left[\big(\sigma(e^{-t}u+\sqrt{\Delta_t}v)-\mu_1e^{-t}u\big)\big(\sigma(e^{-t}u+\sqrt{\Delta_t}w)-\mu_1e^{-t}u\big)\right],
\end{equation}
$u,v,w\sim\mathcal{N}(0,1)$ independent. We now replace $\vF$ by its Gaussian equivalent in the matrices $\vU$ and $\vV$.\footnote{With a slight abuse of notation we use the equal sign $=$ for two matrices that have asymptotically the same spectral densities.} Using $\mathbb{E}[\vY\vxi^T]=\sqrt{\Delta_t}\,\vI_d$ and $\mathbb{E}[\vOmega\vxi^T]=0$,
\begin{equation}
    \tilde{\vV}=\frac{1}{\sqrt{\Delta_t}}\mathbb{E}_{\vx,\vxi}\!\left[\Big(\mu_1\tfrac{\vW\vY}{\sqrt{d}}+\mu_*\vOmega\Big)\vxi^T\right]=\mu_1\frac{\vW}{\sqrt{d}}=\vV_n^\infty,
    \qquad
    \tilde{\vU}=\mu_1^2\frac{\vW\vW^T}{d}+\mu_*^2\vI_p.
\end{equation}
For $\vV_n^m$ and $\vU_n^m$ there is no expectation to compute and one simply substitutes the surrogate,
\begin{align}
    \vV_n^m&=\frac{1}{\sqrt{\Delta_t}nm}\sum_{\nu,\mu}\Big(\mu_1\tfrac{\vW(e^{-t}\vx^\nu+\sqrt{\Delta_t}\vxi^{\nu\mu})}{\sqrt{d}}+\mu_*\vOmega^{\nu\mu}\Big)(\vxi^{\nu\mu})^T,\\
    \vU_n^m&=\frac{1}{nm}\sum_{\nu,\mu}\Big(\mu_1\tfrac{\vW(e^{-t}\vx^\nu+\sqrt{\Delta_t}\vxi^{\nu\mu})}{\sqrt{d}}+\mu_*\vOmega^{\nu\mu}\Big)\Big(\mu_1\tfrac{\vW(e^{-t}\vx^\nu+\sqrt{\Delta_t}\vxi^{\nu\mu})}{\sqrt{d}}+\mu_*\vOmega^{\nu\mu}\Big)^T.
\end{align}

\paragraph*{Gaussian equivalent of $\vU_n^\infty$.}
The matrix $\vU_n^\infty=\frac1n\sum_\nu\mathbb{E}_{\vxi}[\vF\vF^T]$ is more involved because the expectation is taken on the noise only. Its GEP has been derived in \citet{george_2025,bonnaire2025diffusionmodelsdontmemorize}.\footnote{See Lemma~B.1 of \citet{george_2025} and Lemma~C.1 of \citet{bonnaire2025diffusionmodelsdontmemorize}.} We introduce the scalar constants
\begin{align}
    v_t^2&=\mathbb{E}_{u,v,w}[\sigma(e^{-t}u+\sqrt{\Delta_t}v)\sigma(e^{-t}u+\sqrt{\Delta_t}w)]-\bigl(\mathbb{E}_{u,v}[\sigma(e^{-t}u+\sqrt{\Delta_t}v)\,u]\bigr)^2,\\
    s_t^2&=\mathbb{E}_u[\sigma(u)^2]-\mathbb{E}_{u,v,w}[\sigma(e^{-t}u+\sqrt{\Delta_t}v)\sigma(e^{-t}u+\sqrt{\Delta_t}w)]-\bigl(\mathbb{E}_{u,v}[v\,\sigma(e^{-t}u+\sqrt{\Delta_t}v)]\bigr)^2,
\end{align}
which, as we show below, coincide with $v_t^2=\mu_*^2\kappa$ and $s_t^2=\mu_*^2(1-\kappa)$. Recall the Hermite expansion $\sigma(z)=\sum_{s\ge0}\frac{c_s}{s!}He_s(z)$, with $\mathbb{E}[He_sHe_{s'}]=s!\,\delta_{ss'}$, $c_0=0$, $c_1=\mu_1$ and $\mu_*^2=\sum_{s\ge2}\frac{c_s^2}{s!}$. Write $\vh_\alpha^\nu(\vxi)=\frac{\vW_\alpha\cdot(e^{-t}\vx^\nu+\sqrt{\Delta_t}\vxi)}{\sqrt d}$ for the preactivation, and
\begin{equation}
    R(h)=\sigma(h)-\mu_1 h=\sum_{s\ge2}\frac{c_s}{s!}He_s(h),
    \qquad
    \sigma_0(g)=\mathbb{E}_{w\sim\mathcal{N}(0,1)}[\sigma(g+\sqrt{\Delta_t}w)].
\end{equation}

\emph{Diagonal terms.} For $\alpha=\beta$, $\vh_\alpha^\nu(\vxi) \sim \mathcal{N}(0,1)$ asymptotically, and the data average concentrates, so up to $O(1/n)$\footnote{Terms of order $O(1/d)$ are irrelevant to the asymptotic spectrum.} we have,
\begin{equation}
    (\vU_n^\infty)_{\alpha\alpha}=\mathbb{E}_{\vxi}[\sigma(\vh_\alpha^\nu)^2]=\mathbb{E}_{z\sim\mathcal{N}(0,1)}[\sigma(z)^2]=\mu_1^2+\mu_*^2=\lVert\sigma\rVert^2.
\end{equation}

\emph{Off-diagonal terms.} Fix $\alpha\neq\beta$, freeze $\vW,\vx^\nu$, and write $\vh_\alpha^\nu=\vg_\alpha^\nu+\sqrt{\Delta_t}u_\alpha$ with $\vg_\alpha^\nu=\frac{e^{-t}\vW_\alpha\cdot\vx^\nu}{\sqrt d}\sim\mathcal{N}(0,e^{-2t})$ and $u_\alpha=\frac{\vW_\alpha\cdot\vxi}{\sqrt d}$, so $\mathbb{E}_{\vxi}[u_\alpha u_\beta]=\rho_{\alpha\beta}=\frac{\vW_\alpha\cdot\vW_\beta}{d}=O(d^{-1/2})$. The Mehler--Kibble formula \citep{kibble1945} gives
\begin{equation}
    \mathbb{E}_{\vxi}[\sigma(\vh_\alpha^\nu)\sigma(\vh_\beta^\nu)]
    =\sum_{s\ge0}\frac{\rho_{\alpha\beta}^s}{s!}\,\mathbb{E}_u[He_s(u)\sigma(\vg_\alpha^\nu+\sqrt{\Delta_t}u)]\,\mathbb{E}_u[He_s(u)\sigma(\vg_\beta^\nu+\sqrt{\Delta_t}u)].
\end{equation}
Terms $s\ge2$ are $O(\rho^2)=O(1/d)$ and are dropped. The $s=0$ term is $\sigma_0(\vg_\alpha^\nu)\sigma_0(\vg_\beta^\nu)$, while averaging the $s=1$ coefficient over the data gives $\rho_{\alpha\beta}(\mathbb{E}_{g,u}[u\,\sigma(g+\sqrt{\Delta_t}u)])^2=\rho_{\alpha\beta}\,\Delta_t\mu_1^2$, where Stein's lemma yields $\mathbb{E}[u\,\sigma(h)]=\sqrt{\Delta_t}\mu_1$ with $h=g+\sqrt{\Delta_t}u$. Hence
\begin{equation}
    (\vU_n^\infty)_{\alpha\beta}=\frac1n\sum_\nu\sigma_0(\vg_\alpha^\nu)\sigma_0(\vg_\beta^\nu)+\Delta_t\mu_1^2\frac{\vW_\alpha\cdot\vW_\beta}{d}+O(1/d).
\end{equation}
The field $\vg_\alpha^\nu$ is Gaussian and $\sigma_0$ smooth, so a second Gaussian equivalence applies, $\sigma_0(\vg_\alpha^\nu)\to\mu_1 \vg_\alpha^\nu+v_t\veta_\alpha^\nu$ with $\veta_\alpha^\nu\sim\mathcal{N}(0,1)$. The linear coefficient is $\mu_1$ (Stein's lemma) and the residual variance is
\begin{equation}
    v_t^2=\mathbb{E}_g[\sigma_0(g)^2]-\mu_1^2e^{-2t},
    \qquad
    \mathbb{E}_g[\sigma_0(g)^2]=\sum_{s\ge0}\frac{c_s^2}{s!}e^{-2st},
\end{equation}
using Mehler with shared $g\sim\mathcal{N}(0,e^{-2t})$. Since $\sum_{s\ge2}\frac{c_s^2}{s!}e^{-2st}=\mu_*^2\kappa$, we obtain $v_t^2=\mu_*^2\kappa$, i.e.\ $v_t=\mu_*\sqrt{\kappa}$. On the diagonal the same surrogate gives $\mu_1^2+\mu_*^2\kappa$; the deficit $\mu_*^2(1-\kappa)$ relative to $\lVert\sigma\rVert^2=\mu_1^2+\mu_*^2$ is restored by an isotropic term, giving $s_t^2=\mu_*^2(1-\kappa)$ and
\begin{equation}
    \vU_n^\infty=\frac{\vG}{\sqrt n}\frac{\vG^T}{\sqrt n}+\Delta_t\mu_1^2\frac{\vW\vW^T}{d}+\mu_*^2(1-\kappa)\vI_p,
    \qquad
    \vG_\alpha^\nu=e^{-t}\mu_1\tfrac{\vW_\alpha\cdot\vx^\nu}{\sqrt d}+v_t\veta_\alpha^\nu.
\end{equation}
The surrogate $\veta$ is the noise-average of the residual $\vOmega$ and is correlated with it. Writing $R^{\nu\mu}=\mu_*\vOmega^{\nu\mu}$ and $R_0^\nu=v_t\veta^\nu=\mathbb{E}_w[R(\vg^\nu+\sqrt{\Delta_t}w)]=\lim_{m\to\infty}\frac1m\sum_{\mu'}R^{\nu\mu'}$, conditioning on $\vg^\nu$ gives $\mathbb{E}[R^{\nu\mu}R_0^\nu]=\mathbb{E}_g[(R_0^\nu)^2]=v_t^2=\mu_*^2\kappa$, hence
\begin{equation}
    \mathbb{E}[\vOmega_\alpha^{\nu\mu}\veta_\beta^{\nu'}]=\frac{\mathbb{E}[R^{\nu\mu}R_0^\nu]}{\mu_*\,v_t}\,\delta_{\alpha\beta}\delta^{\nu\nu'}=\sqrt{\kappa}\,\delta_{\alpha\beta}\delta^{\nu\nu'}.
\end{equation}

\paragraph*{Summary.} The GEP yields $\tilde{\vV}=\vV_n^\infty=\mu_1\frac{\vW}{\sqrt d}$, $\tilde{\vU}=\mu_1^2\frac{\vW\vW^T}{d}+\mu_*^2\vI_p$, $\vV_n^m,\vU_n^m$ as above, and
\begin{equation}
    \vU_n^\infty=\frac1n\sum_\nu\big(e^{-t}\mu_1\tfrac{\vW\vx^\nu}{\sqrt d}+v_t\veta^\nu\big)\big(e^{-t}\mu_1\tfrac{\vW\vx^\nu}{\sqrt d}+v_t\veta^\nu\big)^T+\Delta_t\mu_1^2\frac{\vW\vW^T}{d}+s_t^2\vI_p,
\end{equation}
with $v_t^2=\mu_*^2\kappa$, $s_t^2=\mu_*^2(1-\kappa)$. Here $\veta$ is a Gaussian noise with $\veta^\nu_\alpha\sim\mathcal{N}(0,1)$, independent of $\vW$ and $\vY$ but correlated with the residual $\vOmega$ through $\E[\vOmega_\alpha^{\nu\mu}\veta_\beta^{\nu'}]=\sqrt{\kappa}\,\delta_{\alpha\beta}\delta^{\nu\nu'}$.

\subsection{Gaussian-equivalent losses and trace decomposition}
From now on we denote by $\vF^{\nu\mu}_\alpha$ its Gaussian equivalent $\mu_1\frac{\vW\vY}{\sqrt d}+\mu_*\vOmega$, and by $\vH=\E_{\vxi}[\vF]\in\mathbb{R}^{p\times n}$ the noise-averaged feature matrix, with entries $\vH^\nu_\alpha=e^{-t}\mu_1\frac{\vW_\alpha\cdot\vx^\nu}{\sqrt d}+v_t\veta^\nu_\alpha$ and $v_t=\mu_*\sqrt\kappa$. Substituting the surrogates into the losses,
\begin{align}
    \Ltrain &=1 - \frac{1}{d(nm)^2}\Tr\!\left(\vxi \vF^T\mathcal{G}_{-\lambda,0,0} \vF\vxi^T\right),\\
    \Ltest &=1 - \frac{2\mu_1\sqrt{\Delta_t}}{dnm}\Tr\!\left(\vxi \vF^T\mathcal{G}_{-\lambda,0,0}\tfrac{\vW}{\sqrt d}\right)\notag\\
    &\quad + \frac{1}{d(nm)^2}\Tr\!\left(\vxi \vF^T\mathcal{G}_{-\lambda,0,0}\,\tilde\vU\,\mathcal{G}_{-\lambda,0,0} \vF\vxi^T\right),\\
    \Linfty&= 1 - \frac{2\mu_1\sqrt{\Delta_t}}{dnm}\Tr\!\left(\vxi \vF^T\mathcal{G}_{-\lambda,0,0}\tfrac{\vW}{\sqrt d}\right)\notag\\
    &\quad + \frac{1}{d(nm)^2}\Tr\!\left(\vxi \vF^T\mathcal{G}_{-\lambda,0,0}\,\vU_n^\infty\,\mathcal{G}_{-\lambda,0,0} \vF\vxi^T\right),
\end{align}
with $\tilde\vU=\mu_1^2\frac{\vW\vW^T}{d}+\mu_*^2\vI_p$ and $\vU_n^\infty=\frac{\vH\vH^T}{n}+\Delta_t\mu_1^2\frac{\vW\vW^T}{d}+s_t^2\vI_p$ the Gaussian-equivalent second-moment matrices of the summary above, $s_t^2=\mu_*^2(1-\kappa)$, and $\mathcal{G}_{z,\zeta,\epsilon}$ the resolvent Eq.~\ref{eq:resolvent} of the main text,
\begin{equation}\label{eq:app-resolvent}
    \mathcal{G}_{z,\zeta,\epsilon}=\Big(\tfrac{\vF\vF^T}{nm}-z\vI_p-\zeta\tfrac{\vW\vW^T}{d}+\epsilon\tfrac{\vH\vH^T}{n}\Big)^{-1},
    \qquad
    \mathcal{G}_{-\lambda,0,0}=\Big(\tfrac{\vF\vF^T}{nm}+\lambda\vI_p\Big)^{-1}.
\end{equation}
Every quadratic form appearing above is a derivative of the single generating trace $T_1$ of Eq.~\ref{eq:traces},
\begin{equation}\label{eq:app-T1T2}
    T_1(z,\zeta,\epsilon)=\frac{1}{d(nm)^2}\Tr\!\left(\vxi\vF^{T}\mathcal{G}_{z,\zeta,\epsilon}\vF\vxi^{T}\right),
    \qquad
    T_2=\frac{1}{dnm}\Tr\!\left(\vxi\vF^T\mathcal{G}_{-\lambda,0,0}\tfrac{\vW}{\sqrt d}\right).
\end{equation}
Indeed, differentiating Eq.\ref{eq:app-resolvent} gives $\partial_z\mathcal{G}=\mathcal{G}^2$, $\partial_\zeta\mathcal{G}=\mathcal{G}\frac{\vW\vW^T}{d}\mathcal{G}$ and $\partial_\epsilon\mathcal{G}=-\mathcal{G}\frac{\vH\vH^T}{n}\mathcal{G}$, so that, writing $\mathcal{G}=\mathcal{G}_{-\lambda,0,0}$ and evaluating all derivatives at $(z,\zeta,\epsilon)=(-\lambda,0,0)$,
\begin{align}\label{eq:app-T345}
    T_3&=-\partial_\epsilon T_1=\tfrac{1}{d(nm)^2}\Tr\!\left(\vxi\vF^T\mathcal{G}\tfrac{\vH\vH^T}{n}\mathcal{G}\vF\vxi^T\right),\notag\\
    T_4&=\partial_\zeta T_1=\tfrac{1}{d(nm)^2}\Tr\!\left(\vxi\vF^T\mathcal{G}\tfrac{\vW\vW^T}{d}\mathcal{G}\vF\vxi^T\right),\qquad
    T_5=\partial_z T_1=\tfrac{1}{d(nm)^2}\Tr\!\left(\vxi\vF^T\mathcal{G}^2\vF\vxi^T\right).
\end{align}
Collecting terms, the three losses read as announced in Eqs.~\ref{eq:loss_trainm}--\ref{eq:loss_traininf} of the main text,
\begin{align}
    \Ltrain&=1-T_1(-\lambda,0,0),\\
    \Ltest&=1-2\mu_1\sqrt{\Delta_t}\,T_2+\mu_1^2\,\partial_\zeta T_1\big|_{(-\lambda,0,0)}+\mu_*^2\,\partial_z T_1\big|_{(-\lambda,0,0)},\\
    \Linfty&=1-2\mu_1\sqrt{\Delta_t}\,T_2-\partial_\epsilon T_1\big|_{(-\lambda,0,0)}+\Delta_t\mu_1^2\,\partial_\zeta T_1\big|_{(-\lambda,0,0)}+\mu_*^2(1-\kappa)\,\partial_z T_1\big|_{(-\lambda,0,0)},
\end{align}
for which we use throughout the shorthands $T_3,T_4,T_5$ of Eq.\ref{eq:app-T345}. It remains to compute $T_1$ and $T_2$. Both are controlled by the order parameters of Eq.~\ref{eq:qr-def},
\begin{equation}
    q(z,\zeta,\epsilon)=\frac1p\Tr\,\mathcal{G}_{z,\zeta,\epsilon},\qquad
    r(z,\zeta,\epsilon)=\frac1p\Tr\!\Big(\tfrac{\vW^T}{\sqrt d}\,\mathcal{G}_{z,\zeta,\epsilon}\,\tfrac{\vW}{\sqrt d}\Big),
\end{equation}
which, as shown next, obey the self-consistent equations of Theorem~\ref{thm:main}. Proposition~\ref{prop:app-saddle} derives them at $\epsilon=0$, which suffices for $T_1,T_2,T_4,T_5$; the $\epsilon$-deformed saddle, needed only for $T_3$, is obtained in Appendix~\ref{app:traces}.

\subsection{Self-consistent equations for the Stieltjes transform}
\label{app:self-consistent}
\begin{proposition}[Self-consistent equations]\label{prop:app-saddle}
Let $\rho_{\vSigma}$ denote the limiting spectral distribution of the data covariance $\vSigma$. In the proportional limit the order parameters $(q,r,h)$ and the conjugates $(\hat r,\hat h)$ satisfy the fixed-point system
\begin{align}
  z\psip &= \frac{1-\psip}{q}-\frac{r}{q^2}
     +\frac{\psip\mu_*^2\tfrac{1+(m-1)\kappa}{m}}{L_1}
     +\frac{(m-1)\,\psip\mu_*^2\tfrac{1-\kappa}{m}}{L_2},\label{eq:gen-q}\\
  \hat r &= \frac1q-\psip\zeta+\frac{\psip\mu_1^2\tfrac{\Delta_t}{m}}{L_1}+\frac{(m-1)\,\psip\mu_1^2\tfrac{\Delta_t}{m}}{L_2},
  \qquad
  \hat h = \frac{\psip\mu_1^2 e^{-2t}}{L_1},\label{eq:gen-rh-hat}\\
  r &= \int\frac{\dd\rho_{\vSigma}(\lambda)}{\hat r+\lambda\hat h},
  \qquad
  h = \int\frac{\lambda\,\dd\rho_{\vSigma}(\lambda)}{\hat r+\lambda\hat h},\label{eq:gen-rh}
\end{align}
with $L_1=L_1(q,r,h)=1+\frac{\psip}{\psin}\mu_1^2e^{-2t}h+\frac{\psip}{\psin m}\big[\mu_1^2\Delta_t r+\mu_*^2 q(1+(m-1)\kappa)\big]$ and $L_2=L_2(q,r)=1+\frac{\psip}{\psin m}\big[\mu_1^2\Delta_t r+\mu_*^2 q(1-\kappa)\big]$.
\end{proposition}

\paragraph*{Remark.}
For isotropic data $\vSigma=\vI_d$ one has $\rho_{\vSigma}=\delta(\lambda-1)$, so Eq.\ref{eq:gen-rh} gives $r=h=1/(\hat r+\hat h)$, i.e.\ $\hat r+\hat h=1/r$ and $L_1(q,r,r)=L_1(q,r)$. Adding the two equations of Eq.\ref{eq:gen-rh-hat} eliminates the conjugates and the system collapses to
\begin{align}
  \frac1r+\psip\zeta &= \frac1q
     +\frac{\mu_1^2\psip\!\left(e^{-2t}+\tfrac{\Delta_t}{m}\right)}{L_1(q,r)}
     +\frac{(m-1)\,\psip\mu_1^2\tfrac{\Delta_t}{m}}{L_2(q,r)},\\
  z\psip &= \frac{1-\psip}{q}-\frac{r}{q^2}
     +\frac{\psip\mu_*^2\tfrac{1+(m-1)\kappa}{m}}{L_1(q,r)}
     +\frac{(m-1)\,\psip\mu_*^2\tfrac{1-\kappa}{m}}{L_2(q,r)},
\end{align}
with $L_1(q,r)=1+\frac{\psip}{\psin m}\big[\mu_1^2\,r(e^{-2t}m+\Delta_t)+\mu_*^2\,q(1+(m-1)\kappa)\big]$ and $L_2(q,r)=1+\frac{\psip}{\psin m}\big[\mu_1^2\,r\,\Delta_t+\mu_*^2\,q(1-\kappa)\big]$. This is the self-consistent system verified by $(q,r)$ of Theorem~\ref{thm:main} of the main text.

We prove the system for general covariance matrix $\vSigma$ and then specialize to the case $\vSigma=\vI_d$.
\begin{proof}
Throughout the replica computations we use the \emph{Einstein summation convention} (repeated indices mean implicit summation), with feature indices $\alpha,\beta\in\{1,\dots,p\}$, Latin coordinate indices $i,j,k\in\{1,\dots,d\}$, sample indices $\nu,\nu'\in\{1,\dots,n\}$, noise indices $\mu,\mu'\in\{1,\dots,m\}$, and replica indices $a,b\in\{1,\dots,s\}$.
We compute $q(z,\zeta)=\frac1p\Tr(\vU-z\vI_p-\zeta\frac{\vW\vW^T}{d})^{-1}$ with $\vU_{\alpha\beta}=\frac{(\vF\vF^T)_{\alpha\beta}}{nm}$ and $\vU$ replaced by its Gaussian equivalent,
\begin{equation}
    \vU_{\alpha\beta}=\frac{1}{nm}\bigl(\mu_1 \vW_{\alpha i}\vY_i^{\nu\mu}+\mu_*\vOmega_\alpha^{\nu\mu}\bigr)\bigl(\mu_1 \vW_{\beta j}\vY_j^{\nu\mu}+\mu_*\vOmega_\beta^{\nu\mu}\bigr),
    \quad
    \vY_i^{\nu\mu}=e^{-t}\vSigma_{ij}^{1/2}\vx_j^\nu \vv^\mu+\sqrt{\Delta_t}\,\vxi_i^{\nu\mu},
\end{equation}
$\mathbb{E}[\vx_i^\nu \vx_j^{\nu'}]=\delta^{\nu\nu'}\vSigma_{ij}$, $\mathbb{E}[\vxi_i^{\nu\mu}\vxi_j^{\nu'\mu'}]=\delta_{ij}\delta^{\nu\nu'}\delta^{\mu\mu'}$, $\vv^\mu=1$. Write $\vM=\vM(z,\zeta)=\vU-z\vI_p-\zeta\tfrac{\vW\vW^T}{d}$, so that $q(z,\zeta)=\frac1p\Tr\,\vM^{-1}$. The trace of the inverse is the derivative of a log-determinant: since $\partial_z\vM=-\vI_p$ and $\partial_z\log\det\vM=\Tr(\vM^{-1}\partial_z\vM)$,
\begin{equation}\label{eq:app-q-logdet}
    q(z,\zeta)=\frac1p\Tr\,\vM^{-1}=-\frac1p\,\partial_z\log\det\vM=\frac2p\,\partial_z\log\det\vM^{-1/2}.
\end{equation}
We introduce the partition function
\begin{equation}\label{eq:app-Z}
   \mathcal{Z}= \det\vM^{-1/2}=\int\frac{\dd\phi}{(2\pi)^{p/2}}\,e^{-\frac12\phi^T\vM\phi},
\end{equation}
Combining Eq.\ref{eq:app-q-logdet}--Eq.\ref{eq:app-Z}, $q=\frac2p\,\partial_z\log\mathcal{Z}$; as $q$ concentrates, we replace it by its quenched average $\frac2p\,\partial_z\,\mathbb{E}\log\mathcal{Z}$.

The remaining obstacle is the average of a logarithm, which we handle with the replica method \citep{mezard1987spin}: the identity $\log x=\lim_{s\to0}(x^s-1)/s$ trades $\mathbb{E}\log\mathcal{Z}$ for the integer moments $\mathbb{E}\,\mathcal{Z}^s$,\footnote{Throughout the computation we discard non-exponential prefactors, as they give subleading contributions.}
\begin{equation}
    q(z,\zeta)=2\,\partial_z\lim_{s\to0}\lim_{p\to\infty}\frac{1}{ps}\,\mathbb{E}\!\left[\mathcal{Z}^s-1\right],
    \qquad
    \mathcal{Z}^s=\det\!\Big(\vU-z\vI_p-\zeta\tfrac{\vW\vW^T}{d}\Big)^{-s/2}.
\end{equation}
For integer $s$ the $s$-th moment is a product of $s$ identical Gaussian integrals over independent copies (\emph{replicas}) $\phi^a$, $a=1,\dots,s$, namely $\mathcal{Z}^s=\int\prod_{a=1}^s\frac{\dd\phi^a}{(2\pi)^{p/2}}\,e^{-\frac12\sum_a\phi^{aT}\vM\phi^a}$, and we average over the disorder before continuing to real $s\to0$. Carrying out this disorder average,
\begin{align}
    \mathbb{E}[\mathcal{Z}^s]&=\int\prod_a\dd\phi^a\;
    e^{\frac z2\phi^a\cdot\phi^a+\frac{\zeta}{2d}\phi^{aT}\vW\vW^T\phi^a}\notag\\
    &\times\mathbb{E}_{\vW,\vx,\vxi,\vOmega}\!\left[\exp\!\Big(-\tfrac{1}{2nm}\phi_\alpha^a(\mu_1 \vW_{\alpha i}\vY_i^{\nu\mu}+\mu_*\vOmega_\alpha^{\nu\mu})(\mu_1 \vW_{\beta j}\vY_j^{\nu\mu}+\mu_*\vOmega_\beta^{\nu\mu})\phi_\beta^a\Big)\right].
\end{align}
We decouple the $\vW$-dependence from the data with the auxiliary field
\begin{equation}
    1=\int\dd\omega^a\dd\hat\omega^a\;e^{i\hat\omega_i^a(\sqrt p\,\omega_i^a-\phi_\alpha^a \vW_{\alpha i})}.
\end{equation}

\paragraph*{Average over $\vOmega$.} Collecting all $\vOmega$-dependent terms and using the Gaussian toolkit,
\begin{align}
    \mathbb{E}_{\vOmega}&\!\left[\exp\!\Big(-\tfrac{\mu_*^2}{2nm}\phi_\alpha^a\phi_\beta^a\vOmega_\alpha^{\nu\mu}\vOmega_\beta^{\nu\mu}-\tfrac{\mu_1\mu_*}{nm}\phi_\alpha^a\phi_\beta^a \vW_{\alpha i}\vY_i^{\nu\mu}\vOmega_\beta^{\nu\mu}\Big)\right]\\
    &=e^{-\frac12\log\det \vC_\Omega}\,e^{\frac12 \vJ_\Omega^T \vC_\Omega^{-1}\vJ_\Omega},
\end{align}
with, using $\vW_{\alpha i}\phi_\alpha^a=\sqrt p\,\omega_i^a$,
\begin{align}
    (\vC_\Omega)_{\alpha\beta}^{\nu\mu\nu'\mu'}
    &=\vI_p\otimes\vI_n\otimes[\kappa\,\one_m+(1-\kappa)\vI_m]^{-1}+\frac{\mu_*^2}{nm}(\phi_\alpha^a\phi_\beta^a)\otimes\vI_n\otimes\vI_m,\\
    (\vJ_\Omega)_\beta^{\nu\mu}&=\frac{\mu_1\mu_*\sqrt p}{nm}\,\omega_i^a\,\vY_i^{\nu\mu}\,\phi_\beta^a.
\end{align}

\paragraph*{Average over $\vW$.} The only remaining $\vW$-dependent factor is
\begin{equation}
    \mathbb{E}_{\vW}\!\left[e^{-i\hat\omega_i^a\phi_\alpha^a \vW_{\alpha i}}\right]=e^{-\frac12\hat\omega_i^a\phi_\alpha^a\hat\omega_i^b\phi_\alpha^b},
\end{equation}
so that
\begin{align}
    \mathbb{E}[\mathcal{Z}^s]=\int\prod_a\dd\phi^a\dd\omega^a\dd\hat\omega^a\;
    &e^{\frac z2\phi^a\cdot\phi^a+\frac{p\zeta}{2d}\omega^a\cdot\omega^a-\frac12\hat\omega_i^a\phi_\alpha^a\hat\omega_i^b\phi_\alpha^b+i\sqrt p\,\hat\omega_i^a\omega_i^a}\notag\\
    &\times e^{-\frac12\log\det \vC_\Omega}\,\mathbb{E}_{\vx,\vxi}\!\left[e^{-\frac{\mu_1^2p}{2nm}\omega_i^a\omega_j^a \vY_i^{\nu\mu}\vY_j^{\nu\mu}+\frac12 \vJ_\Omega^T \vC_\Omega^{-1}\vJ_\Omega}\right].
\end{align}

\paragraph*{Average over $\vx$ and $\vxi$.} With $\vY_i^{\nu\mu}=e^{-t}\vSigma_{ij}^{1/2}\vx_j^\nu \vv^\mu+\sqrt{\Delta_t}\vxi_i^{\nu\mu}$, the data covariance and the quadratic form in the exponent read
\begin{align}
    \vK_{ij}^{\nu\mu\nu'\mu'}&=\mathbb{E}[\vY_i^{\nu\mu}\vY_j^{\nu'\mu'}]=\vI_n\otimes\bigl(e^{-2t}\vSigma\otimes\one_m+\Delta_t\vI_d\otimes\vI_m\bigr),\\
    (\vC_y)_{ij}^{\nu\mu\nu'\mu'}&=\frac{\mu_1^2p}{nm}\omega_i^a\omega_j^a\,\vI_n\otimes\vI_m-\frac{\mu_1^2\mu_*^2p}{(nm)^2}\omega_i^a\phi_\alpha^a\,\omega_j^b\phi_\beta^b\,(\vC_\Omega^{-1})_{\alpha\beta}^{\nu\mu\nu'\mu'}.
\end{align}
Integrating out $\vx,\vxi$ produces a determinant,
\begin{align}
    \mathbb{E}[\mathcal{Z}^s]=\int\prod_a\dd\phi^a\dd\omega^a\dd\hat\omega^a\;
    &e^{\frac z2\phi^a\cdot\phi^a+\frac{p\zeta}{2d}\omega^a\cdot\omega^a-\frac12\hat\omega_i^a\phi_\alpha^a\hat\omega_i^b\phi_\alpha^b+i\sqrt p\,\hat\omega_i^a\omega_i^a}\notag\\
    &\times e^{-\frac12\log\det \vC_\Omega}\,e^{-\frac12\log\det(\vI+\vK\vC_y)}.
\end{align}

\paragraph*{Integrating over $\hat\omega^a$.} The $\hat\omega$-integral is Gaussian,
\begin{equation}
    \int\dd\hat\omega^a\;e^{-\frac12\hat\omega_i^a\phi_\alpha^a\hat\omega_i^b\phi_\alpha^b+i\sqrt p\,\hat\omega_i^a\omega_i^a}=e^{-\frac d2\log\det(\phi^a\cdot\phi^b)}\,e^{-\frac p2\omega_i^a\omega_i^b(\phi^a\cdot\phi^b)^{-1}}.
\end{equation}

\paragraph*{Introducing the overlaps between replicas.} We insert
\begin{align}
    1&=\int\prod_{ab}\tfrac{\dd \vQ^{ab}\dd\hat{\vQ}^{ab}}{2\pi i/p}\;e^{\hat{\vQ}^{ab}(p\vQ^{ab}-\phi_\alpha^a\phi_\alpha^b)},\\
    1&=\int\prod_{ab}\tfrac{\dd \vR^{ab}\dd\hat{\vR}^{ab}}{2\pi i/d}\;e^{\hat{\vR}^{ab}(d\vR^{ab}-\omega_i^a\omega_i^b)},
    \qquad
    1=\int\prod_{ab}\tfrac{\dd \vH^{ab}\dd\hat{\vH}^{ab}}{2\pi i/d}\;e^{\hat{\vH}^{ab}(d\vH^{ab}-\omega_i^a\vSigma_{ij}\omega_j^b)},
\end{align}
so that
\begin{align}
    \mathbb{E}[\mathcal{Z}^s]={}&\int\prod_a\dd\phi^a\dd\omega^a\,\dd \vQ\dd\hat{\vQ}\dd \vR\dd\hat{\vR}\dd \vH\dd\hat{\vH}\notag\\
    &\times e^{\hat{\vQ}^{ab}(p\vQ^{ab}-\phi_\alpha^a\phi_\alpha^b)+\hat{\vR}^{ab}(d\vR^{ab}-\omega_i^a\omega_i^b)+\hat{\vH}^{ab}(d\vH^{ab}-\omega_i^a\vSigma_{ij}\omega_j^b)}\notag\\
    &\times e^{\frac{pz}{2}\Tr \vQ+\frac{p\zeta}{2}\Tr \vR}\,e^{-\frac d2\log\det \vQ}\,e^{-\frac d2\Tr(\vR\vQ^{-1})}\,e^{-\frac12\log\det(\cdot)},
\end{align}

where the remaining determinant, after simplification, is

\begin{align}
    \log\det(\cdot)={}&n\log\det\Big(\delta^{ab}[\kappa\one_m+(1-\kappa)\vI_m]^{-1}+\tfrac{\mu_*^2\psi_p}{\psi_n m}\vQ^{ab}\vI_m\notag\\
    &-\tfrac{\mu_1^2\mu_*^2\psi_p^2}{\psi_n^2m^2}\bigl(e^{-2t}\vH^{ac}\one_m+\Delta_t\vR^{ac}\vI_m\bigr)\vQ^{cb}\notag\\
    &+\tfrac{\mu_1^2\psi_p}{\psi_n m}\bigl[e^{-2t}\vH\one_m+\Delta_t\vR\vI_m\bigr]\bigl[\delta^{ab}(\kappa\one_m+(1-\kappa)\vI_m)^{-1}+\tfrac{\mu_*^2\psi_p}{\psi_n m}\vQ^{ab}\vI_m\bigr]\Big).
\end{align}

\paragraph*{Integrating out $\phi^a,\omega^a$.} These are now Gaussian,
\begin{equation}
    \int\prod_a\dd\phi^a\dd\omega^a\;e^{-\hat{\vQ}^{ab}\phi_\alpha^a\phi_\alpha^b-\hat{\vR}^{ab}\omega_i^a\omega_i^b-\hat{\vH}^{ab}\omega_i^a\vSigma_{ij}\omega_j^b}=e^{-\frac p2\log\det\hat{\vQ}}\,e^{-\frac12\log\det(\hat{\vR}\otimes\vI_d+\hat{\vH}\otimes\vSigma)},
\end{equation}
leaving a single integral over $(\vQ,\hat{\vQ},\vR,\hat{\vR},\vH,\hat{\vH})$.

\paragraph*{Replica-symmetric diagonal Ansatz.} We adopt the replica-symmetric diagonal ansatz \citep{mezard1987spin} $\vQ^{ab}=q\delta^{ab}$, $\vR^{ab}=r\delta^{ab}$, $\vH^{ab}=h\delta^{ab}$ (and likewise for the conjugates).\footnote{The off-diagonal overlaps $\vQ^{ab}$ for $a\neq b$ vanish at the saddle point; see e.g.~\citet{Ascoli_2020}.} Using $\log\det(a\vI_m+b\one_m)=\log(a+mb)+(m-1)\log a$, the saddle on $\hat q$ enforces $\hat q=1/(2q)$, and the effective action (divided by $2/sd$) reads
\begin{align}
\label{eq:app-action}
    S(q,r,\hat r,h,\hat h)
    =&-\psi_p\log q-r\hat r-h\hat h+\int\dd\rho_{\vSigma}(\lambda)\log(\hat r+\lambda\hat h)-\psi_p zq-\psi_p\zeta r+\log q+\frac rq\notag\\
    &+\psi_n\log\!\left[1+\frac{\mu_1^2e^{-2t}\psi_ph}{\psi_n}+\frac{\mu_1^2\Delta_t\psi_pr}{\psi_nm}+\bigl(1+(m-1)\kappa\bigr)\frac{\mu_*^2\psi_pq}{\psi_nm}\right]\notag\\
    &+\psi_n(m-1)\log\!\left[1+\frac{\mu_1^2\Delta_t\psi_pr}{\psi_nm}+(1-\kappa)\frac{\mu_*^2\psi_pq}{\psi_nm}\right].
\end{align}

\paragraph*{Saddle-point equations.} Define
\begin{align}
    L_1(q,r,h)&=1+\frac{\mu_1^2e^{-2t}\psi_ph}{\psi_n}+\frac{\mu_1^2\Delta_t\psi_pr}{\psi_nm}+\bigl(1+(m-1)\kappa\bigr)\frac{\mu_*^2\psi_pq}{\psi_nm},\\
    L_2(q,r)&=1+\frac{\mu_1^2\Delta_t\psi_pr}{\psi_nm}+(1-\kappa)\frac{\mu_*^2\psi_pq}{\psi_nm}.
\end{align}
Stationarity of Eq.\ref{eq:app-action} gives
\begin{align}
    \partial_q S=0&:\quad\psi_pz=\frac{1-\psi_p}{q}-\frac{r}{q^2}+\frac{(1+(m-1)\kappa)\mu_*^2\psi_p/m}{L_1}+\frac{(m-1)(1-\kappa)\mu_*^2\psi_p/m}{L_2},\\
    \partial_r S=0&:\quad\hat r=\frac1q-\psi_p\zeta+\frac{\mu_1^2\psi_p\Delta_t/m}{L_1}+\frac{(m-1)\mu_1^2\psi_p\Delta_t/m}{L_2},\\
    \partial_h S=0&:\quad\hat h=\frac{\mu_1^2\psi_pe^{-2t}}{L_1},\\
    \partial_{\hat r}S=0&:\quad r=\int\frac{\dd\rho_{\vSigma}(\lambda)}{\hat r+\lambda\hat h},
    \qquad
    \partial_{\hat h}S=0:\quad h=\int\frac{\lambda\,\dd\rho_{\vSigma}(\lambda)}{\hat r+\lambda\hat h}.
\end{align}
For the isotropic case $\rho_{\vSigma}=\delta(\lambda-1)$ one has $r=h=1/(\hat r+\hat h)$, i.e.\ $\hat r+\hat h=1/r$, and $h=r$ so $L_1(q,r,r)=L_1(q,r)$. Adding the $\hat r$ and $\hat h$ equations,
\begin{align}
    \frac1r+\psi_p\zeta&=\frac1q+\frac{\mu_1^2\psi_p(e^{-2t}+\frac{\Delta_t}{m})}{L_1(q,r)}+\frac{(m-1)\mu_1^2\psi_p\Delta_t/m}{L_2(q,r)},\\
    \psi_pz&=\frac{1-\psi_p}{q}-\frac{r}{q^2}+\frac{(1+(m-1)\kappa)\mu_*^2\psi_p/m}{L_1(q,r)}+\frac{(m-1)(1-\kappa)\mu_*^2\psi_p/m}{L_2(q,r)},
\end{align}
with $L_1(q,r)=1+\frac{\psi_p}{\psi_nm}[\mu_1^2r(e^{-2t}m+\Delta_t)+\mu_*^2q(1+(m-1)\kappa)]$ and $L_2(q,r)=1+\frac{\psi_p}{\psi_nm}[\mu_1^2r\Delta_t+\mu_*^2q(1-\kappa)]$. 
\end{proof}

\subsection{Computation of the traces}
\label{app:comp-traces}

\subsubsection{Computation of \texorpdfstring{$T_1$, $T_4$ and $T_5$}{T1, T4 and T5}}
We compute here the generating trace Eq.\ref{eq:app-T1T2} at $\epsilon=0$, which by Eq.\ref{eq:app-T345} also yields $T_4=\partial_\zeta T_1$ and $T_5=\partial_z T_1$,
\begin{equation}
    T_1(z,\zeta,0)=\frac{1}{d(nm)^2}\,\E\!\left[\Tr\!\big(\vxi\vF^{T}\mathcal{G}_{z,\zeta,0}\,\vF\vxi^{T}\big)\right],
    \quad
    \mathcal{G}_{z,\zeta,0}=\Big(\tfrac{\vF\vF^{T}}{nm}-z\vI_p-\zeta\tfrac{\vW\vW^{T}}{d}\Big)^{-1}.
\end{equation}
Using the replica representation of an inverse matrix Eq.\ref{eq:replica-inverse}, $$(\mathcal{G}_{z,\zeta,0})_{\alpha\beta}=\int\prod_a\dd\phi^a\,\phi^1_\alpha\phi^1_\beta\,e^{-\frac12\phi^a_{\alpha'}(\mathcal{G}_{z,\zeta,0}^{-1})_{\alpha'\beta'}\phi^a_{\beta'}},$$ we obtain
\begin{multline}
    T_1(z,\zeta,0)=\frac{1}{dnm}\,\vxi^{\nu\mu}_i\frac{\vF^{\nu\mu}_\alpha}{\sqrt{nm}}
    \\\times\int\prod_a\dd\phi^a\;\phi^1_\alpha\phi^1_\beta\;e^{-\frac12\phi^a_{\alpha'}\big(\frac{\vF^{\nu'\mu'}_{\alpha'}\vF^{\nu'\mu'}_{\beta'}}{nm}-z\delta_{\alpha'\beta'}-\zeta\frac{\vW_{\alpha'k}\vW_{\beta'k}}{d}\big)\phi^a_{\beta'}}\frac{\vF^{\nu''\mu''}_\beta}{\sqrt{nm}}\vxi^{\nu''\mu''}_i.
\end{multline}
We decouple the dependence on $\vW$ with the auxiliary field $1=\int\dd\omega^a\dd\hat\omega^a\,e^{i\hat\omega_i^a(\sqrt p\,\omega_i^a-\phi^a_\alpha \vW_{\alpha i})}$ and replace the contraction $\vF^{\nu\mu}_\alpha\phi^a_\alpha$ by the scalar Gaussian field
\begin{equation}
    \vG_a^{\nu\mu}=\mu_1\sqrt p\,\frac{\omega^a_i\vY_i^{\nu\mu}}{\sqrt d}+\mu_*\,\phi^a_\alpha\vOmega^{\nu\mu}_\alpha .
\end{equation}
As in the proof of Proposition~\ref{prop:app-saddle}, we introduce the replica overlaps $\vQ^{ab}=\frac1p\phi^a_\alpha\phi^b_\alpha$ and $\vR^{ab}=\frac1d\omega^a_i\omega^b_i$. The covariances of $\vG$ follow from $\E[\vY_i^{\nu\mu}\vY_j^{\nu'\mu'}]$ and the GEP covariance of $\vOmega$,
\begin{align}
    \E[\vG_a^{\nu\mu}\vG_b^{\nu'\mu'}]&=\mu_1^2 p\,\vR^{ab}\delta^{\nu\nu'}\big(e^{-2t}\vv^\mu \vv^{\mu'}+\Delta_t\delta^{\mu\mu'}\big)+\mu_*^2 p\,\vQ^{ab}\delta^{\nu\nu'}\big(\kappa \vv^\mu \vv^{\mu'}+(1-\kappa)\delta^{\mu\mu'}\big),\\
    \E[\vG_a^{\nu\mu}\vxi_i^{\nu'\mu'}]&=\mu_1\sqrt{\tfrac pd}\,\omega_i^a\sqrt{\Delta_t}\,\delta^{\nu\nu'}\delta^{\mu\mu'}.
\end{align}
Evaluating at the replica-symmetric saddle $(\vQ^{11},\vR^{11})=(q,r)$ of Proposition~\ref{prop:app-saddle}, we only need to compute the non-exponential prefactor: the action and the saddle-point equations are identical to those of Proposition~\ref{prop:app-saddle}. Writing $\langle\,\cdot\,\rangle_*$ for the Gaussian average over $(\vxi,\vG)$ \emph{tilted} by the replica weight $e^{-\frac{1}{2nm}\vG_a\vG_a^{T}}$,
\begin{equation}\label{eq:star-measure}
    \langle\mathcal{O}\rangle_*:=\frac{\mathbb{E}\big[\mathcal{O}\,e^{-\frac{1}{2nm}\vG_a\vG_a^{T}}\big]}{\mathbb{E}\big[e^{-\frac{1}{2nm}\vG_a\vG_a^{T}}\big]},
\end{equation}
the prefactor reads
\begin{equation}
    T_1(z,\zeta,0)=\frac{1}{nm}\,\big\langle \vxi_i^{\nu\mu}\vxi_i^{\nu'\mu'}\vG_1^{\nu\mu}\vG_1^{\nu'\mu'}\big\rangle_*.
\end{equation}
The covariance of $\vG_1$ over the noise index $\mu$ is the $m\times m$ matrix, where $\vv^\mu=1$ denotes that $\vv\in\mathbb{R}^m$ is the all-ones vector (since each data point $\vx^\nu$ is shared identically across all $m$ noise copies),
\begin{equation}
    (\vSigma_G)^{\mu\mu'}=\mu_1^2 r\big(e^{-2t}\vv^\mu \vv^{\mu'}+\Delta_t\delta^{\mu\mu'}\big)+\mu_*^2 q\big(\kappa \vv^\mu \vv^{\mu'}+(1-\kappa)\delta^{\mu\mu'}\big).
\end{equation}
Since $\vSigma_G$ has rank-1 plus diagonal structure, the matrix $\vPsi_m=(\vI_m+\frac{p}{nm}\vSigma_G)^{-1}$ has exactly two distinct eigenvalues whose inverses are
\begin{align}
    &L_1=1+\tfrac{\psip}{\psin m}\big[\mu_1^2 r(e^{-2t}m+\Delta_t)+\mu_*^2 q(1+(m-1)\kappa)\big]\ (\text{eigenvector }\vv),\\
    &L_2=1+\tfrac{\psip}{\psin m}\big[\mu_1^2 r\Delta_t+\mu_*^2 q(1-\kappa)\big]\ (\text{mult. }m-1),
\end{align}
so $\Tr\vPsi_m=\frac1{L_1}+\frac{m-1}{L_2}$.
The measure $\langle\cdot\rangle_*$ is Gaussian. At the RS saddle, $\vxi^{\nu\mu}$ is a standard Gaussian and $\vG_1^{\nu\mu}=\mu_1\sqrt{p}\,\omega^1_i\vY_i^{\nu\mu}/\sqrt{d}+\mu_*\phi^1_\alpha\vOmega^{\nu\mu}_\alpha$ is a linear function of the Gaussian variables $(\vxi,\vx^\nu,\vOmega)$, so $(\vxi,\vG_1)$ are jointly Gaussian. The tilt $e^{-\frac{1}{2nm}\sum_{\nu,\mu}(\vG_1^{\nu\mu})^2}$ is a quadratic exponential in $\vG_1$; tilting a Gaussian by a Gaussian weight yields another Gaussian (it shifts the precision matrix of $\vG_1$ by $\frac{1}{nm}\vI$), so $\langle\cdot\rangle_*$ is a Gaussian measure and Wick's theorem applies.

According to Wick's theorem, for jointly Gaussian fields, $\langle A_1A_2B_1B_2\rangle=\langle A_1A_2\rangle\langle B_1B_2\rangle+\langle A_1B_1\rangle\langle A_2B_2\rangle+\langle A_1B_2\rangle\langle A_2B_1\rangle$. With $A_k=\vxi_i^{\nu_k\mu_k}$ and $B_k=\vG_1^{\nu_k\mu_k}$,
\begin{align}
    \big\langle \vxi_i^{\nu\mu}\vxi_i^{\nu'\mu'}\vG_1^{\nu\mu}\vG_1^{\nu'\mu'}\big\rangle_*&=\underbrace{\langle\vxi^{\nu\mu}\vxi^{\nu'\mu'}\rangle_*\langle \vG^{\nu\mu}\vG^{\nu'\mu'}\rangle_*}_{\mathrm{P1}}
    +\underbrace{\langle\vxi^{\nu\mu}\vG^{\nu\mu}\rangle_*\langle\vxi^{\nu'\mu'}\vG^{\nu'\mu'}\rangle_*}_{\mathrm{P2}}\\
    &
    +\underbrace{\langle\vxi^{\nu'\mu'}\vG^{\nu\mu}\rangle_*\langle\vxi^{\nu\mu}\vG^{\nu'\mu'}\rangle_*}_{\mathrm{P3}}.
\end{align}
The tilt Eq.\ref{eq:star-measure} shifts the two-point functions. For jointly Gaussian $(\vxi,\vG)$, tilting by $e^{-\frac{1}{2nm}\vG\vG^T}$ shifts the precision of $\vG$ by $\frac{1}{nm}\vI$: $\vSigma_{GG,*}=(\vSigma_{GG}^{-1}+\frac{1}{nm}\vI)^{-1}$, $\vSigma_{\xi G,*}=\vSigma_{\xi G}(\vI+\frac{1}{nm}\vSigma_{GG})^{-1}$, and $\vSigma_{\xi\xi}$ receives an $O(1/nm)$ correction. Since $\vSigma_{GG}=p\vSigma_G$, one has $(\vI+\frac{p}{nm}\vSigma_G)^{-1}=\vPsi_m$; the noise therefore stays standard, $\langle\vxi_i^{\nu\mu}\vxi_j^{\nu'\mu'}\rangle_*=\delta_{ij}\delta^{\nu\nu'}\delta^{\mu\mu'}+O(1/nm)$, while
\begin{equation}
    \langle \vG^{\nu\mu}\vG^{\nu'\mu'}\rangle_*=\delta^{\nu\nu'}\,p\,(\vSigma_G\vPsi_m)^{\mu\mu'},
    \qquad
    \langle\vxi_i^{\nu\mu}\vG_1^{\nu'\mu'}\rangle_*=\mu_1\sqrt{\psip\Delta_t}\,\omega_i^1\,\delta^{\nu\nu'}(\vPsi_m)^{\mu\mu'}.
\end{equation}
Since $\vPsi_m^{-1}=\vI_m+\frac{p}{nm}\vSigma_G$, multiplying on the right by $\vPsi_m$ gives $\vI_m=\vPsi_m+\frac{p}{nm}\vSigma_G\vPsi_m$, i.e.\ $p\,\vSigma_G\vPsi_m=nm(\vI_m-\vPsi_m)$. The diagonal pairing then collapses (Einstein summation, $\delta_{ii}=d$):
\begin{align}
    \mathrm{P1}&=\frac{1}{d(nm)^2}\,\delta_{ii}\delta^{\nu\nu'}\delta^{\mu\mu'}\,\delta^{\nu\nu'}p(\vSigma_G\vPsi_m)^{\mu\mu'}
    =\frac{dnp}{d(nm)^2}\Tr(\vSigma_G\vPsi_m)\\
    &=\frac1m\big[m-\Tr\vPsi_m\big]
    =1-\frac{\Tr\vPsi_m}{m}.
\end{align}
The crossed pairings use $\frac1d\sum_i(\omega_i^1)^2=r$:
\begin{align}
    \mathrm{P2}&=\frac{1}{d(nm)^2}\langle\vxi_i^{\nu\mu}\vG_1^{\nu\mu}\rangle_*\langle\vxi_i^{\nu'\mu'}\vG_1^{\nu'\mu'}\rangle_*
    =\frac{\mu_1^2\psip\Delta_t}{d(nm)^2}(dr)\,n^2(\Tr\vPsi_m)^2
    =\frac{\mu_1^2\psip\Delta_t r}{m^2}(\Tr\vPsi_m)^2,\\
    \mathrm{P3}&=\frac{1}{d(nm)^2}\langle\vxi_i^{\nu\mu}\vG_1^{\nu'\mu'}\rangle_*\langle\vxi_i^{\nu'\mu'}\vG_1^{\nu\mu}\rangle_*\notag\\
    &=\frac{\mu_1^2\psip\Delta_t}{d(nm)^2}\,\omega_i^1\omega_i^1\,\delta^{\nu\nu'}\delta^{\nu'\nu}\,(\vPsi_m)^{\mu\mu'}(\vPsi_m)^{\mu'\mu}\notag\\
    &=\frac{\mu_1^2\psip\Delta_t}{d(nm)^2}\cdot dr\cdot n\cdot\Tr(\vPsi_m^2)
    =\frac{\mu_1^2\psip\Delta_t}{nm^2}\,r\,\Tr(\vPsi_m^2)=O(1/n),
\end{align}
In the second line of P3, the two Kronecker deltas $\delta^{\nu\nu'}\delta^{\nu'\nu}=\delta^{\nu\nu'}$ force $\nu=\nu'$, collapsing the double sum over samples to $n$ terms; then $\frac1d\sum_i(\omega_i^1)^2=r$ and $\sum_{\mu,\mu'}(\vPsi_m)^{\mu\mu'}(\vPsi_m)^{\mu'\mu}=\Tr(\vPsi_m^2)=\frac{1}{L_1^2}+\frac{m-1}{L_2^2}=O(1)$. The resulting factor $1/n$ makes P3 subleading. Collecting $\mathrm{P1}+\mathrm{P2}$ (and dropping the subleading $\mathrm{P3}$),
\begin{equation}
    T_1(z,\zeta,0)=f_1\big(q,r,0\big),
    \quad
    f_1(q,r,0)=1-\frac{\mathcal{T}}{m}+\frac{\mu_1^2\psip\Delta_t r}{m^2}\,\mathcal{T}^2,
    \quad
    \mathcal{T}=\Tr\vPsi_m=\frac1{L_1}+\frac{m-1}{L_2},
\end{equation}
which is claim \textit{(ii)} of Theorem~\ref{thm:main} for $T_1$ at $\epsilon=0$; the $\epsilon$-dependence of $f_1$ is obtained in Appendix~\ref{app:traces} below. Here $(q,r)=(q(z,\zeta,0),r(z,\zeta,0))$ is the saddle of Proposition~\ref{prop:app-saddle}, and $T_4=\partial_\zeta T_1|_{(-\lambda,0,0)}$, $T_5=\partial_z T_1|_{(-\lambda,0,0)}$ follow by differentiating $f_1$ through that saddle.

\subsubsection{Computation of $T_2$}
We wish to compute $T_2=\frac{1}{dnm}\Tr\!\big(\vxi\vF^T\mathcal{G}_{-\lambda,0,0}\frac{\vW}{\sqrt d}\big)$ of Eq.\ref{eq:app-T1T2}, with $\mathcal{G}_{-\lambda,0,0}=(\frac{\vF\vF^T}{nm}+\lambda\vI_p)^{-1}$. Applying the replica representation Eq.\ref{eq:replica-inverse} to $(\mathcal{G}_{-\lambda,0,0})_{\alpha\beta}$:
\begin{equation}
    T_2=\frac{1}{dnm}\,\E\!\left[\vxi^{\nu\mu}_i\vF^{\nu\mu}_\alpha\int\prod_a\dd\phi^a\;\phi^1_\alpha\phi^1_\beta\;
    e^{-\frac12\phi^a_{\alpha'}\!\left(\frac{\vF^{\nu'\mu'}_{\alpha'}\vF^{\nu'\mu'}_{\beta'}}{nm}+\lambda\delta_{\alpha'\beta'}\right)\!\phi^a_{\beta'}}\frac{\vW_{i\beta}}{\sqrt d}\right].
\end{equation}
We introduce the auxiliary field $1=\int\dd\omega^a\dd\hat\omega^a\,e^{i\hat\omega_i^a(\sqrt p\,\omega_i^a-\phi^a_\alpha\vW_{\alpha i})}$ and the replica overlaps $\vQ^{ab}=\frac1p\phi^a_\alpha\phi^b_\alpha$, $\vR^{ab}=\frac1d\omega^a_i\omega^b_i$, and replace $\vF^{\nu\mu}_\alpha\phi^a_\alpha$ by
\begin{equation}
    \vG_a^{\nu\mu}=\mu_1\sqrt p\,\frac{\omega^a_i\vY_i^{\nu\mu}}{\sqrt d}+\mu_*\,\phi^a_\alpha\vOmega^{\nu\mu}_\alpha.
\end{equation}
The contraction $\phi^1_\beta\frac{\vW_{i\beta}}{\sqrt d}=\sqrt{\psip}\,\omega^1_i$ then follows from the decoupling. The exponential action is identical to that in the $T_1$ computation above, so evaluating at the saddle $(q,r)$ we only need the prefactor. Since $\sqrt{\psip}\,\omega^1_i$ is constant with respect to the $\langle\cdot\rangle_*$ average, it factors out, and the two-point function $\langle\vxi_i^{\nu\mu}\vG_1^{\nu\mu}\rangle_*=\mu_1\sqrt{\psip\Delta_t}\,\omega^1_i\,(\vPsi_m)^{\mu\mu}$ gives
\begin{equation}
    T_2=\frac{\sqrt{\psip}}{dnm}\sum_{i,\nu,\mu}\omega^1_i\,\langle\vxi_i^{\nu\mu}\vG_1^{\nu\mu}\rangle_*
    =\frac{\psip\mu_1\sqrt{\Delta_t}}{dnm}\sum_{i,\nu,\mu}(\omega^1_i)^2\,(\vPsi_m)^{\mu\mu}.
\end{equation}
Summing the indices: $\frac{1}{d}\sum_i(\omega^1_i)^2=r$, the $n$ sample indices contribute a factor $n$, and $\sum_\mu(\vPsi_m)^{\mu\mu}=\Tr\vPsi_m$, so
\begin{equation}
    T_2=f_2(q,r),\qquad
    f_2(q,r)=\frac{\psip}{m}\mu_1\sqrt{\Delta_t}\,r\,\Tr\vPsi_m=\frac{\psip}{m}\mu_1\sqrt{\Delta_t}\,r\!\left(\frac{1}{L_1(q,r)}+\frac{m-1}{L_2(q,r)}\right),
\end{equation}
with $(q,r)=(q(-\lambda,0,0),r(-\lambda,0,0))$. This is claim \textit{(ii)} of Theorem~\ref{thm:main} for $T_2$.

\subsubsection{Computation of \texorpdfstring{$T_3$}{T3}}
\label{app:traces}
Recall from Eq.\ref{eq:app-T345} that $T_3=-\partial_\epsilon T_1|_{(-\lambda,0,0)}$, so that computing it amounts to switching on the third deformation of the resolvent Eq.\ref{eq:app-resolvent}. Recall also that $\vH=\E_{\vxi}[\vF]$ has entries $\vH_\alpha^\nu=e^{-t}\mu_1\frac{\vW_\alpha\cdot\vx^\nu}{\sqrt{d}}+\mu_*\sqrt{\kappa}\,\veta_\alpha^\nu$, where $\veta_\alpha^\nu\sim\mathcal{N}(0,1)$ is independent of $\vW$, $\vx^\nu$, and $\vxi^{\nu\mu}$ (see Appendix~\ref{app:gep}). The correlations of $\veta$ with the other random variables are
\begin{equation}
    \E[\veta_\alpha^\nu\veta_\beta^{\nu'}]=\delta_{\alpha\beta}\delta^{\nu\nu'},\quad
    \E[\vOmega_\alpha^{\nu\mu}\veta_\beta^{\nu'}]=\sqrt{\kappa}\,\delta_{\alpha\beta}\delta^{\nu\nu'},\quad
    \E[\vxi_i^{\nu\mu}\veta_\alpha^{\nu'}]=0.
\end{equation}
Writing the full resolvent $\mathcal{G}_{z,\zeta,\epsilon}$ of Eq.\ref{eq:app-resolvent} with the replica representation, introducing $\omega=\vW^{T}\phi/\sqrt p$ as before, and denoting $\vG_a^{\nu\mu}=\mu_1\sqrt p\,\frac{\omega_i^a\vY_i^{\nu\mu}}{\sqrt d}+\mu_*\phi_\alpha^a\vOmega_\alpha^{\nu\mu}$ as in the computation of $T_1(z,\zeta,0)$,
\begin{align}
    T_1(z,\zeta,\epsilon)=\frac{1}{d(nm)^2}\E\!\int\prod_a&\dd\phi^a\dd\omega^a\dd\hat\omega^a\;
    e^{i\hat\omega^a(\sqrt p\,\omega^a-\phi^a\vW)}\,\Tr(\vxi \vG_1\vG_1^T\vxi^T)\notag\\
    &\times e^{\frac{z}{2}\phi^a\cdot\phi^a+\frac{\zeta}{2d}\phi^{aT}\vW\vW^T\phi^a
      -\frac{1}{2nm}\vG_a\vG_a^T-\frac{\epsilon}{2n}\vH_\alpha^\nu\phi^a_\alpha\vH_\beta^\nu\phi^a_\beta}.
\end{align}
We introduce the replica overlaps $\vQ^{ab}=\frac1p\phi^a_\alpha\phi^b_\alpha$, $\vR^{ab}=\frac1d\omega^a_i\omega^b_i$, and denote by $\vL_a^\nu=\vH_\alpha^\nu\phi_\alpha^a=e^{-t}\mu_1\sqrt{\psi_p}\,\omega_i^a\vx_i^\nu+\mu_*\sqrt{\kappa}\,\phi_\alpha^a\veta_\alpha^\nu$ the Gaussian field associated with the term $\epsilon\frac{\vH\vH^T}{n}$ in $\mathcal{G}_{z,\zeta,\epsilon}$, with covariances
\begin{equation}
\begin{aligned}
    \E[\vL^\nu \vL^{\nu'}]&=\delta^{\nu\nu'}p\,\sH,
    &\E[\vG^{\nu\mu}\vL^{\nu'}]&=\delta^{\nu\nu'}p\,\sH\,\vv^\mu,\\
    \E[\vxi_i^{\nu\mu}\vG_a^{\nu'\mu'}]&=\mu_1\sqrt{\psi_p\Delta_t}\,\omega_i^a\delta^{\nu\nu'}\delta^{\mu\mu'},
    &\E[\vxi^{\nu\mu}\vL^{\nu'}]&=0,
\end{aligned}
\end{equation}
where $r=\vR^{11}$ and $q=\vQ^{11}$ are the RS saddle values and $\sH=\mu_1^2e^{-2t}r+\mu_*^2\kappa q$. The joint covariance matrix of $(\vG,\vL)$ is accordingly
\begin{equation}
    \vSigma_{(\vG,\vL)}=\begin{pmatrix}p\vSigma_G & p\sH\,\vv\\[4pt] p\sH\,\vv^T & p\sH\end{pmatrix}.
\end{equation}

\paragraph*{Computation of the action.} We denote $\vSigma_G$ the covariance matrix of $\vG$, $\E[\vG_1^{\nu\mu}\vG_1^{\nu'\mu'}]=\delta^{\nu\nu'}(\vSigma_G)^{\mu\mu'}$,
\begin{equation}
    (\vSigma_G)^{\mu\mu'}=\mu_1^2 r\bigl(e^{-2t}\vv^\mu\vv^{\mu'}+\Delta_t\delta^{\mu\mu'}\bigr)+\mu_*^2 q\bigl(\kappa\vv^\mu\vv^{\mu'}+(1-\kappa)\delta^{\mu\mu'}\bigr),
\end{equation}
and $\vPsi_m=(\vI_m+\frac{p}{nm}\vSigma_G)^{-1}$ with eigenvalues $L_1(q,r)$ (eigenvector $\vv$) and $L_2(q,r)$ (multiplicity $m-1$), as derived in the computation of $T_1(z,\zeta,0)$ above. The precision matrix of $(\vG,\vL)$ is shifted by $\vK_{(\vG,\vL)}=\mathrm{diag}(\frac{1}{nm}\vI_m,\frac{\epsilon}{n})$, so $\E[e^{-\frac{1}{2nm}\vG\vG^T-\frac{\epsilon}{2n}\vL\vL}]=e^{-\frac12\log\det(\vI+\vSigma_{(\vG,\vL)} \vK_{(\vG,\vL)})-\frac12\log\det\vSigma_{(\vG,\vL)}^{-1}}$. With
\begin{equation}
    \vI+\vSigma_{(\vG,\vL)}\vK_{(\vG,\vL)}=\begin{pmatrix}\vPsi_m^{-1}&\tfrac{\epsilon p\sH}{n}\vv\\[4pt]\tfrac{p\sH}{nm}\vv^T&1+\tfrac{\epsilon p\sH}{n}\end{pmatrix},
\end{equation}
Define $b(q,r)=\frac{\psip\sH}{\psin}\!\left(1-\frac{\psip\sH}{\psin L_1(q,r)}\right)$. We apply the Schur determinant identity $\det\bigl(\begin{smallmatrix}A&B\\C&D\end{smallmatrix}\bigr)=\det(A)\cdot\det(D-CA^{-1}B)$ with $A=\vPsi_m^{-1}$, $B=\frac{\epsilon p\sH}{n}\vv$, $C=\frac{p\sH}{nm}\vv^T$, $D=1+\frac{\epsilon p\sH}{n}$. Using $\vv^T\vPsi_m\vv=m/L_1(q,r)$, the Schur complement evaluates to
\begin{equation}
    D-CA^{-1}B = 1+\frac{\epsilon p\sH}{n}-\frac{\epsilon p^2{\sHsquare}}{n^2m}\cdot\frac{m}{L_1(q,r)}
    = 1+\frac{\epsilon p\sH}{n}\!\left(1-\frac{p\sH}{nL_1(q,r)}\right)=1+\epsilon\,b(q,r),
\end{equation}
so $\det(\vI+\vSigma_{(\vG,\vL)}\vK_{(\vG,\vL)})=\det(\vPsi_m^{-1})\cdot(1+\epsilon b)=(1+\epsilon b)/\det\vPsi_m$, giving
\begin{equation}
    -\log\det(\vI+\vSigma_{(\vG,\vL)}\vK_{(\vG,\vL)})=\log\det\vPsi_m-\log(1+\epsilon b).
\end{equation}
Hence the new action reads, with $S_0$ denoting the action at $\epsilon=0$,
\begin{equation}
    S_\epsilon(q,r)=S_0(q,r)+\psi_n\log[1+\epsilon\,b(q,r)].
\end{equation}
Since the $\zeta$-deformation only adds the term $-\psi_p\zeta r$ to the action, which is $\epsilon$-independent and does not touch the prefactor, it carries through unchanged and the new saddle point equations read
\begin{align}
    \frac1r+\psi_p\zeta&=\frac1q+\frac{\psi_p\mu_1^2(e^{-2t}+\frac{\Delta_t}{m})}{L_1(q,r)}+\frac{(m-1)\psi_p\mu_1^2\frac{\Delta_t}{m}}{L_2(q,r)}+\frac{\epsilon\psi_p[\mu_1^2e^{-2t}\gamma+\sH\partial_r\gamma]}{1+\epsilon b},\label{eq:app-def-r}\\
    z\psi_p&=\frac{1-\psi_p}{q}-\frac{r}{q^2}+\frac{\psi_p\mu_*^2\frac{1+(m-1)\kappa}{m}}{L_1(q,r)}+\frac{(m-1)\psi_p\mu_*^2\frac{1-\kappa}{m}}{L_2(q,r)}+\frac{\epsilon\psi_p[\mu_*^2\kappa\gamma+\sH\partial_q\gamma]}{1+\epsilon b},\label{eq:app-def-q}
\end{align}
with $\gamma(q,r)=1-\frac{\psi_p\sH}{\psi_nL_1(q,r)}$ and
\begin{align}
    &\partial_r\gamma=-\frac{\psi_p\mu_1^2e^{-2t}}{\psi_nL_1(q,r)}+\frac{\psi_p^2\mu_1^2(e^{-2t}m+\Delta_t)\sH}{\psi_n^2mL_1(q,r)^2},\\
    &\partial_q\gamma=-\frac{\psi_p\mu_*^2\kappa}{\psi_nL_1(q,r)}+\frac{\psi_p^2\mu_*^2(1+(m-1)\kappa)\sH}{\psi_n^2mL_1(q,r)^2}.
\end{align}

\paragraph*{Computation of the prefactor.} The tilted two-point functions are modified by the GG block of $(\vI+\vSigma \vK)^{-1}$. Applying the block matrix inversion formula $(M^{-1})_{11}=(A-BD^{-1}C)^{-1}$ with $A=\vPsi_m^{-1}$, $B=\frac{\epsilon p\sH}{n}\vv$, $C=\frac{p\sH}{nm}\vv^T$, $D=1+\frac{\epsilon p\sH}{n}$, and using $\vPsi_m\vv=\vv/L_1(q,r)$, the GG block of $(\vI+\vSigma_{(\vG,\vL)}\vK_{(\vG,\vL)})^{-1}$ is
\begin{equation}
    \vPsi_{GG}^{(\epsilon)}=\bigl(A-BD^{-1}C\bigr)^{-1}=\vPsi_m+\frac{\epsilon \chi}{1+\epsilon b}\,\vv\vv^T,
    \qquad
    \chi=\frac{\psi_p^2\sHsquare}{\psi_n^2mL_1^2},
\end{equation}
so that $\Tr\vPsi_{GG}^{(\epsilon)}=\frac{1}{L_1}+m\,\alpha(\epsilon)+\frac{m-1}{L_2}$ with $\alpha(\epsilon)=\frac{\epsilon \chi}{1+\epsilon b}$. Replacing $\vPsi_m$ by $\vPsi_{GG}^{(\epsilon)}$ in the formula for $T_1(z,\zeta,0)$ gives, evaluating at the new saddle point $(q_\epsilon,r_\epsilon)$ by the same argument as in the computation of $T_1(z,\zeta,0)$, the function $f_1$ of Theorem~\ref{thm:main} in full,
\begin{equation}\label{eq:app-f1}
\begin{aligned}
    T_1(z,\zeta,\epsilon)&=f_1(q,r,\epsilon),
    \qquad
    f_1(q,r,\epsilon)=1-\frac{\mathcal{T}(\epsilon)}{m}+\frac{\mu_1^2\psi_p\Delta_t\,r}{m^2}\,\mathcal{T}(\epsilon)^2,\\
    \mathcal{T}(\epsilon)&=\Tr\vPsi_{GG}^{(\epsilon)}=\frac{1}{L_1}+\frac{m-1}{L_2}+\frac{m\,\epsilon\,\chi}{1+\epsilon b},
\end{aligned}
\end{equation}
with $(q,r)=(q(z,\zeta,\epsilon),r(z,\zeta,\epsilon))$ the solution of Eq.\ref{eq:app-def-r}--Eq.\ref{eq:app-def-q}. Setting $\epsilon=0$ recovers the expression of the previous subsection. Hence
\begin{equation}
    T_3=-\partial_\epsilon T_1\big|_{(-\lambda,0,0)}=-\frac{\dd f_1(q_\epsilon,r_\epsilon,\epsilon)}{\dd\epsilon}\Big|_{\epsilon=0},
\end{equation}
the $\epsilon$-dependence entering both explicitly (through $\alpha$) and implicitly (through $q_\epsilon,r_\epsilon$).

\subsection{Bias and variance of the score estimator}
\label{app:bias_variance}
We compute here the bias and variance of the decomposition Eq.\ref{eq:bias_variance} of the main text as a function of $T_2$, $T_4$ and $T_5$ defined above. With the shorthands $T_4=\partial_\zeta T_1|_{(-\lambda,0,0)}$ and $T_5=\partial_z T_1|_{(-\lambda,0,0)}$ of Eq.\ref{eq:app-T345}, Eq.\ref{eq:app-bv-final} below is exactly Eq.~\ref{eq:bias_variance_traces} of the main text. We emphasize that the proof relies on an isotropy argument and thus does not apply for anisotropic data.

\paragraph*{The exact score is linear.}
Since $\vx^\nu\sim\mathcal{N}(0,\vI_d)$ and $e^{-2t}+\Delta_t=1$, the noised marginal is exactly $P_t=\mathcal{N}(0,\vI_d)$, so
\begin{equation}\label{eq:app-bv-exact}
    \vs_{\mathrm{exact}}(\vy)=\nabla_\vy\log P_t(\vy)=-\vy,
\end{equation}
hence
\begin{equation}\label{eq:app-bv-identity}
    \Ltest(\vs)=e^{-2t}+\frac{\Delta_t}{d}\,\E_\vy\lVert\vs(\vy)-\vs_{\mathrm{exact}}(\vy)\rVert^2.
\end{equation}

\paragraph*{The mean estimator is linear.}
The joint law of $(\vW,\{\vx^\nu\},\{\vxi^{\nu\mu}\})$ is invariant under the simultaneous rotation $\vx\mapsto \vR\vx$, $\vxi\mapsto \vR\vxi$, $\vW\mapsto\vW \vR^T$ for any $\vR\in O(d)$, and the ridge ERM is equivariant, $\vs_{\vR\mathcal{D},\vW \vR^T}(\vR\vy)=\vR\,\vs_{\mathcal{D},\vW}(\vy)$, since loss and penalty are invariant under $(\vA,\mathcal{D},\vW)\mapsto(\vR\vA,\vR\mathcal{D},\vW \vR^T)$ and the minimizer is unique for $\lambda>0$. Averaging, $\langle\vs_\mathcal{D}(\vR\vy)\rangle=\vR\langle\vs_\mathcal{D}(\vy)\rangle$ for all $\vR$. Writing $\langle\vs_\mathcal{D}(\vy)\rangle=a\,\hat\vy+\vv$ with $\vv\perp\vy$ and applying the element of the stabilizer of $\vy$ acting as $-\mathrm{id}$ on $\vy^\perp$ gives $\vv=-\vv=0$; the amplitude then depends on $\vy$ only through $\lVert\vy\rVert$. Hence $\langle\vs_\mathcal{D}(\vy)\rangle=c(\lVert\vy\rVert)\,\vy$ at any finite $d$, and since $\lVert\vy\rVert^2/d\to1$ concentrates in the proportional limit,
\begin{equation}\label{eq:app-bv-linear}
    \langle\vs_\mathcal{D}(\vy)\rangle=c\,\vy, 
\end{equation}
with $c$ a deterministic constant. Projecting Eq.\ref{eq:app-bv-linear} on $\vy$ and using $\E\lVert\vy\rVert^2=d$ gives $c=\frac1d\E_\vy[\langle\vs_\mathcal{D}(\vy)\rangle\cdot\vy]$. Because $\E[\vxi\mid\vy]=\sqrt{\Delta_t}\vy$, the feature--noise correlation of Appendix~\ref{app:gep} obeys, identically at finite $d$, $\tilde\vV=\frac{1}{\sqrt{\Delta_t}}\E_{\vx,\vxi}[\sigma(\cdot)\vxi^T]=\E_\vy[\sigma(\vW\vy/\sqrt d)\vy^T]$, so computing $c$ needs no approximation beyond the substitution $\tilde\vV\to\mu_1\vW/\sqrt d$ already used for $\Ltest$. With the explicit minimizer $\vA^{\mathrm{ERM}}/\sqrt p=-(\vV_n^m)^T(\vU_n^m+\lambda\vI_p)^{-1}$,
\begin{equation}\label{eq:app-bv-c}
    c=\frac{\mu_1}{d}\left\langle\Tr\!\Big(\tfrac{\vA^{\mathrm{ERM}}}{\sqrt p}\tfrac{\vW}{\sqrt d}\Big)\right\rangle
     =-\frac{\mu_1}{\sqrt{\Delta_t}}\,\frac{1}{dnm}\left\langle\Tr\!\Big(\vxi\vF^T\mathcal{G}_{-\lambda,0,0}\tfrac{\vW}{\sqrt d}\Big)\right\rangle
     =-\frac{\mu_1}{\sqrt{\Delta_t}}\,T_2 .
\end{equation}

\paragraph*{Closed-form bias and variance.}
Since $\vs_{\mathrm{exact}}(\vy)=-\vy$ and $\langle\vs_\mathcal{D}(\vy)\rangle=c\vy$, the bias of Eq.\ref{eq:bias_variance} is immediate. For the variance, expand $\mathcal{V}=\frac1d\E_\vy\langle\lVert\vs\rVert^2\rangle-\frac1d\E_\vy\lVert\langle\vs_\mathcal{D}\rangle\rVert^2$ and note that the second moment is the object already producing the quadratic part of $\Ltest$: with the same substitution $\tilde\vU=\mu_1^2\vW\vW^T/d+\mu_*^2\vI_p$, one has $\frac{\Delta_t}{d}\E_\vy\langle\lVert\vs(\vy)\rVert^2\rangle=\langle\Tr(\frac{\vA}{\sqrt p}\tilde\vU\frac{\vA^T}{\sqrt p})\rangle\frac{\Delta_t}{d}=\mu_1^2T_4+\mu_*^2T_5$. Hence
\begin{equation}\label{eq:app-bv-final}
\mathcal{B}^2=(1+c)^2,
    \qquad
    \mathcal{V}=\frac{\mu_1^2T_4+\mu_*^2T_5}{\Delta_t}-c^2,
    \qquad
    c=-\frac{\mu_1T_2}{\sqrt{\Delta_t}}.
\end{equation}
Using $\Delta_t\cdot2c=-2\mu_1\sqrt{\Delta_t}T_2$ and $e^{-2t}+\Delta_t=1$,
\begin{equation}
    e^{-2t}+\Delta_t\big(\mathcal{B}^2+\mathcal{V}\big)
    =e^{-2t}+\Delta_t(1+2c)+\mu_1^2T_4+\mu_*^2T_5
    =1-2\mu_1\sqrt{\Delta_t}T_2+\mu_1^2T_4+\mu_*^2T_5=\Ltest .
\end{equation}
Substituting the asymptotic value of $T_2$ obtained above yields the explicit slope
\begin{equation}\label{eq:app-bv-c-closed}
    c=-\frac{\psip}{m}\,\mu_1^2\,r\left(\frac{1}{L_1(q,r)}+\frac{m-1}{L_2(q,r)}\right),
\end{equation}
at the saddle $(q,r)$ of Proposition~\ref{prop:app-saddle} evaluated at $(z,\zeta)=(-\lambda,0)$.

\subsection{The overparameterized limit \texorpdfstring{$\psi_p\to\infty$}{psip -> infinity}}
\label{app:psip_infty_limit}
We consider two scalings of the ridge strength as $\psi_p\to\infty$: regime~(i), fixed $\lambda=O(1)$ (same scaling as \citet{george_2025}) and regime~(ii), $\lambda=\tilde\lambda\,\psi_p$ with $\tilde\lambda>0$ fixed (same scaling as \citet{mei2020}).

\subsubsection{Regime (i): fixed regularization \texorpdfstring{$\lambda=O(1)$}{lambda = O(1)}}
\begin{proposition}[Overparameterized limit, fixed regularization]\label{prop:app-psip-inf}
Define
\begin{equation*}
    \ell_1^\star(r)=\mu_1^2(e^{-2t}m+\Delta_t)r+\frac{\mu_*^2(1+(m-1)\kappa)}{\lambda},
    \qquad
    \ell_2^\star(r)=\mu_1^2\Delta_t r+\frac{\mu_*^2(1-\kappa)}{\lambda},
\end{equation*}
and
\begin{align}
    &g(r)=\frac{1}{\ell_1^\star(r)}+\frac{m-1}{\ell_2^\star(r)},
    \quad
    R(r)=\mu_1^2\psi_n\!\left[\frac{e^{-2t}m+\Delta_t}{\ell_1^\star(r)}+\frac{(m-1)\Delta_t}{\ell_2^\star(r)}\right],\nonumber \\
    &\Phi(r)=\psi_n g(r)-\mu_1^2\Delta_t\psi_n^2 r g(r)^2 .
\end{align}
For the trace $T_3$, let $r(\tilde\epsilon)$ solve the deformed fixed point $\frac1r=\lambda+R(r)+\frac{\tilde\epsilon K(r)}{1+\tilde\epsilon\beta(r)}$ with $r(0)=r_\infty$, where
\begin{equation*}
  \sigma_\star^2(r)=\mu_1^2e^{-2t}r+\tfrac{\mu_*^2\kappa}{\lambda},\ \
  \gamma_\infty=1-\tfrac{m\sigma_\star^2}{\ell_1^\star},\ \
  K=\mu_1^2e^{-2t}\gamma_\infty+\sigma_\star^2\gamma_\infty',\ \
  \beta=\tfrac{\sigma_\star^2\gamma_\infty}{\psi_n},\ \
  \chi_\infty=\tfrac{m\sigma_\star^4}{(\ell_1^\star)^2},
\end{equation*}
and set $G(\tilde\epsilon)=g(r)+\frac{1}{\psi_n}\frac{\tilde\epsilon \chi_\infty}{1+\tilde\epsilon\beta}$ and $\hat\Phi(\tilde\epsilon)=\psi_n G-\mu_1^2\Delta_t\psi_n^2 r G^2$.
Let $r_\infty$ be the unique root in $(0,1/\lambda)$ of $\frac1r=\lambda+R(r)$, and set $D=\frac{1}{r_\infty^2}+R'(r_\infty)$. Then, as $\psi_p\to\infty$ at fixed $(\lambda>0,\psi_n,m,t)$, the solutions of Proposition~\ref{prop:app-saddle} converge to $q_\infty=1/\lambda$, $r\to r_\infty$, the traces converge to
\begin{equation*}
    T_1\to1,\quad T_5\to0,\quad T_2\to\mu_1\sqrt{\Delta_t}\,\psi_n r_\infty g(r_\infty),\quad T_4\to-\frac{\Phi'(r_\infty)}{D},\quad T_3\to T_3^\infty=\hat\Phi'(0),
\end{equation*}
and the losses converge to
\begin{align*}
\Ltrain(\psi_p=\infty)&=0,
    \qquad
 \Ltest(\psi_p=\infty)=1-2\mu_1^2\Delta_t\psi_n r_\infty g(r_\infty)-\mu_1^2\frac{\Phi'(r_\infty)}{D},\\
  \Linfty(\psi_p=\infty)&=1-2\mu_1^2\Delta_t\psi_n r_\infty g(r_\infty)+\hat\Phi'(0)-\Delta_t\mu_1^2\frac{\Phi'(r_\infty)}{D},
\end{align*}
where $\hat\Phi'(0)$ is the derivative of $\hat\Phi$ at $\tilde\epsilon=0$. 
\end{proposition}

\begin{proof}
Recall the finite-$\psi_p$ equations of Proposition~\ref{prop:app-saddle} at $z=-\lambda$, $\zeta=0$,
\begin{align}
  \frac1r &= \frac1q+\frac{\mu_1^2\psi_p(e^{-2t}+\Delta_t/m)}{L_1}+\frac{(m-1)\mu_1^2\psi_p\Delta_t/m}{L_2},\\
  -\lambda\psi_p &= \frac{1-\psi_p}{q}-\frac{r}{q^2}+\frac{\psi_p\mu_*^2(1+(m-1)\kappa)/m}{L_1}+\frac{(m-1)\psi_p\mu_*^2(1-\kappa)/m}{L_2},
\end{align}
with $L_i=1+\frac{\psi_p}{\psi_nm}\ell_i$, $\ell_1=\mu_1^2r(e^{-2t}m+\Delta_t)+\mu_*^2q(1+(m-1)\kappa)$ and $\ell_2=\mu_1^2r\Delta_t+\mu_*^2q(1-\kappa)$. Here $\ell_1,\ell_2=O(1)$, so
\begin{equation}
    \frac{1}{L_i}=\frac{\psi_nm}{\psi_p\ell_i}\bigl(1+O(\psi_p^{-1})\bigr),
    \qquad
    \frac{\psi_p}{L_i}\to\frac{\psi_nm}{\ell_i}.
    \label{eq:app-Lscaling}
\end{equation}
Dividing the $q$-equation by $\psi_p$, the $1/L_i$ terms drop and $z=-1/q+O(\psi_p^{-1})$, i.e.\ $q_\infty=1/\lambda$. Inserting $1/q\to\lambda$ and Eq.\ref{eq:app-Lscaling} into the $r$-equation at $\zeta=0$ gives the scalar fixed point
\begin{equation}
    \frac{1}{r_\infty}=\lambda+R(r_\infty),
    \qquad
    R(r)=\mu_1^2\psi_n\!\left[\frac{e^{-2t}m+\Delta_t}{\ell_1^\star(r)}+\frac{(m-1)\Delta_t}{\ell_2^\star(r)}\right],
\end{equation}
with $\ell_i^\star(r)=\ell_i(1/\lambda,r)$, i.e.\ $\ell_1^\star(r)=\mu_1^2(e^{-2t}m+\Delta_t)r+\frac{\mu_*^2(1+(m-1)\kappa)}{\lambda}$ and $\ell_2^\star(r)=\mu_1^2\Delta_tr+\frac{\mu_*^2(1-\kappa)}{\lambda}$. The root is unique in $(0,1/\lambda)$. Writing $g(r)=\frac{1}{\ell_1^\star}+\frac{m-1}{\ell_2^\star}$, Eq.\ref{eq:app-Lscaling} gives $\mathcal{T}=\frac{\psi_nm}{\psi_p}g(r)+O(\psi_p^{-2})$, so
\begin{equation}
    T_1=1-\frac{1}{\psi_p}\Phi(r)+O(\psi_p^{-2}),
    \qquad
    \Phi(r)=\psi_ng(r)-\mu_1^2\Delta_t\psi_n^2rg(r)^2,
\end{equation}
hence $T_1\to1$ (exact interpolation, $\Ltrain(\psi_p=\infty)=0$) and $T_2\to\mu_1\sqrt{\Delta_t}\psi_nr_\infty g(r_\infty)$.

For the derivatives, the term $\psi_pz$ in the $q$-equation is balanced by $-\psi_p/q$, so a perturbation $\delta z=O(1)$ is absorbed by $\delta q=O(1)$ and $\partial_zq,\partial_zr=O(1)$, giving $T_5=\partial_z T_1=-\frac{1}{\psi_p}(\partial_q\Phi\,\partial_zq+\partial_r\Phi\,\partial_zr)=O(\psi_p^{-1})\to0$. In the $r$-equation, $\psi_p\zeta$ competes with $O(1)$ terms, so the natural coordinate is $\tilde\zeta=\psi_p\zeta$ and the reduced equation is $\frac1r=-\tilde\zeta+\lambda+R(r)$, giving $\frac{dr}{d\tilde\zeta}=1/D$ with $D=\frac{1}{r_\infty^2}+R'(r_\infty)$. Then $\partial_\zeta=\psi_p\partial_{\tilde\zeta}$ yields
\begin{equation}
    T_4=-\Phi'(r)\frac{dr}{d\tilde\zeta}\ \longrightarrow\ T_4^\infty=-\frac{\Phi'(r_\infty)}{D},
\end{equation}
with $g'(r)=-\mu_1^2[\frac{e^{-2t}m+\Delta_t}{(\ell_1^\star)^2}+\frac{(m-1)\Delta_t}{(\ell_2^\star)^2}]$, $R'(r)=-\mu_1^4\psi_n[\frac{(e^{-2t}m+\Delta_t)^2}{(\ell_1^\star)^2}+\frac{(m-1)\Delta_t^2}{(\ell_2^\star)^2}]$ and $\Phi'(r)=\psi_ng'(r)-\mu_1^2\Delta_t\psi_n^2(g^2+2rgg')$. Collecting the limits,
\begin{equation}
\Ltest(\psi_p=\infty)=1-2\mu_1^2\Delta_t\psi_nr_\infty g(r_\infty)-\mu_1^2\frac{\Phi'(r_\infty)}{D}.
\end{equation}

The trace $T_3$ comes from the $\epsilon$-deformed prefactor of the $T_3$ subsection. In the present limit the deformation parameter is large, $b=O(\psi_p)$, so we rescale $\tilde\epsilon=\psi_p\epsilon$; the deformed saddle then reads $\frac1r=\lambda+R(r)+\frac{\tilde\epsilon K(r)}{1+\tilde\epsilon\beta(r)}$ with
\begin{align}
  &\sigma_\star^2(r)=\mu_1^2e^{-2t}r+\tfrac{\mu_*^2\kappa}{\lambda},\quad
  \gamma_\infty=1-\tfrac{m\sigma_\star^2}{\ell_1^\star},\quad
  K=\mu_1^2e^{-2t}\gamma_\infty+\sigma_\star^2\gamma_\infty',\nonumber\\
  &\beta=\tfrac{\sigma_\star^2\gamma_\infty}{\psi_n},\quad
  \chi_\infty=\tfrac{m\,\sigma_\star^4}{(\ell_1^\star)^2},
\end{align}
and $\gamma_\infty'=\partial_r\gamma_\infty$. Setting $G(\tilde\epsilon)=g(r)+\frac{1}{\psi_n}\frac{\tilde\epsilon \chi_\infty}{1+\tilde\epsilon\beta}$ and $\hat\Phi(\tilde\epsilon)=\psi_nG-\mu_1^2\Delta_t\psi_n^2 rG^2$, the deformed prefactor is $T_1(\tilde\epsilon)=1-\frac{1}{\psi_p}\hat\Phi(\tilde\epsilon)+O(\psi_p^{-2})$, so $T_3^\infty=\hat\Phi'(0)$ with $r'(0)=-K(r_\infty)/D$ and $G'(0)=g'(r_\infty)r'(0)+\chi_\infty(r_\infty)/\psi_n$, i.e.
\begin{equation}
  T_3^\infty=\psi_n G'(0)-\mu_1^2\Delta_t\psi_n^2\big[r'(0)\,g(r_\infty)^2+2r_\infty g(r_\infty)\,G'(0)\big].
\end{equation}
Inserting $T_2$, $T_3^\infty$ and $T_4$, and using $T_5\to0$, into $\Linfty=1-2\mu_1\sqrt{\Delta_t}T_2+T_3+\Delta_t\mu_1^2T_4+\mu_*^2(1-\kappa)T_5$ gives the stated $\Linfty(\psi_p=\infty)$.
\end{proof}

\subsubsection{Regime (ii): strong regularization \texorpdfstring{$\lambda=\tilde\lambda\,\psi_p$}{lambda = lambda-tilde psi_p}}
\begin{proposition}[Overparameterized limit, strong regularization]\label{prop:app-psip-inf-strong}
Let $\lambda=\tilde\lambda\,\psi_p$ with $\tilde\lambda>0$ fixed, and define the scaled order parameters $\bar q=\psi_p q$, $\bar r=\psi_p r$, the $O(1)$ resolvent factors
\begin{align}
  &\bar L_1(\bar q,\bar r)=1+\tfrac{1}{\psi_n m}\big[\mu_1^2(e^{-2t}m+\Delta_t)\bar r+\mu_*^2(1+(m-1)\kappa)\bar q\big],\\
  &\bar L_2(\bar q,\bar r)=1+\tfrac{1}{\psi_n m}\big[\mu_1^2\Delta_t\bar r+\mu_*^2(1-\kappa)\bar q\big],
\end{align}
and, writing $\bar L_i^\star(\bar r)=\bar L_i(1/\tilde\lambda,\bar r)$,
\begin{equation*}
  \bar R(\bar r)=\tfrac1m\Big(\tfrac{\mu_1^2(e^{-2t}m+\Delta_t)}{\bar L_1^\star(\bar r)}+\tfrac{(m-1)\mu_1^2\Delta_t}{\bar L_2^\star(\bar r)}\Big),
  \quad
  \bar g(\bar r)=\tfrac{1}{\bar L_1^\star(\bar r)}+\tfrac{m-1}{\bar L_2^\star(\bar r)},
  \quad
  \bar\Phi(\bar r)=\tfrac{\bar g(\bar r)}{m}-\tfrac{\mu_1^2\Delta_t\bar r}{m^2}\bar g(\bar r)^2.
\end{equation*}
For the trace $T_3$, let $\bar r(\epsilon)$ solve the deformed fixed point $\frac1{\bar r}=\tilde\lambda+\bar R(\bar r)+\frac{\epsilon\bar K(\bar r)}{1+\epsilon\bar b(\bar r)}$ with $\bar r(0)=\bar r_\infty$, where
\begin{align}
  &\bar\sigma_\star^2(\bar r)=\mu_1^2e^{-2t}\bar r+\tfrac{\mu_*^2\kappa}{\tilde\lambda},\ \
  \bar\gamma=1-\tfrac{\bar\sigma_\star^2}{\psi_n\bar L_1^\star},\ \
  \bar K=\mu_1^2e^{-2t}\bar\gamma+\bar\sigma_\star^2\bar\gamma',\ \ \\
  &\bar b=\tfrac{\bar\sigma_\star^2}{\psi_n}\bar\gamma,\ \
  \bar\chi=\tfrac{\bar\sigma_\star^4}{\psi_n^2 m(\bar L_1^\star)^2},
\end{align}
and set $\bar G(\epsilon)=\bar g(\bar r(\epsilon))+m\frac{\epsilon\bar\chi}{1+\epsilon\bar b}$, $\bar\Phi_\epsilon=\frac{\bar G(\epsilon)}{m}-\frac{\mu_1^2\Delta_t\bar r(\epsilon)}{m^2}\bar G(\epsilon)^2$ (so $\bar\Phi_0=\bar\Phi(\bar r_\infty)$).
Let $\bar r_\infty$ be the root of $\frac1{\bar r}=\tilde\lambda+\bar R(\bar r)$ and $\bar D=\frac1{\bar r_\infty^2}+\bar R'(\bar r_\infty)$. Then, as $\psi_p\to\infty$, $\bar q\to1/\tilde\lambda$, $\bar r\to\bar r_\infty$,
\begin{align}
  &T_1\to1-\bar\Phi(\bar r_\infty),\quad T_5\to0,\quad T_2\to\tfrac1m\mu_1\sqrt{\Delta_t}\,\bar r_\infty\bar g(\bar r_\infty),\\
  &T_4\to-\tfrac{\bar\Phi'(\bar r_\infty)}{\bar D},\quad T_3\to\bar T_3^\infty=\tfrac{\dd\bar\Phi_\epsilon}{\dd\epsilon}\big|_{0},
\end{align}
and the losses converge to
\begin{align*}
  \Ltrain(\psi_p=\infty)&=\bar\Phi(\bar r_\infty),
  \qquad
  \Ltest(\psi_p=\infty)=1-\tfrac{2}{m}\mu_1^2\Delta_t\,\bar r_\infty\bar g(\bar r_\infty)-\mu_1^2\tfrac{\bar\Phi'(\bar r_\infty)}{\bar D},\\
   \Linfty(\psi_p=\infty)&=1-\tfrac{2}{m}\mu_1^2\Delta_t\,\bar r_\infty\bar g(\bar r_\infty)+\bar T_3^\infty-\Delta_t\mu_1^2\tfrac{\bar\Phi'(\bar r_\infty)}{\bar D},
\end{align*}
where $\bar T_3^\infty=\frac{\dd\bar\Phi_\epsilon}{\dd\epsilon}\big|_0$ is the derivative at $\epsilon=0$.
\end{proposition}
\begin{proof}
\emph{Recall the finite-$\psi_p$ equations} of Proposition~\ref{prop:app-saddle} at $\zeta=0$,
\begin{align}
  \frac1r &= \frac1q+\frac{\mu_1^2\psi_p(e^{-2t}+\Delta_t/m)}{L_1}+\frac{(m-1)\mu_1^2\psi_p\Delta_t/m}{L_2},\\
  z\psi_p &= \frac{1-\psi_p}{q}-\frac{r}{q^2}+\frac{\psi_p\mu_*^2(1+(m-1)\kappa)/m}{L_1}+\frac{(m-1)\psi_p\mu_*^2(1-\kappa)/m}{L_2},
\end{align}
with $L_i=1+\frac{\psi_p}{\psi_nm}\ell_i$, $\ell_1=\mu_1^2r(e^{-2t}m+\Delta_t)+\mu_*^2q(1+(m-1)\kappa)$ and $\ell_2=\mu_1^2r\Delta_t+\mu_*^2q(1-\kappa)$.
\emph{Take $\psi_p\to\infty$ with $z=-\lambda=-\tilde\lambda\psi_p$.} With $\bar q=\psi_p q$, $\bar r=\psi_p r$, the resolvent factors become $L_i=\bar L_i(\bar q,\bar r)+O(\psi_p^{-1})$; crucially they remain $O(1)$, so---unlike regime~(i)---the $+1$ cannot be dropped. Dividing the $q$-equation by $\psi_p^2$, all terms vanish except the leading one, leaving $-\tilde\lambda=-1/\bar q$, i.e.\ $\bar q_\infty=1/\tilde\lambda$. Dividing the $r$-equation by $\psi_p$ gives $\frac1{\bar r}+\zeta=\frac1{\bar q}+\frac{\mu_1^2(e^{-2t}+\Delta_t/m)}{\bar L_1}+\frac{(m-1)\mu_1^2\Delta_t/m}{\bar L_2}$; at $\zeta=0$, $\bar q=1/\tilde\lambda$ this is the scalar fixed point $\frac1{\bar r}=\tilde\lambda+\bar R(\bar r)$. Substituting $r=\bar r/\psi_p$ into $T_1=1-\frac{\mathcal T}{m}+\frac{\mu_1^2\psi_p\Delta_t r}{m^2}\mathcal T^2$, with $\mathcal T\to\bar g(\bar r_\infty)$ and $\psi_p r\to\bar r_\infty$, gives $T_1\to1-\bar\Phi(\bar r_\infty)$, bounded away from $1$: the training error stays positive. Likewise $T_2\to\frac1m\mu_1\sqrt{\Delta_t}\bar r_\infty\bar g(\bar r_\infty)$. For $T_4=\partial_\zeta T_1$, implicit differentiation of $\frac1{\bar r}+\zeta=\tilde\lambda+\bar R(\bar r)$ gives $-\frac1{\bar r^2}\partial_\zeta\bar r+1=\bar R'(\bar r)\partial_\zeta\bar r$, hence $\partial_\zeta\bar r=1/\bar D$ and $T_4=-\bar\Phi'(\bar r_\infty)/\bar D$. For $T_5=\partial_z T_1$: since $\bar q\approx-\psi_p/z$, $\partial_z\bar q=\psi_p/z^2=1/(\tilde\lambda^2\psi_p)\to0$, and as the $z$-dependence of $\bar r$ is mediated by $\bar q$, $T_5\to0$. Assembling $\Ltest=1-2\mu_1\sqrt{\Delta_t}T_2+\mu_1^2 T_4+\mu_*^2 T_5$ yields the stated $\Ltrain$ and $\Ltest$.

Here the field variance of the $T_3$ deformation is small, $\sigma_H^2=\mu_1^2e^{-2t}r+\mu_*^2\kappa q=\bar\sigma_\star^2/\psi_p$ with $\bar\sigma_\star^2(\bar r)=\mu_1^2e^{-2t}\bar r+\frac{\mu_*^2\kappa}{\tilde\lambda}$, so the deformation parameters stay $O(1)$ and \emph{no} rescaling of $\epsilon$ is needed (contrast regime~(i), where $b=O(\psi_p)$):
\begin{equation}
  \bar\gamma=1-\tfrac{\bar\sigma_\star^2}{\psi_n\bar L_1^\star},\qquad
  \bar b=\tfrac{\bar\sigma_\star^2}{\psi_n}\bar\gamma,\qquad
  \bar\chi=\tfrac{\bar\sigma_\star^4}{\psi_n^2 m(\bar L_1^\star)^2},\qquad
  \bar K=\mu_1^2e^{-2t}\bar\gamma+\bar\sigma_\star^2\bar\gamma',\quad \bar\gamma'=\partial_{\bar r}\bar\gamma.
\end{equation}
The deformed saddle becomes $\frac1{\bar r}=\tilde\lambda+\bar R(\bar r)+\frac{\epsilon\bar K(\bar r)}{1+\epsilon\bar b}$, giving $\bar r'(0)=-\bar K(\bar r_\infty)/\bar D$. With $\bar G(\epsilon)=\bar g(\bar r(\epsilon))+m\frac{\epsilon\bar\chi}{1+\epsilon\bar b}$ (so $\bar G(0)=\bar g(\bar r_\infty)$, $\bar G'(0)=\bar g'(\bar r_\infty)\bar r'(0)+m\bar\chi(\bar r_\infty)$) the deformed prefactor is $T_1(\epsilon)\to1-\bar\Phi_\epsilon$, $\bar\Phi_\epsilon=\frac{\bar G(\epsilon)}{m}-\frac{\mu_1^2\Delta_t\bar r(\epsilon)}{m^2}\bar G(\epsilon)^2$, and since $T_3=-\frac{\dd T_1(\epsilon)}{\dd\epsilon}|_0$,
\begin{equation}
  \bar T_3^\infty=\frac{\dd\bar\Phi_\epsilon}{\dd\epsilon}\Big|_0=\frac{\bar G'(0)}{m}-\frac{\mu_1^2\Delta_t}{m^2}\big[\bar r'(0)\,\bar g(\bar r_\infty)^2+2\bar r_\infty\bar g(\bar r_\infty)\,\bar G'(0)\big].
\end{equation}
Inserting $T_2$, $\bar T_3^\infty$ and $T_4$, and using $T_5\to0$, into $\Linfty=1-2\mu_1\sqrt{\Delta_t}T_2+T_3+\Delta_t\mu_1^2T_4+\mu_*^2(1-\kappa)T_5$ gives the stated $\Linfty(\psi_p=\infty)$.
\end{proof}

\subsection{The \texorpdfstring{$m\to\infty$}{m -> infinity} limit of the equations}
\label{app:m_infty_limit}
\begin{proposition}[$m\to\infty$ limit]\label{prop:app-minf}
As $m\to\infty$ at fixed $(\psi_p,\psi_n,t,\lambda)$,  the order parameters $(q,r)=(q(z,\zeta),r(z,\zeta))$ satisfy
\begin{align}
    \frac1r+\psi_p\zeta&=\frac1q+\frac{\mu_1^2\psi_pe^{-2t}}{L_1^\infty}+\psi_p\mu_1^2\Delta_t,\\
    \psi_pz&=\frac{1-\psi_p}{q}-\frac{r}{q^2}+\frac{\psi_p\mu_*^2\kappa}{L_1^\infty}+\psi_p\mu_*^2(1-\kappa).
\end{align}
The training loss and the empirical test loss coincide, $\Ltrain=\Linfty$, with $\Linfty=1-\mu_1^2\psi_p\Delta_tr$, and
\begin{equation}
\Ltest(m=\infty)=1-2\mu_1^2\Delta_t\psi_pr+\mu_1^2\psi_p\Delta_t\bigl(\mu_1^2\partial_\zeta r+\mu_*^2\partial_z r\bigr),
\end{equation}
with $r=r(-\lambda,0)$ and the derivatives evaluated at $(z,\zeta)=(-\lambda,0)$.
\end{proposition}
We recover the results of \citet{george_2025, bonnaire2025diffusionmodelsdontmemorize}\footnote{In Theorem 3.2 of \citet{george_2025}, they denote $r$ by $\mathcal{K}$.}.
\begin{proof}
By Proposition~\ref{prop:app-saddle} (isotropic case), at finite $m$ the order parameters obey
\begin{align}
  \frac1r+\psi_p\zeta &= \frac1q+\frac{\mu_1^2\psi_p\!\left(e^{-2t}+\frac{\Delta_t}{m}\right)}{L_1}+\frac{(m-1)\mu_1^2\psi_p\frac{\Delta_t}{m}}{L_2},\\
  \psi_p z &= \frac{1-\psi_p}{q}-\frac{r}{q^2}+\frac{\psi_p\mu_*^2\frac{1+(m-1)\kappa}{m}}{L_1}+\frac{(m-1)\psi_p\mu_*^2\frac{1-\kappa}{m}}{L_2},
\end{align}
with $L_1=1+\frac{\psi_p}{\psi_nm}[\mu_1^2 r(e^{-2t}m+\Delta_t)+\mu_*^2 q(1+(m-1)\kappa)]$ and $L_2=1+\frac{\psi_p}{\psi_nm}[\mu_1^2 r\Delta_t+\mu_*^2 q(1-\kappa)]$. In $L_1$ the bracket grows linearly in $m$, since $\frac{1}{m}[\mu_1^2 r(e^{-2t}m+\Delta_t)+\mu_*^2 q(1+(m-1)\kappa)]\to\mu_1^2 e^{-2t}r+\mu_*^2\kappa q$; hence $L_1\to L_1^\infty=1+\frac{\psi_p}{\psi_n}(\mu_1^2 e^{-2t}r+\mu_*^2\kappa q)$, whereas the $L_2$-bracket stays $O(1)$ so $L_2\to1$. Using in addition $e^{-2t}+\frac{\Delta_t}{m}\to e^{-2t}$, $\frac{m-1}{m}\to1$, $\frac{1+(m-1)\kappa}{m}\to\kappa$ and $\frac{(m-1)(1-\kappa)}{m}\to1-\kappa$, the four fractions tend to $\frac{\mu_1^2\psi_p e^{-2t}}{L_1^\infty}$, $\psi_p\mu_1^2\Delta_t$, $\frac{\psi_p\mu_*^2\kappa}{L_1^\infty}$ and $\psi_p\mu_*^2(1-\kappa)$, giving the stated equations.  
For the losses, $\tau=\mathcal{T}/m\to1$ with $\tau-1=O(1/m)$, so $\partial_z\tau,\partial_\zeta\tau\to0$ and $f_1(q,r,0)\to\mu_1^2\psi_p\Delta_tr$, whence
\begin{equation}
    T_1\to\mu_1^2\psi_p\Delta_tr,
    \quad
    T_2\to\mu_1\sqrt{\Delta_t}\psi_pr,
    \quad
    T_4\to\mu_1^2\psi_p\Delta_t\,\partial_\zeta r,
    \quad
    T_5\to\mu_1^2\psi_p\Delta_t\,\partial_z r.
\end{equation}
At $m=\infty$ the empirical noise average in $\Ltrain$ coincides with the exact one, so $\Ltrain=\Linfty$ by definition and the common value follows from $\Ltrain=1-T_1$. Substituting into the loss decomposition gives the stated expressions, the derivatives $\partial_\zeta r,\partial_z r$ being obtained by implicit differentiation of the two self-consistent equations.
\end{proof}

\subsection{Extension to sub-Gaussian data}
\label{app:subgaussian}

Throughout the main text we assumed $\vx^\nu\sim\mathcal{N}(0,\vI_d)$. We show here that the results on the losses extend to any zero-mean distribution with covariance $\vSigma$ and sub-Gaussian tails, provided that $\Tr(\vSigma)/d\to\sigma_x^2=O(1)$ as $d\to\infty$, and that the activation satisfies $\mu_0:=\mathbb{E}_{z\sim\mathcal{N}(0,1)}[\sigma(\sqrt{e^{-2t}\sigma_x^2+\Delta_t}\,z)]=0$. We emphasize that the expressions for the bias-variance decomposition relies on an isotropy argument and thus cannot be extended to the anisotropy case. The argument has two steps.

\paragraph*{The preactivation is asymptotically Gaussian.}
The preactivation of neuron $\alpha$ on sample $\nu$ at noise level $t$ is $h_\alpha^{\nu\mu} = e^{-t}\vW_\alpha\cdot\vx^\nu/\sqrt{d}+\sqrt{\Delta_t}\vW_\alpha\cdot\vxi^{\nu\mu}/\sqrt{d}$. The noise term is exactly Gaussian. For the data term $g_\alpha^\nu=e^{-t}\vW_\alpha\cdot\vx^\nu/\sqrt{d}$, conditioning on $\vW$, the CLT for linear forms of sub-Gaussian vectors yields that $\vg^\nu=(g_\alpha^\nu)_\alpha$ is asymptotically Gaussian with covariance $e^{-2t}\vW\vSigma\vW^T/d$. This covariance matrix concentrates to a diagonal one: its diagonal entries converge to $e^{-2t}\Tr(\vSigma)/d=:e^{-2t}\sigma_x^2$ while its off-diagonal entries vanish, since $\vW_\alpha$ and $\vW_\beta$ are independent for $\alpha\neq\beta$. Setting $\Gamma_t=e^{-2t}\sigma_x^2+\Delta_t$, the full preactivation is therefore asymptotically $h_\alpha^{\nu\mu}\sim\mathcal{N}(0,\Gamma_t)$ with independent components across neurons, and the GEP can be applied.

\paragraph*{Redefinition of the GEP coefficients.}
Since the marginal distribution of $h_\alpha^{\nu\mu}$ is $\mathcal{N}(0,\Gamma_t)$, the GEP coefficients must be recomputed accordingly. Writing $h=\sqrt{\Gamma_t}\,z$ with $z\sim\mathcal{N}(0,1)$ and the data preactivation as $e^{-t}\sigma_x u$ with $u\sim\mathcal{N}(0,1)$,

\begin{align}
\mu_1&=\mathbb{E}\big[\sigma'(\sqrt{\Gamma_t}\,z)\big],\\
\mu_*^2&=\mathbb{E}\big[\sigma(\sqrt{\Gamma_t}\,z)^2\big]-\mu_1^2\Gamma_t,\\
\kappa&=\frac{1}{\mu_*^2}\,\mathbb{E}_u\Big[\big(\mathbb{E}_v\big[\sigma(e^{-t}\sigma_x u+\sqrt{\Delta_t}\,v)\big]-\mu_1 e^{-t}\sigma_x u\big)^2\Big],
\end{align}
where $z,u,v\sim\mathcal{N}(0,1)$ are independent.

For $\sigma_x=1$ one recovers $\Gamma_t=1$ and the original definitions of the main text. The GEP itself---the equivalence of resolvent traces between $\vF=\sigma(\vW\vY/\sqrt{d})$ and its linear Gaussian surrogate---extends to sub-Gaussian data by universality \citep{Peche2019,goldt_2021,hu2023}. The data covariance $\vSigma$ enters the self-consistent equations only through the spectral measure $\rho_{\vSigma}$ in Proposition~\ref{prop:app-saddle}.

\paragraph*{Conclusion.} All results of the main text of the three losses extend to zero-mean sub-Gaussian data with covariance $\vSigma$ and $\Tr(\vSigma)/d\to\sigma_x^2=O(1)$, with $\mu_1,\mu_*,\kappa$ redefined via $\Gamma_t=e^{-2t}\sigma_x^2+\Delta_t$, and $\rho_{\vSigma}$ replacing $\delta(\lambda-1)$ in the resolvent integrals.

\subsection{Additional Results and Figures}
\label{app:additional_analytical_figures}

In this section we present additional results and figures on the analytical part.

\subsubsection{Extended discussion on the effect of regularization.}
\label{App:analytical_discussion_t}

In the RFNN model, the optimal training strategy depends critically on the diffusion time $t$. At large $t$, we recover the benign overfitting of the classical regression setting \citep{mei2020} (inset of the right panel of Fig.~\ref{fig:effect_regularization_appendix}): for every fixed $\psi_p$, the test loss is a strictly increasing function of $\lambda$, so zero regularization is always optimal. The global minimum over $(\psi_p,\lambda)$ is attained at $\psi_p\to\infty$ and $\lambda=0$: overparameterization is unconditionally beneficial, and any regularization strictly hurts. On the other hand, at small $t$, malign overfitting fundamentally alters this picture. Without regularization ($\lambda\to0^+$), the optimal model size remains $\psi_p\to\infty$ for $m=1$ (inset of the left panel of Fig.~\ref{fig:effect_regularization_appendix}), but shifts to the underparameterized regime ($\psi_p^*<\psi_n$) for $m>1$ (left panel of Fig.~\ref{fig:effect_regularization_appendix}): overparameterizing an unregularized model is detrimental. However, jointly optimizing over $(\psi_p,\lambda)$ restores the benefit of large models: the global optimum is always $\psi_p\to\infty$ with $\lambda/\psi_p=O(1)$ (right panel of Fig.~\ref{fig:effect_regularization_appendix}), so a large well-regularized model always outperforms a small unregularized one.

\begin{figure}
    \centering
    
    \includegraphics[width=\linewidth]{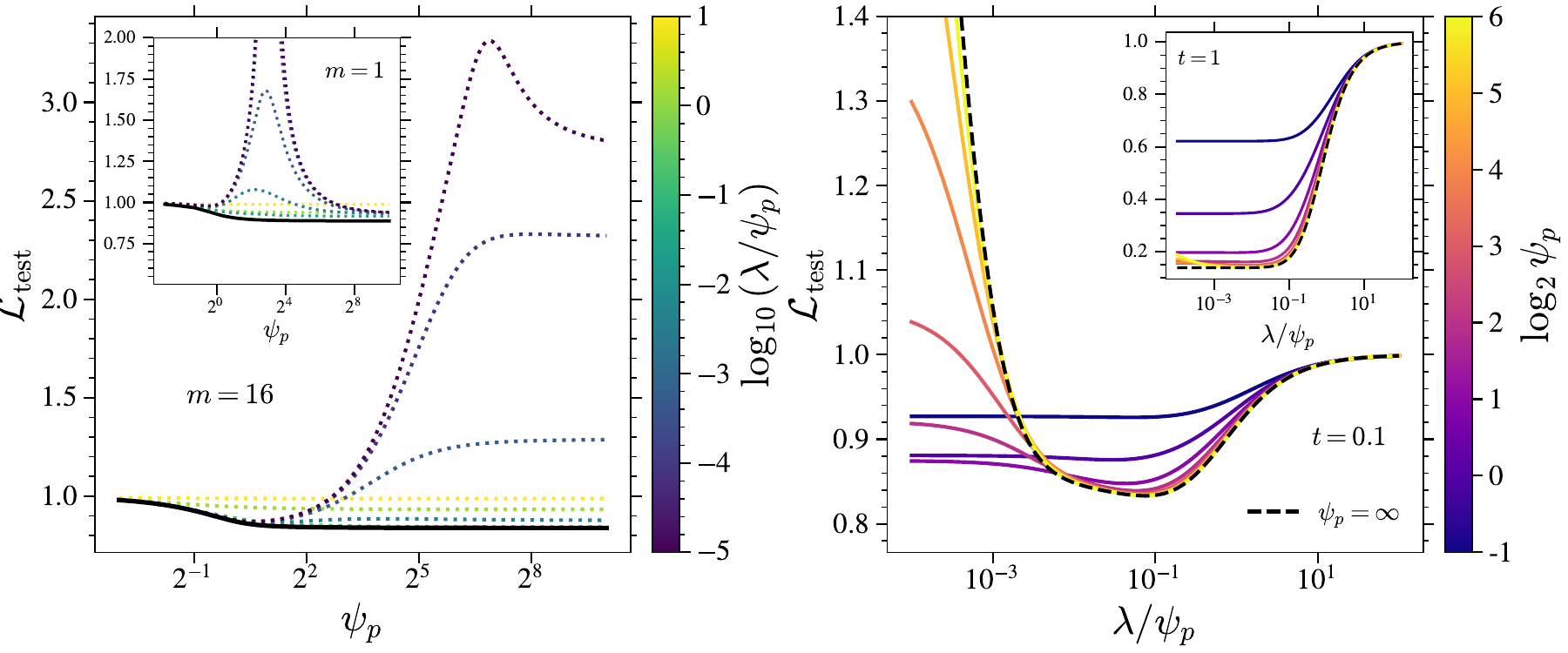}

    \caption{\textbf{Benefits of regularization on generalization.} \emph{(Left)} Test loss $\Ltest$ of the RFNN as a function of $\psi_p=p/d$ for several values of $\lambda_p=\psi_p \tilde{\lambda}$ with $\psi_n=n/d=8,\ \sigma=\tanh,\ t=0.1$ and $m=16$; inset: same at $m=1$. The full black line is the lower envelope. \emph{(Right)} Test loss as a function of $\lambda/\psi_p$ for several values of $\psi_p$ at $m=16,\ t=0.1$, together with the $\psi_p=\infty$ limit (black dashed); inset: same at $t=1.0$.}
    \label{fig:effect_regularization_appendix}
\end{figure}

\subsubsection{Additional results on the bias and variance decomposition}

In Fig.~\ref{fig:Bias_variance_m_1} we show the bias and the variance as functions of the model size $\psi_p$, comparing the analytical curves to numerical experiments at finite $d$, for $m=1$ (\emph{left}) and $m=4$ (\emph{middle}). At $m=1$ the bias decreases monotonically with $\psi_p$, whereas at $m=4$ it passes through a minimum and increases again, before saturating at a finite value. In Fig.~\ref{fig:Bias_variance_m_1} (\emph{right}) we plot the $\psi_p\to\infty$ limit of both quantities as a function of $\psi_n$: the bias saturates at a value that remains $O_{\psi_n}(1)$ for $m>1$, while it decays as a power law in $\psi_n$ for $m=1$; the variance decays as a power law in $\psi_n$ for both $m=1$ and $m>1$.

\begin{figure}
 \centering
    \includegraphics[width=0.3\linewidth]{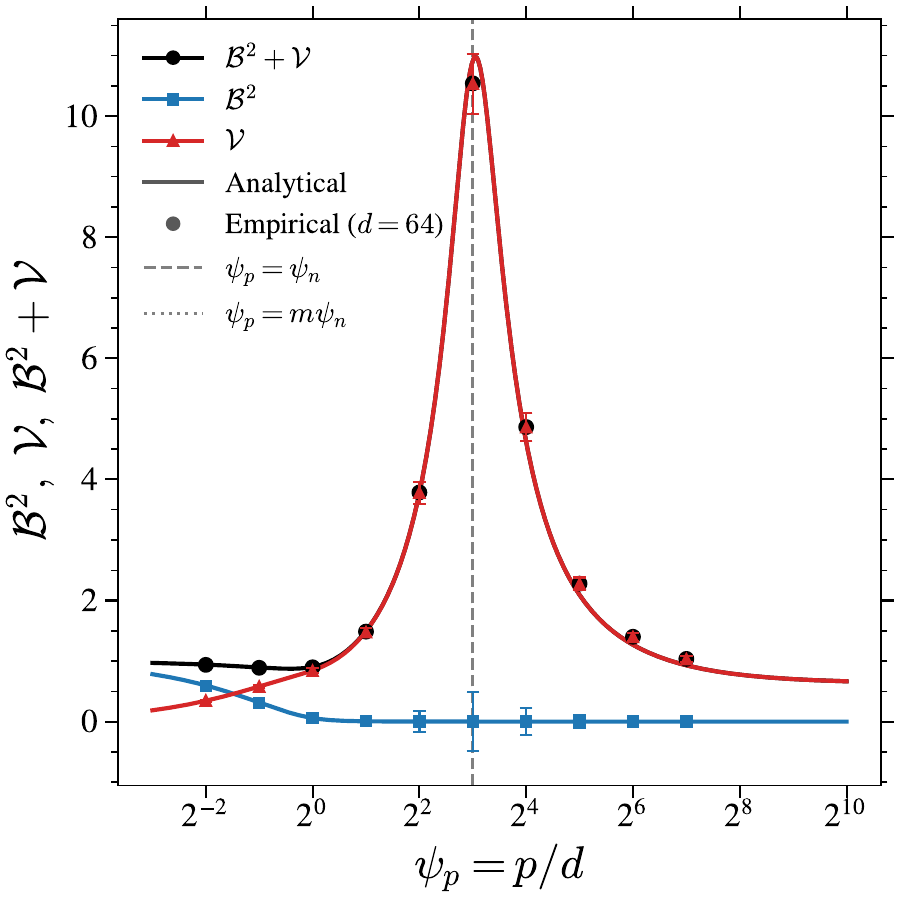}
  \includegraphics[width=0.3\linewidth]{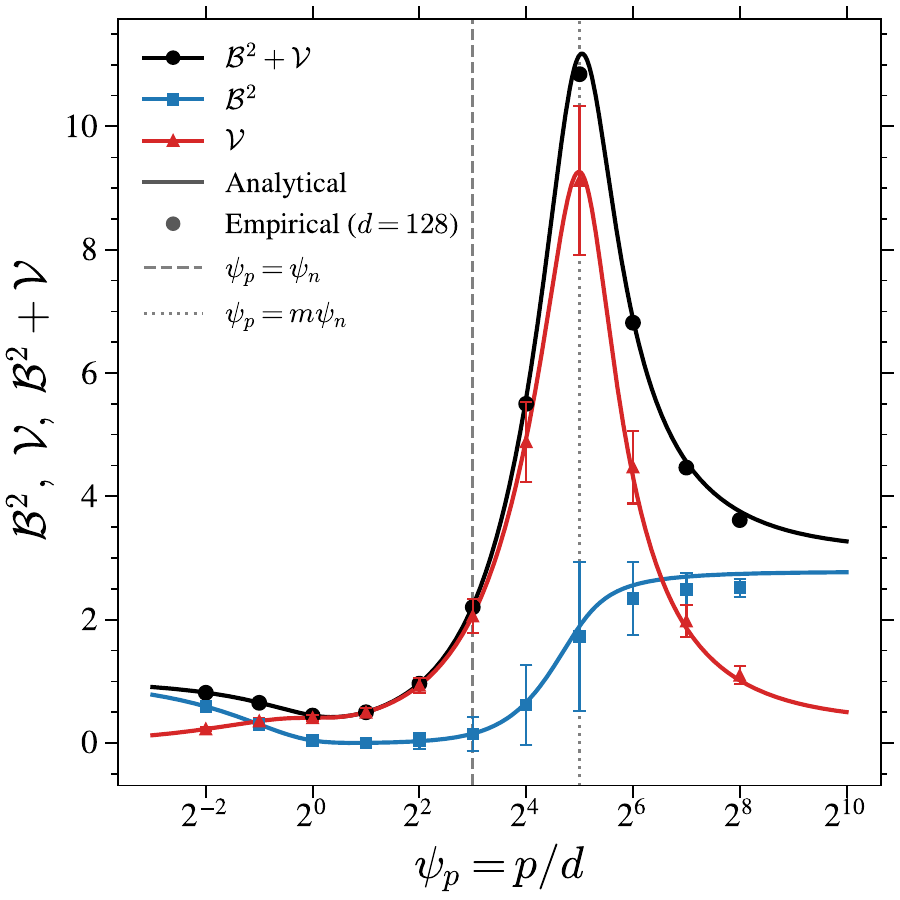}
\includegraphics[width=0.3\linewidth]{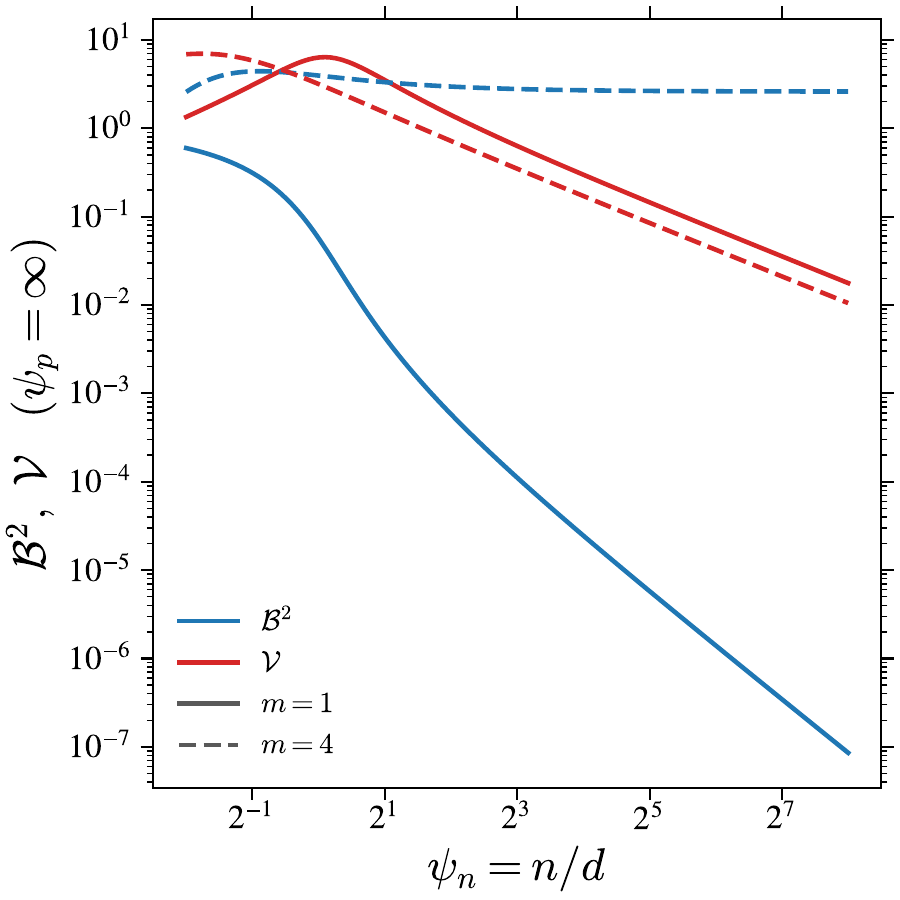}
  
    \caption{(\textit{Left}) Bias and variance for the RFNN model as a function of $\psi_p=p/d$ for $\psi_n=8,\ t=0.1,\ m=1,\ \lambda=10^{-3}$ and finite dimension $d=64$. Solid curves are the analytical predictions; markers are numerical estimates averaged over 20 independent runs. Error bars correspond to $\pm3$SE. (\textit{Middle}) Bias and variance for the RFNN model as a function of $\psi_p=p/d$ for $\psi_n=8,\ t=0.1,\ m=4,\ \lambda=10^{-3}$ and finite dimension $d=128$. Solid curves are the analytical predictions obtained from Theorem~\ref{thm:main} through the bias--variance expressions of Appendix~\ref{app:bias_variance}; markers are numerical estimates from 10 independent runs. Error bars correspond to $\pm3$SE. \emph{(Right)} Bias (full line) and variance (dotted line) as a function of the sample complexity $\psi_n$ for $m=1$ and $m=4$ obtained by solving the analytical equations at  $\psi_p=\infty$ for $t=0.1$.}
    \label{fig:Bias_variance_m_1}
\end{figure}

\subsubsection{Additional Figures}
In this section, we present some additional figures on the RFNN model. The activation function taken for the figures is always $\sigma=\tanh$.

\paragraph*{Colormaps of the losses.} Fig.\ref{fig:colormap_analytical} shows colormaps of the three losses in the $(\psi_n,\psi_p)$ space. We observe that the line $\psi_p=m\psi_n$ delimits a region with low $\Linfty$ and $\Ltrain$ but large $\Ltest$. We also observe that the test loss presents two peaks: one at $\psi_p=m\psi_n$ as discussed in the main text as well as another peak at $m\psi_n=1$ which is related to the two peaks found in \citet{d_Ascoli_triple_2021} that were located at $\psi_n=1$ for supervised learning.

\paragraph*{Scaling of the generalization and noise gaps.} Fig.\ref{fig:generalization_gap_analytical} plots the generalization gap $\Ltest-\Ltrain$ and the noise gap $\Linfty-\Ltrain$ as a function of $\psi_p$ and shows that they are actually functions of $\psi_p/\psi_n$ for a fixed $m$.

\begin{figure}[h]
    \centering
    \includegraphics[width=1.\linewidth]{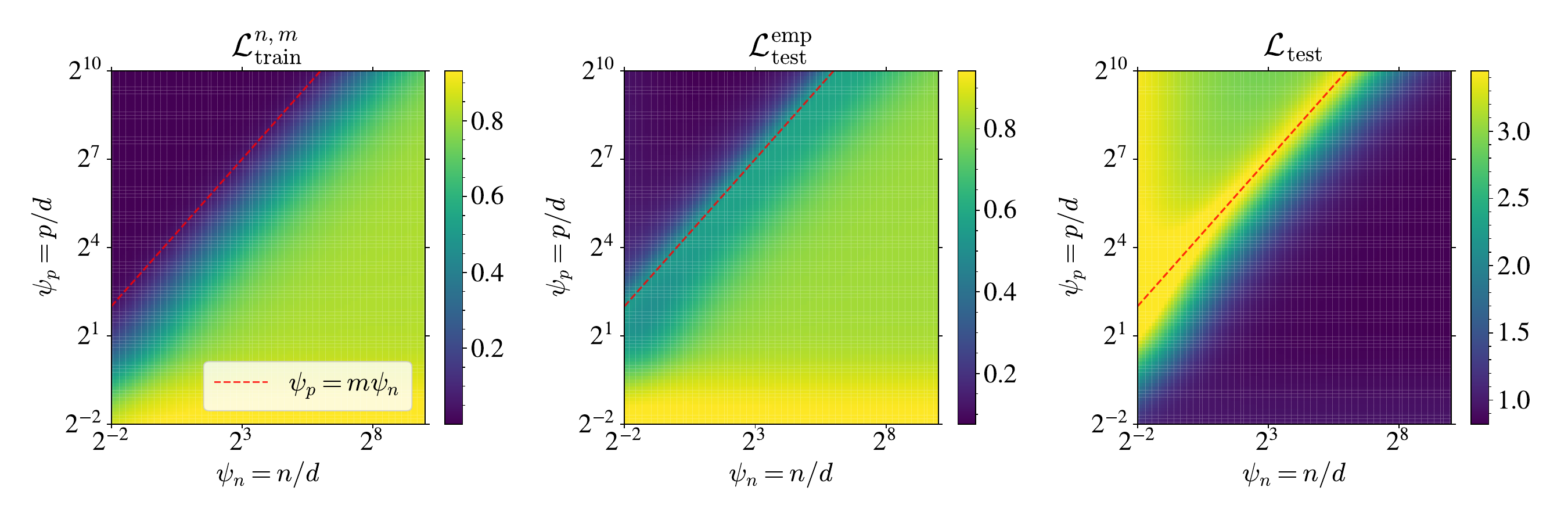}
    \caption{Colormap of the analytical solutions of the three losses for $m=16,\ t=0.1$ as a function of $\psi_n$ and $\psi_p$.}
    \label{fig:colormap_analytical}
\end{figure}

\begin{figure}[h]
    \centering
    \includegraphics[width=0.8\linewidth]{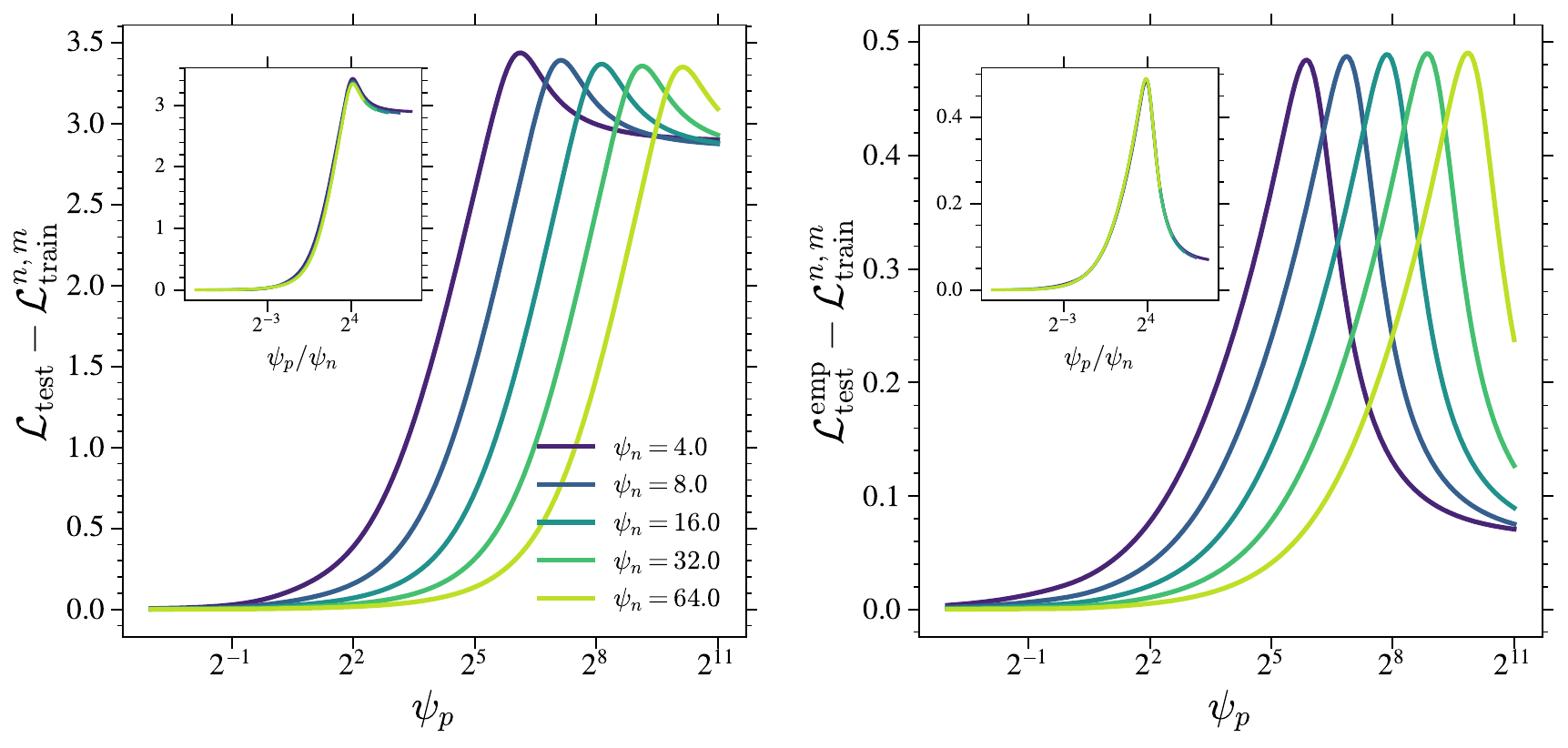}
    \caption{Evolution of the \emph{(Left)} generalization and \emph{(Right)} noise gaps as a function of $\psi_p$ for the RFNN model at $t=0.1,\ \lambda=10^{-3},\ \sigma=\tanh$ and $m=16$. We observe that both gaps scale as $\psi_p/\psi_n$.}
    \label{fig:generalization_gap_analytical}
\end{figure}

\section{LLM Usage}
\label{app:llm_usage}

We describe here the role played by large
language models (LLMs) in the preparation of this work. 
  
\paragraph*{Code.} LLMs were used to write and debug parts of the code, both for the
numerical experiments and for the scripts producing the figures of this paper.

\paragraph*{Writing and presentation.} LLMs were used for copy-editing throughout:
correcting grammar, finding typographical errors, and uniformizing notation, style and
cross-referencing across sections.

\paragraph*{Checking the derivations.} We used LLMs to re-check the algebra of our
analytical results and to audit the manuscript for internal inconsistencies. In particular, an LLM located a typo in an intermediate step of
our computation of the trace $T_1(z,\zeta,0)$ (Appendix~\ref{app:comp-traces}), which had been blocking our progress on the remainder of the derivation.

\paragraph*{Responsibility.} Every LLM-assisted derivation was checked independently and thoroughly by
the authors, and every suggested edit was reviewed before inclusion. The authors take
full responsibility for the content, the originality and the scientific integrity of this
work, including any remaining errors.

%% file: Bibliography_arxiv.bib
@inproceedings{neyshabur2015search,
  title={In Search of the Real Inductive Bias: On the Role of Implicit Regularization in Deep Learning},
  author={Neyshabur, Behnam and Tomioka, Ryota and Srebro, Nathan},
  booktitle={International Conference on Learning Representations (ICLR) Workshop},
  year={2015},
  url={https://arxiv.org/abs/1412.6614}
}

@book{wasserman2004all,
  title={All of statistics: a concise course in statistical inference},
  author={Wasserman, Larry},
  volume={26},
  year={2004},
  publisher={Springer}
}

@article{geiger2020scaling,
  title={Scaling description of generalization with number of parameters in deep learning},
  author={Geiger, Mario and Jacot, Arthur and Spigler, Stefano and Gabriel, Franck and Sagun, Levent and d'Ascoli, Stéphane and Biroli, Giulio and Hongler, Clément and Wyart, Matthieu},
  journal={Journal of Statistical Mechanics: Theory and Experiment},
  volume={2020},
  number={2},
  pages={023401},
  year={2020},
  publisher={IOP Publishing},
  url={https://arxiv.org/abs/1909.11572}
}

@inproceedings{lawrence1997lessons,
  title={Lessons in Neural Network Training: Overfitting May Be Harder than Expected},
  author={Lawrence, Steve and Giles, C. Lee and Tsoi, Ah Chung},
  booktitle={Proceedings of the Fourteenth National Conference on Artificial Intelligence (AAAI-97)},
  pages={540--545},
  year={1997},
  organization={AAAI Press}
}

@article{Biroli_2024,
  title={Dynamical regimes of diffusion models},
  author={Giulio Biroli and Tony Bonnaire and Valentin de Bortoli and Marc M{\'e}zard},
  journal={Nature Communications},
  volume={15},
  number={9957},
  year={2024},
  publisher={Nature Publishing Group},
  url={https://www.nature.com/articles/s41467-024-9957-y},
  doi={10.1038/s41467-024-9957-y}
}

@inproceedings{
george_2025,
title={Denoising Score Matching with Random Features: Insights on Diffusion Models From Precise Learning Curves},
author={Anand Jerry George and Rodrigo Veiga and Nicolas Macris},
booktitle={The 29th International Conference on Artificial Intelligence and Statistics},
year={2026},
url={https://openreview.net/forum?id=ZnplHm2uRt}
}

@InProceedings{Gerace_2020,
  title = 	 {Generalisation error in learning with random features and the hidden manifold model},
  author =       {Gerace, Federica and Loureiro, Bruno and Krzakala, Florent and Mezard, Marc and Zdeborova, Lenka},
  booktitle = 	 {Proceedings of the 37th International Conference on Machine Learning},
  pages = 	 {3452--3462},
  year = 	 {2020},
  editor = 	 {III, Hal Daumé and Singh, Aarti},
  volume = 	 {119},
  series = 	 {Proceedings of Machine Learning Research},
  month = 	 {13--18 Jul},
  publisher =    {PMLR},
  url = 	 {https://proceedings.mlr.press/v119/gerace20a.html}
}

@article{goldt2020modeling,
  title={Modeling the influence of data structure on learning in neural networks: The hidden manifold model},
  author={Goldt, Sebastian and M{\'e}zard, Marc and Krzakala, Florent and Zdeborov{\'a}, Lenka},
  journal={Physical Review X},
  volume={10},
  number={4},
  pages={041044},
  year={2020},
  publisher={APS}
}

@inproceedings{goldt_2021,
  author       = {Sebastian Goldt and
                  Bruno Loureiro and
                  Galen Reeves and
                  Florent Krzakala and
                  Marc M{\'{e}}zard and
                  Lenka Zdeborov{\'{a}}},
  editor       = {Joan Bruna and
                  Jan S. Hesthaven and
                  Lenka Zdeborov{\'{a}}},
  title        = {The Gaussian equivalence of generative models for learning with shallow
                  neural networks},
  booktitle    = {Mathematical and Scientific Machine Learning, 16-19 August 2021, Virtual
                  Conference / Lausanne, Switzerland},
  series       = {Proceedings of Machine Learning Research},
  volume       = {145},
  pages        = {426--471},
  publisher    = {{PMLR}},
  year         = {2021},
  url          = {https://proceedings.mlr.press/v145/goldt22a.html},
}

@InProceedings{Ascoli_2020,
  title = 	 {Double Trouble in Double Descent: Bias and Variance(s) in the Lazy Regime},
  author =       {D'Ascoli, St{\'e}phane and Refinetti, Maria and Biroli, Giulio and Krzakala, Florent},
  booktitle = 	 {Proceedings of the 37th International Conference on Machine Learning},
  pages = 	 {2280--2290},
  year = 	 {2020},
  editor = 	 {III, Hal Daumé and Singh, Aarti},
  volume = 	 {119},
  series = 	 {Proceedings of Machine Learning Research},
  month = 	 {13--18 Jul},
  publisher =    {PMLR},
  url = 	 {https://proceedings.mlr.press/v119/d-ascoli20a.html}
}

@misc{baptista_2025,
      title={Memorization and Regularization in Generative Diffusion Models}, 
      author={Ricardo Baptista and Agnimitra Dasgupta and Nikola B. Kovachki and Assad Oberai and Andrew M. Stuart},
      year={2025},
      eprint={2501.15785},
      archivePrefix={arXiv},
      primaryClass={cs.LG},
      url={https://arxiv.org/abs/2501.15785}, 
}

@inproceedings{Rahimi_2007,
 author = {Rahimi, Ali and Recht, Benjamin},
 booktitle = {Advances in Neural Information Processing Systems},
 editor = {J. Platt and D. Koller and Y. Singer and S. Roweis},
 pages = {},
 publisher = {Curran Associates, Inc.},
 title = {Random Features for Large-Scale Kernel Machines},
 url = {https://proceedings.neurips.cc/paper_files/paper/2007/file/013a006f03dbc5392effeb8f18fda755-Paper.pdf},
 volume = {20},
 year = {2007}
}

@ARTICLE{Vincent_2011,
  author={Vincent, Pascal},
  journal={Neural Computation}, 
  title={A Connection Between Score Matching and Denoising Autoencoders}, 
  year={2011},
  volume={23},
  number={7},
  pages={1661-1674},
  doi={10.1162/NECO_a_00142}}

@inproceedings{kadkhodaie_2024,
title={Generalization in diffusion models arises from geometry-adaptive harmonic representations},
author={Zahra Kadkhodaie and Florentin Guth and Eero P Simoncelli and St{\'e}phane Mallat},
booktitle={The Twelfth International Conference on Learning Representations},
year={2024},
url={https://openreview.net/forum?id=ANvmVS2Yr0}
}

@InProceedings{sohl-dickstein_15,
  title = 	 {Deep Unsupervised Learning using Nonequilibrium Thermodynamics},
  author = 	 {Sohl-Dickstein, Jascha and Weiss, Eric and Maheswaranathan, Niru and Ganguli, Surya},
  booktitle = 	 {Proceedings of the 32nd International Conference on Machine Learning},
  pages = 	 {2256--2265},
  year = 	 {2015},
  editor = 	 {Bach, Francis and Blei, David},
  volume = 	 {37},
  series = 	 {Proceedings of Machine Learning Research},
  address = 	 {Lille, France},
  month = 	 {07--09 Jul},
  publisher =    {PMLR},
  url = 	 {https://proceedings.mlr.press/v37/sohl-dickstein15.html},
}

@article{hyvarinen_05,
  author  = {Aapo Hyv{{\"a}}rinen},
  title   = {Estimation of Non-Normalized Statistical Models by Score Matching},
  journal = {Journal of Machine Learning Research},
  year    = {2005},
  volume  = {6},
  number  = {24},
  pages   = {695--709},
  url     = {http://jmlr.org/papers/v6/hyvarinen05a.html}
}

@inproceedings{
song2021b,
title={Score-Based Generative Modeling through Stochastic Differential Equations},
author={Yang Song and Jascha Sohl-Dickstein and Diederik P Kingma and Abhishek Kumar and Stefano Ermon and Ben Poole},
booktitle={International Conference on Learning Representations},
year={2021},
url={https://openreview.net/forum?id=PxTIG12RRHS}
}

@inproceedings{song2019,
 author = {Song, Yang and Ermon, Stefano},
 booktitle = {Advances in Neural Information Processing Systems},
 editor = {H. Wallach and H. Larochelle and A. Beygelzimer and F. d\textquotesingle Alch\'{e}-Buc and E. Fox and R. Garnett},
 pages = {},
 publisher = {Curran Associates, Inc.},
 title = {Generative Modeling by Estimating Gradients of the Data Distribution},
 url = {https://proceedings.neurips.cc/paper_files/paper/2019/file/3001ef257407d5a371a96dcd947c7d93-Paper.pdf},
 volume = {32},
 year = {2019}
}

@inproceedings{ho2020,
 author = {Ho, Jonathan and Jain, Ajay and Abbeel, Pieter},
 booktitle = {Advances in Neural Information Processing Systems},
 editor = {H. Larochelle and M. Ranzato and R. Hadsell and M.F. Balcan and H. Lin},
 pages = {6840--6851},
 publisher = {Curran Associates, Inc.},
 title = {Denoising Diffusion Probabilistic Models},
 url = {https://proceedings.neurips.cc/paper_files/paper/2020/file/4c5bcfec8584af0d967f1ab10179ca4b-Paper.pdf},
 volume = {33},
 year = {2020}
}

@article{haussmann_1986,
  title={Time reversal of diffusions},
  author={Haussmann, U.G. and Pardoux, E.},
  journal={The Annals of Probability},
  volume={14},
  number={4},
  pages={1188--1205},
  year={1986},
  publisher={Institute of Mathematical Statistics},
  doi={10.1214/aop/1176992362}
}

@inproceedings{
cui_2024,
title={Analysis of Learning a Flow-based Generative Model from Limited Sample Complexity},
author={Hugo Cui and Florent Krzakala and Eric Vanden-Eijnden and Lenka Zdeborova},
booktitle={The Twelfth International Conference on Learning Representations},
year={2024},
url={https://openreview.net/forum?id=ndCJeysCPe}
}

@inproceedings{cui_2025,
 author = {Cui, Hugo and Pehlevan, Cengiz and Lu, Yue},
 booktitle = {Advances in Neural Information Processing Systems},
 doi = {10.52202/085713-0187},
 editor = {D. Belgrave and C. Zhang and H. Lin and R. Pascanu and P. Koniusz and M. Ghassemi and N. Chen},
 pages = {5253--5296},
 publisher = {Curran Associates, Inc.},
 title = {A solvable model of learning generative diffusion: theory and insights},
 url = {https://proceedings.neurips.cc/paper_files/paper/2025/file/082d3d795520c43214da5123e56a3a34-Paper-Conference.pdf},
 volume = {38, Main Conference},
 year = {2025}
}

@misc{li_2024_good_score,
      title={A Good Score Does not Lead to A Good Generative Model}, 
      author={Sixu Li and Shi Chen and Qin Li},
      year={2024},
      eprint={2401.04856},
      archivePrefix={arXiv},
      primaryClass={cs.LG},
      url={https://arxiv.org/abs/2401.04856}, 
}

@inproceedings{somepalli_2022,
    title={Diffusion Art or Digital Forgery? Investigating Data Replication in Diffusion Models},
    author={Somepalli, Gowthami and Singla, Vasu and Goldblum, Micah and Geiping, Jonas and Goldstein, Tom},
    booktitle={Proceedings of the IEEE/CVF Conference on Computer Vision and Pattern Recognition},
    year={2023}
  }

@article{somepalli_2023,
title={Understanding and mitigating copying in diffusion models},
author={Somepalli, Gowthami and Singla, Vasu and Goldblum, Micah and Geiping, Jonas and Goldstein, Tom},
journal={Advances in Neural Information Processing Systems},
volume={36},
pages={47783--47803},
year={2023}
}

@inproceedings{Carlini_2023, author = {Carlini, Nicholas and Hayes, Jamie and Nasr, Milad and Jagielski, Matthew and Sehwag, Vikash and Tram\`{e}r, Florian and Balle, Borja and Ippolito, Daphne and Wallace, Eric}, title = {Extracting training data from diffusion models}, year = {2023}, isbn = {978-1-939133-37-3}, publisher = {USENIX Association}, address = {USA}, booktitle = {Proceedings of the 32nd USENIX Conference on Security Symposium}, articleno = {294}, numpages = {18}, location = {Anaheim, CA, USA}, series = {SEC '23} }

@misc{ventura2025,
      title={Manifolds, Random Matrices and Spectral Gaps: The geometric phases of generative diffusion}, 
      author={Enrico Ventura and Beatrice Achilli and Gianluigi Silvestri and Carlo Lucibello and Luca Ambrogioni},
      year={2025},
      eprint={2410.05898},
      archivePrefix={arXiv},
      primaryClass={stat.ML},
      url={https://arxiv.org/abs/2410.05898}, 
}

@misc{achilli2024,
      title={Losing dimensions: Geometric memorization in generative diffusion}, 
      author={Beatrice Achilli and Enrico Ventura and Gianluigi Silvestri and Bao Pham and Gabriel Raya and Dmitry Krotov and Carlo Lucibello and Luca Ambrogioni},
      year={2024},
      eprint={2410.08727},
      archivePrefix={arXiv},
      primaryClass={stat.ML},
      url={https://arxiv.org/abs/2410.08727}, 
}

@article{mei2020,
  title={The Generalization Error of Random Features Regression: Precise Asymptotics and the Double Descent Curve},
  author={Song Mei and Andrea Montanari},
  journal={Communications on Pure and Applied Mathematics},
  year={2019},
  volume={75},
  url={https://api.semanticscholar.org/CorpusID:199668852}
}

@inproceedings{
song2022DDIM,
title={Denoising Diffusion Implicit Models},
author={Jiaming Song and Chenlin Meng and Stefano Ermon},
booktitle={International Conference on Learning Representations},
year={2021},
url={https://openreview.net/forum?id=St1giarCHLP}
}

@inproceedings{yoon2023diffusion,
title={Diffusion Probabilistic Models Generalize when They Fail to Memorize},
author={TaeHo Yoon and Joo Young Choi and Sehyun Kwon and Ernest K. Ryu},
booktitle={ICML 2023 Workshop on Structured Probabilistic Inference {\&} Generative Modeling},
year={2023},
url={https://openreview.net/forum?id=shciCbSk9h}
}

@article{
gu2023memorization,
title={On Memorization in Diffusion Models},
author={Xiangming Gu and Chao Du and Tianyu Pang and Chongxuan Li and Min Lin and Ye Wang},
journal={Transactions on Machine Learning Research},
issn={2835-8856},
year={2025},
url={https://openreview.net/forum?id=D3DBqvSDbj},
note={}
}

@inproceedings{heusel2017gans,
author = {Heusel, Martin and Ramsauer, Hubert and Unterthiner, Thomas and Nessler, Bernhard and Hochreiter, Sepp},
title = {GANs trained by a two time-scale update rule converge to a local nash equilibrium},
year = {2017},
isbn = {9781510860964},
publisher = {Curran Associates Inc.},
address = {Red Hook, NY, USA},
booktitle = {Proceedings of the 31st International Conference on Neural Information Processing Systems},
pages = {6629–6640},
numpages = {12},
location = {Long Beach, California, USA},
series = {NIPS'17}
}

@article{kibble1945,
  author    = {W. F. Kibble},
  title     = {An extension of a theorem of Mehler’s on Hermite polynomials},
  journal   = {Mathematical Proceedings of the Cambridge Philosophical Society},
  volume    = {41},
  number    = {1},
  pages     = {12--15},
  year      = {1945},
  month     = jun,
  issn      = {0305-0041, 1469-8064},
  doi       = {10.1017/S0305004100022313}
}

@InProceedings{Ronneberger2015,
author="Ronneberger, Olaf
and Fischer, Philipp
and Brox, Thomas",
editor="Navab, Nassir
and Hornegger, Joachim
and Wells, William M.
and Frangi, Alejandro F.",
title="U-Net: Convolutional Networks for Biomedical Image Segmentation",
booktitle="Medical Image Computing and Computer-Assisted Intervention -- MICCAI 2015",
year="2015",
publisher="Springer International Publishing",
address="Cham",
pages="234--241"
}

@inproceedings{CelebA,
  title = {Deep Learning Face Attributes in the Wild},
  author = {Liu, Ziwei and Luo, Ping and Wang, Xiaogang and Tang, Xiaoou},
  booktitle = {Proceedings of International Conference on Computer Vision (ICCV)},
  month = {December},
  year = {2015} 
}

@article{Peche2019,
  author       = {S. Péché},
  title        = {A note on the Pennington-Worah distribution},
  journal      = {Electronic Communications in Probability},
  volume       = {24},
  pages        = {1--7},
  year         = {2019},
  publisher    = {Institute of Mathematical Statistics},
  doi          = {10.1214/19-ECP255},
  url          = {https://doi.org/10.1214/19-ECP255}
}

@inproceedings{karras2022elucidatingdesignspacediffusionbased,
author = {Karras, Tero and Aittala, Miika and Laine, Samuli and Aila, Timo},
title = {Elucidating the design space of diffusion-based generative models},
year = {2022},
isbn = {9781713871088},
publisher = {Curran Associates Inc.},
address = {Red Hook, NY, USA},
booktitle = {Proceedings of the 36th International Conference on Neural Information Processing Systems},
articleno = {1926},
numpages = {13},
location = {New Orleans, LA, USA},
series = {NIPS '22}
}

@book{mezard1987spin,
  title     = {Spin Glass Theory and Beyond: An Introduction to the Replica Method and Its Applications},
  author    = {M{\'e}zard, Marc and Parisi, Giorgio and Virasoro, Miguel Angel},
  volume    = {9},
  year      = {1987},
  publisher = {World Scientific Publishing Company},
  series    = {Lecture Notes in Physics},
  address   = {Singapore}
}

@misc{Favero2025_bigger,
    title={Bigger Isn't Always Memorizing: Early Stopping Overparameterized Diffusion Models},
    author={Alessandro Favero and Antonio Sclocchi and Matthieu Wyart},
    year={2025},
    eprint={2505.16959},
    archivePrefix={arXiv},
    primaryClass={cs.LG}
}

@inproceedings{
bonnaire2025diffusionmodelsdontmemorize,
title={Why Diffusion Models Don{\textquoteright}t Memorize:  The Role of Implicit Dynamical Regularization in Training},
author={Tony Bonnaire and Rapha{\"e}l Urfin and Giulio Biroli and Marc Mezard},
booktitle={The Thirty-ninth Annual Conference on Neural Information Processing Systems},
year={2025},
url={https://openreview.net/forum?id=BSZqpqgqM0}
}

@article{hu2023,
  author       = {Hong Hu and Yue M. Lu},
  title        = {Universality Laws for High-Dimensional Learning with Random Features},
  journal      = {IEEE Transactions on Information Theory},
  year         = {2023},
  volume       = {69},
  number       = {3},
  pages        = {1932--1964},
  doi          = {10.1109/TIT.2022.3217698},
}

@inproceedings{
buchanan2025,
title={On the Edge of Memorization in Diffusion Models},
author={Sam Buchanan and Druv Pai and Yi Ma and Valentin De Bortoli},
booktitle={The Thirty-ninth Annual Conference on Neural Information Processing Systems},
year={2025},
url={https://openreview.net/forum?id=rWW5wdECl8}
}

@inproceedings{bodin2021model,
 author = {Bodin, Antoine and Macris, Nicolas},
 booktitle = {Advances in Neural Information Processing Systems},
 editor = {M. Ranzato and A. Beygelzimer and Y. Dauphin and P.S. Liang and J. Wortman Vaughan},
 pages = {21605--21617},
 publisher = {Curran Associates, Inc.},
 title = {Model, sample, and epoch-wise descents: exact solution of gradient flow in the random feature model},
 url = {https://proceedings.neurips.cc/paper_files/paper/2021/file/b4f8e5c5fb53f5ba81072451531d5460-Paper.pdf},
 volume = {34},
 year = {2021}
}

@misc{merger2026,
      title={Generalization Dynamics of Linear Diffusion Models}, 
      author={Claudia Merger and Sebastian Goldt},
      year={2026},
      eprint={2505.24769},
      archivePrefix={arXiv},
      primaryClass={stat.ML},
      url={https://arxiv.org/abs/2505.24769}, 
}

@article{d_Ascoli_triple_2021,
   title={Triple descent and the two kinds of overfitting: where and why do they appear?*},
   volume={2021},
   ISSN={1742-5468},
   url={http://dx.doi.org/10.1088/1742-5468/ac3909},
   DOI={10.1088/1742-5468/ac3909},
   number={12},
   journal={Journal of Statistical Mechanics: Theory and Experiment},
   publisher={IOP Publishing},
   author={d’Ascoli, Stéphane and Sagun, Levent and Biroli, Giulio},
   year={2021},
   month=Dec, pages={124002} }

@inproceedings{
farghly2026benign,
title={Benign Overfitting Does Not Occur in Diffusion Models},
author={Tyler Farghly and Benjamin Dupuis and Alain Oliviero Durmus and Umut Simsekli},
booktitle={ICML 2026 Workshop on Foundations of Deep Generative Models: Understanding Memorization, Generalization, and Reasoning},
year={2026},
url={https://openreview.net/forum?id=QwP5eaXJTj}
}

@misc{marion2026understandingdiffusionmodelsrequires,
      title={Understanding diffusion models requires rethinking (again) generalization}, 
      author={Pierre Marion and Yu-Han Wu},
      year={2026},
      eprint={2605.06077},
      archivePrefix={arXiv},
      primaryClass={cs.LG},
      url={https://arxiv.org/abs/2605.06077}, 
}

@inproceedings{Pennington_2017,
 author = {Pennington, Jeffrey and Worah, Pratik},
 booktitle = {Advances in Neural Information Processing Systems},
 editor = {I. Guyon and U. Von Luxburg and S. Bengio and H. Wallach and R. Fergus and S. Vishwanathan and R. Garnett},
 pages = {},
 publisher = {Curran Associates, Inc.},
 title = {Nonlinear random matrix theory for deep learning},
 url = {https://proceedings.neurips.cc/paper_files/paper/2017/file/0f3d014eead934bbdbacb62a01dc4831-Paper.pdf},
 volume = {30},
 year = {2017}
}

@article{Nakkiran_2021,
doi = {10.1088/1742-5468/ac3a74},
url = {https://doi.org/10.1088/1742-5468/ac3a74},
year = {2021},
month = {dec},
publisher = {IOP Publishing and SISSA},
volume = {2021},
number = {12},
pages = {124003},
author = {Nakkiran, Preetum and Kaplun, Gal and Bansal, Yamini and Yang, Tristan and Barak, Boaz and Sutskever, Ilya},
title = {Deep double descent: where bigger models and more data hurt*},
journal = {Journal of Statistical Mechanics: Theory and Experiment}
}

@article{Anderson_1982,
title = {Reverse-time diffusion equation models},
journal = {Stochastic Processes and their Applications},
volume = {12},
number = {3},
pages = {313-326},
year = {1982},
issn = {0304-4149},
doi = {https://doi.org/10.1016/0304-4149(82)90051-5},
url = {https://www.sciencedirect.com/science/article/pii/0304414982900515},
author = {Brian D.O. Anderson}
}

@inproceedings{
zhang2017understanding,
title={Understanding deep learning requires rethinking generalization},
author={Chiyuan Zhang and Samy Bengio and Moritz Hardt and Benjamin Recht and Oriol Vinyals},
booktitle={International Conference on Learning Representations},
year={2017},
url={https://openreview.net/forum?id=Sy8gdB9xx}
}

@article{
Belkin_2019,
author = {Mikhail Belkin  and Daniel Hsu  and Siyuan Ma  and Soumik Mandal },
title = {Reconciling modern machine-learning practice and the classical bias–variance trade-off},
journal = {Proceedings of the National Academy of Sciences},
volume = {116},
number = {32},
pages = {15849-15854},
year = {2019},
doi = {10.1073/pnas.1903070116},
URL = {https://www.pnas.org/doi/abs/10.1073/pnas.1903070116},
eprint = {https://www.pnas.org/doi/pdf/10.1073/pnas.1903070116}}

@article{
Barlett_benign,
author = {Peter L. Bartlett  and Philip M. Long  and Gábor Lugosi  and Alexander Tsigler },
title = {Benign overfitting in linear regression},
journal = {Proceedings of the National Academy of Sciences},
volume = {117},
number = {48},
pages = {30063-30070},
year = {2020},
doi = {10.1073/pnas.1907378117},
URL = {https://www.pnas.org/doi/abs/10.1073/pnas.1907378117},
eprint = {https://www.pnas.org/doi/pdf/10.1073/pnas.1907378117}}

@article{hastie2022surprises,
author = {Trevor Hastie and Andrea Montanari and Saharon Rosset and Ryan J. Tibshirani},
title = {{Surprises in high-dimensional ridgeless least squares interpolation}},
volume = {50},
journal = {The Annals of Statistics},
number = {2},
publisher = {Institute of Mathematical Statistics},
pages = {949 -- 986},
year = {2022},
doi = {10.1214/21-AOS2133},
URL = {https://doi.org/10.1214/21-AOS2133}
}

@article{guerra2002thermodynamic,
  title={The thermodynamic limit in mean field spin glass models},
  author={Guerra, Francesco and Toninelli, Fabio Lucio},
  journal={Communications in Mathematical Physics},
  volume={230},
  pages={71--79},
  year={2002}
}

@article{talagrand2006parisi,
  title={The Parisi formula},
  author={Talagrand, Michel},
  journal={Annals of mathematics},
  pages={221--263},
  year={2006}
}

@article{barbier2019optimal,
  title={Optimal errors and phase transitions in high-dimensional generalized linear models},
  author={Barbier, Jean and Krzakala, Florent and Macris, Nicolas and Miolane, L{\'e}o and Zdeborov{\'a}, Lenka},
  journal={Proceedings of the National Academy of Sciences},
  volume={116},
  pages={5451--5460},
  year={2019}
}

@article{gerbelot2023asymptotic,
  title={Asymptotic Errors for Teacher-Student Convex Generalized Linear Models (Or: How to Prove Kabashima’s Replica Formula)},
  author={Gerbelot, C{\'e}dric and Abbara, Alia and Krzakala, Florent},
  journal={IEEE Transactions on Information Theory},
  volume={69},
  pages={1824--1852},
  year={2023}
}

@article{vilucchio2025asymptotics,
  title={Asymptotics of non-convex generalized linear models in high-dimensions: A proof of the replica formula},
  author={Vilucchio, M. and Dandi, Y. and Rossignol, M. P. and Gerbelot, C. and Krzakala, F.},
  journal={arXiv preprint arXiv:2502.20003},
  year={2025}
}

@misc{esser2024scalingrectifiedflowtransformers,
      title={Scaling Rectified Flow Transformers for High-Resolution Image Synthesis}, 
      author={Patrick Esser and Sumith Kulal and Andreas Blattmann and Rahim Entezari and Jonas Müller and Harry Saini and Yam Levi and Dominik Lorenz and Axel Sauer and Frederic Boesel and Dustin Podell and Tim Dockhorn and Zion English and Kyle Lacey and Alex Goodwin and Yannik Marek and Robin Rombach},
      year={2024},
      eprint={2403.03206},
      archivePrefix={arXiv},
      primaryClass={cs.CV},
      url={https://arxiv.org/abs/2403.03206}, 
}

@inproceedings{
yang2025cogvideox,
title={CogVideoX: Text-to-Video Diffusion Models with An Expert Transformer},
author={Zhuoyi Yang and Jiayan Teng and Wendi Zheng and Ming Ding and Shiyu Huang and Jiazheng Xu and Yuanming Yang and Wenyi Hong and Xiaohan Zhang and Guanyu Feng and Da Yin and Yuxuan Zhang and Weihan Wang and Yean Cheng and Bin Xu and Xiaotao Gu and Yuxiao Dong and Jie Tang},
booktitle={The Thirteenth International Conference on Learning Representations},
year={2025},
url={https://openreview.net/forum?id=LQzN6TRFg9}
}

@inproceedings{
kong2021diffwave,
title={DiffWave: A Versatile Diffusion Model for Audio Synthesis},
author={Zhifeng Kong and Wei Ping and Jiaji Huang and Kexin Zhao and Bryan Catanzaro},
booktitle={International Conference on Learning Representations},
year={2021},
url={https://openreview.net/forum?id=a-xFK8Ymz5J}
}

@InProceedings{ali2019_earlystopping,
  title = 	 {A Continuous-Time View of Early Stopping for Least Squares Regression},
  author =       {Ali, Alnur and Kolter, J. Zico and Tibshirani, Ryan J.},
  booktitle = 	 {Proceedings of the Twenty-Second International Conference on Artificial Intelligence and Statistics},
  pages = 	 {1370--1378},
  year = 	 {2019},
  editor = 	 {Chaudhuri, Kamalika and Sugiyama, Masashi},
  volume = 	 {89},
  series = 	 {Proceedings of Machine Learning Research},
  month = 	 {16--18 Apr},
  publisher =    {PMLR},
  url = 	 {https://proceedings.mlr.press/v89/ali19a.html}
}

@inproceedings{
neal2019a,
title={A Modern Take on the Bias-Variance Tradeoff in Neural Networks},
author={Brady Neal and Sarthak Mittal and Aristide Baratin and Vinayak Tantia and Matthew Scicluna and Simon Lacoste-Julien and Ioannis Mitliagkas},
booktitle={ICML 2019 Workshop on Identifying and Understanding Deep Learning Phenomena},
year={2019},
url={https://openreview.net/forum?id=B1guPVr2h4}
}

@inproceedings{nakkiran2021optimal,
  title     = {Optimal Regularization Can Mitigate Double Descent},
  author    = {Nakkiran, Preetum and Venkat, Prayaag and Kakade, Sham M. and Ma, Tengyu},
  booktitle = {International Conference on Learning Representations (ICLR)},
  year      = {2021},
  eprint    = {2003.01897},
  archivePrefix = {arXiv},
  url       = {https://arxiv.org/abs/2003.01897}
}

@misc{latourellevigeant2026generalizationmemorizationoverfittingdiffusion,
      title={Generalization, memorization, and overfitting for diffusion models trained in the lazy high-dimensional regime}, 
      author={Hugo Latourelle-Vigeant and Sinho Chewi and Aram-Alexandre Pooladian and John Sous and Theodor Misiakiewicz},
      year={2026},
      eprint={2608.23938},
      archivePrefix={arXiv},
      primaryClass={stat.ML},
      url={https://arxiv.org/abs/2608.23938}, 
}
